\documentclass{article}

 \usepackage[nonatbib,preprint]{neurips_2026}

\usepackage[utf8]{inputenc} 
\usepackage[T1]{fontenc}    
\usepackage{hyperref}       
\usepackage{url}            
\usepackage{booktabs}       
\usepackage{amsfonts}       
\usepackage{nicefrac}       
\usepackage{microtype}      
\usepackage{xcolor}         
\usepackage[table,dvipsnames]{xcolor}

\usepackage[numbers,compress]{natbib}
\usepackage{amsmath}
\usepackage{thm-restate}
\usepackage{adjustbox}
\usepackage{graphicx}

\usepackage{amsmath,amsfonts,bm}

\def\eqref#1{Eq.~\ref{#1}}

\def\1{\bm{1}}

\DeclareMathAlphabet{\mathsfit}{\encodingdefault}{\sfdefault}{m}{sl}
\SetMathAlphabet{\mathsfit}{bold}{\encodingdefault}{\sfdefault}{bx}{n}

\newcommand{\E}{\mathbb{E}}

\DeclareMathOperator*{\argmin}{arg\,min}

\newcommand{\mask}{0}

\newcommand{\matrixb}{\left[ \begin{array}}
\newcommand{\matrixe}{\end{array} \right]}

\usepackage{xspace}
\makeatletter
\DeclareRobustCommand\onedot{\futurelet\@let@token\@onedot}
\def\@onedot{\ifx\@let@token.\else.\null\fi\xspace}

\def\eg{\emph{e.g}\onedot} 
\def\ie{\emph{i.e}\onedot}

\def\etal{\emph{et al}\onedot}
\makeatother

\renewcommand{\paragraph}[1]{\vspace{1mm}\noindent\textbf{#1}\,\,\,}

\usepackage{enumitem}
\setlist[itemize]{align=parleft,left=0pt}

\usepackage[most]{tcolorbox}
\usepackage{caption}
\newtcolorbox{findingbox}{
  colback=blue!5,
  colframe=blue!60!black,
  boxrule=0.9pt,
  arc=2.5mm,
  left=1.2mm,
  right=1.2mm,
  top=1mm,
  bottom=1mm
}
\usepackage{wrapfig}
\usepackage{booktabs}
\usepackage{multirow}
\usepackage{amsthm}
\usepackage{mathtools}

\usepackage{tabularx}
\theoremstyle{plain}
\usepackage[table]{xcolor}
\newtheorem{theorem}{Theorem}[section]
\newtheorem{proposition}[theorem]{Proposition}
\newtheorem{lemma}[theorem]{Lemma}

\theoremstyle{definition}
\newtheorem{definition}[theorem]{Definition}

\theoremstyle{remark}

\newcommand{\x}{\mathbf{x}}

\def\mask{{m}}
\def\Mask{{\mathbf{M}}}

\def\x{{\mathbf{x}}}
\def\e{{\mathbf{e}}}

\def\model{{\mathbf{\mu}_\theta^{i}}}

\newcommand{\red}[1]{{\color{BrickRed}{#1}}}
\newcommand{\blue}[1]{{\color{MidnightBlue}{#1}}}
\def\eqref#1{Eq.~(\ref{#1})}

\usepackage{subcaption}

\title{Understanding and Accelerating the Training of Masked Diffusion Language Models}

\author{%
Chunsan Hong\textsuperscript{1}\thanks{Equal contribution.}
\quad
Sanghyun Lee\textsuperscript{1}$^*$
\quad
Chieh-Hsin Lai\textsuperscript{2}
\quad
Satoshi Hayakawa\textsuperscript{3}
\\
\textbf{Yuhta Takida\textsuperscript{2}}
\quad
\textbf{Yuki Mitsufuji\textsuperscript{2,4}}
\quad
\textbf{Seungryong Kim\textsuperscript{1}\thanks{Corresponding authors.}}
\quad
\textbf{Jong Chul Ye\textsuperscript{1}$^\dagger$}
\\[2mm]
\textsuperscript{1}KAIST
\quad
\textsuperscript{2}Sony AI
\quad
\textsuperscript{3}The University of Tokyo
\textsuperscript{4}Sony Group Corporation
}

\begin{document}

\maketitle

\begin{abstract}
Masked diffusion models (MDMs) have emerged as a promising alternative to autoregressive models (ARMs) for language modeling. However, MDMs are known to learn substantially more slowly than ARMs, which may become problematic when scaling MDMs to larger models. Therefore, we ask the following question: how can we accelerate standard MDM training while maintaining its final performance? To this end, we first provide a detailed analysis of why MDM training is slow. We find that the main factor is the locality bias of language: the predictive information for a token is concentrated in nearby positions. We further investigate how this bias slows learning and suggest a simple yet effective remedy: bell-shaped time sampling as a training strategy. Notably, MDMs trained with our training recipe reach the same validation negative log-likelihood (NLL) up to $\sim4\times$ faster than standard training on One Billion Word Benchmark (LM1B). We also show faster improvements in generative perplexity, zero-shot perplexity, and downstream task performance on various benchmarks. 
\end{abstract}

\begin{figure}[h]
    \centering
    \vspace{-0.3cm}
    \includegraphics[width=0.9\linewidth]{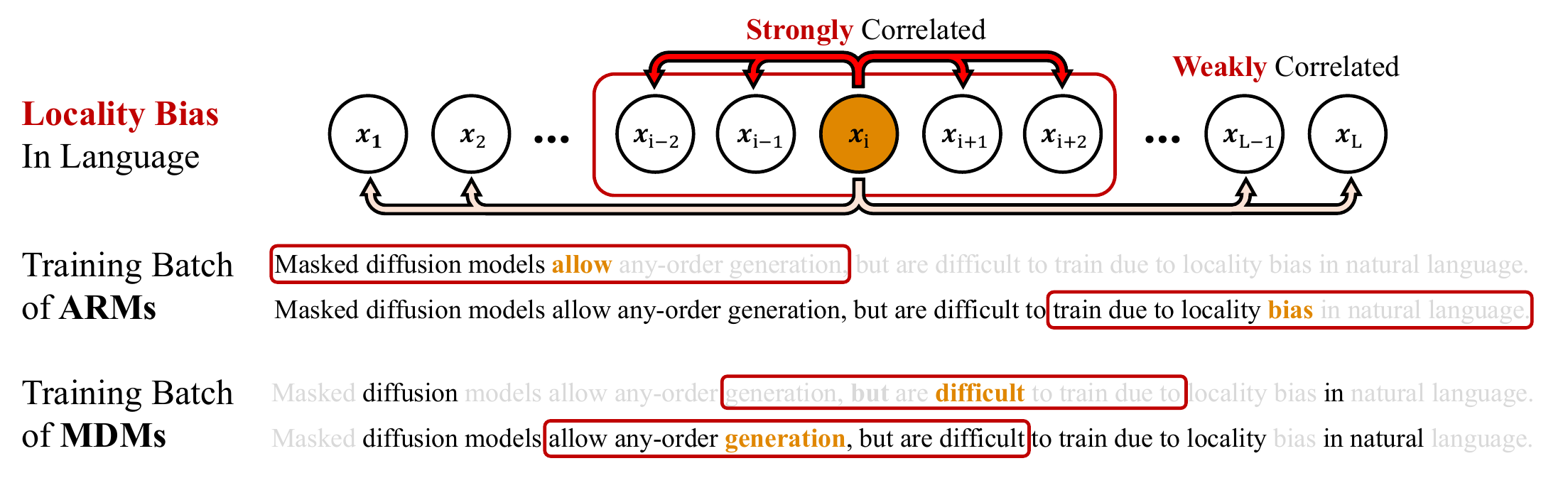}
    \caption{
    \textit{Locality bias} in language makes MDMs slow learners. The \textcolor{BrickRed}{red box} illustrates the concept of locality bias: a token is influenced more strongly by nearby tokens and more weakly by distant ones. 
    During training, ARMs learn to predict the \textcolor{orange!85!black}{target token} from left-filled sequences, where an appropriate amount of local information is always available; in contrast, MDMs learn to predict the \textcolor{orange!85!black}{target token} from arbitrary masked contexts, which may include cases where local information is either extremely scarce or extremely abundant, as highlighted by the \textcolor{BrickRed}{red boxes}.
    Our analysis suggests that these two regimes may be among the critical factors underlying the slow training of MDMs.
    }
    \label{fig:intro}
    \vspace{-0.2cm}
\end{figure}

\section{Introduction}

Recently, discrete diffusion models~\citep{lou2024discrete,Austin2021} have rapidly emerged as a compelling framework for language modeling. They can be broadly divided into uniform discrete diffusion models~\citep{lou2024discrete,sahoo2025diffusionduality} and masked diffusion models (MDMs)~\citep{shi2025simplified,lou2024discrete,sahoo2024simple,ou2024RADD}, with various frameworks~\citep{xie2025vadd,liu2025discrete,hayakawa2025distillation,kim2026flexMDM,zheng2026cadd,zhou2025ccdd,pynadath2025candi,peng2025papl,hong2026lomdm} further developed from these two families. Among this wide range of studies, \textit{the most basic MDM framework}~\citep{ou2024RADD,sahoo2024simple} \textit{remains the most widely adopted} in real-world applications precisely because of its simplicity. Large-scale models, such as LLaDA~\citep{nie2025llada}, Seed Diffusion~\citep{song2025seed}, and Dream~\citep{ye2025dream}, ultimately follow the most standard generation scheme, in which model predictions are used to gradually convert a fully masked sentence into tokens. Various inference-time techniques~\citep{lee2025lookahead,hayakawa2025demystifyingmaskgitsamplerbeyond}, reinforcement learning methods~\citep{hong2025improvingdiscretediffusionunmasking,zhao2025d1scalingreasoningdiffusion}, and even retrieval-augmented generation~\citep{yu2026diffusionRag1,kim2026diffusionRag2} are built upon this basic MDM framework.

Recent evidence suggests standard MDMs can become competitive with autoregressive models (ARMs). For general-purpose models, MDMs such as LLaDA~\citep{nie2025llada} and Dream~\citep{ye2025dream} have surpassed same-sized ARM counterparts such as LLaMA 3 8B~\citep{grattafiori2024llama3herdmodels}, while for coding models, DiffuCoder~\citep{gong2026diffucoder} and Dream-Coder~\citep{xie2025dreamcoder} have matched same-sized Qwen2.5-Coder~\citep{hui2024qwen25coder}. Prior work~\citep{prabhudesai2025LimitedData} further suggests that given sufficiently long training, MDMs can eventually match the performance of ARMs. 

However, the main bottleneck of MDMs is that they learn much more slowly than ARMs. Specifically, MDMs are known to require$\sim16\times$ training compute to reach the same validation loss of ARMs~\citep{prabhudesai2025LimitedData,sahoo2026scaling}.
This becomes a major obstacle when scaling MDMs to real-world LLMs, since the training cost of LLMs is extremely high; for example, even ARM-based LLaMA 3.1 405B spent $\sim$30M H100 GPU hours for training. Differences in training efficiency can translate into enormous differences in both cost and time. In this context, we raise a key question that has been overlooked in prior works:

\begin{center}
\textit{How can we accelerate standard MDM training while maintaining its final performance?}
\end{center}

To this end, we first \textbf{analyze why MDM training is slow.} Stepping back from MDM, we begin by separating two possible causes of slow any-order learning. The first is the locality bias of language: each token in natural language is influenced much more strongly by nearby tokens than by distant ones~\citep{futrell2019localitybias}. Since any-order generation includes orders beyond left-to-right (L2R), locality bias can shape its training dynamics differently. The second is the large order-space complexity of the any-order generation problem: unlike a fixed L2R order, models must learn a vastly larger family of generation problems. We show that the former factor, locality bias, has a substantially greater impact on training speed. Furthermore, we analyze why the locality bias slows MDM training, and find that both high-context (almost no masked) and low-context (almost fully masked) samples waste training budget. 
Although L2R bias, closely related to locality bias, has been pointed out as a challenge in many prior works~\citep{kim2025trainworstplanbest,xue2025anyordergpt}, there has been little detailed analysis of \textit{how} it slows down training. 

Our second contribution is to introduce a \textbf{simple, effective, and generalizable technique for accelerating MDM training}. Motivated by our analysis of why MDMs learn slowly, we propose \textit{bell-shaped time sampling}, which increases the probability of sampling $t$ from the middle-context region, and show that this significantly speeds up MDM training. Notably, bell-shaped time sampling achieves up to \textbf{$\mathbf{\sim4\times}$ faster training} in reaching the same NLL on LM1B~\citep{chelba2013lm1b}. We show that it also leads to faster improvements in generative perplexity (PPL), and zero-shot PPL for models trained on OpenWebText~\citep{Gokaslan2019owt}.
Finally, we show that our method remains effective even in the continual pretraining with ARM initialization~\citep{gong2024scaling} setting commonly adopted for large-scale MDMs~\citep{bie2025llada2}, and can further improve downstream task performance.

\section{Background}
\subsection{Preliminaries}
To analyze why MDMs are slow learners, we compare various classes of language models.
We briefly explain those frameworks and objectives. Refer to Appendix~\ref{sec:related_works} for details and related works.

\noindent\textbf{Notation.} 
Let $\x$ be a categorical sample with $L$ dimensions and $x^i\in\{1,\dots,V\}$ be the $i$-th dimension of $\x$. We define special mask state as $\mask=V+1$.
Let $\e$ denote the standard basis of $\mathbb{R}^{V+1}$.
Thus, each token is represented by a one-hot vector $\e_i$ in $\mathbb{R}^{V+1}$.
Let $\mathrm{Cat}(\boldsymbol{\pi})$ denote the categorical distribution over $V\!+\!1$ classes with $\boldsymbol{\pi}\in\Delta^{V+1}$, the $(V+1)$-simplex.
We denote a noised sequence at time $t$ by $\x_t$. 
We further define a mask indicator $\Mask\in\{0,1\}^L$ and the masked sequence $\x[\Mask]$ that is obtained by replacing $x^i$ with $\mask$ whenever $M^i=1$, while leaving it unchanged whenever $M^i=0$. 
Let $\Mask_n$ denote a mask indicator with mask number of $n$ and $\mathcal{M}_{n}$ as set of all possible combinatorial $\Mask_n$.
For ARM, let $\x^{<i}:=(x^1,\ldots,x^{i-1})$. We write uniform distribution as $\mathrm{U}(\cdot)$. 

\noindent\textbf{Masked diffusion models.}
MDM~\citep{sahoo2024simple,shi2025simplified} defines the forward corruption process using the absorbing mask strategy: once a token is masked, it remains masked throughout the remaining process.
For the diffusion process, define the time interval as $t\in\mathcal{T}=[0,1]$ where we corrupt the data from $t=0$ (original data) to $t=1$ (completely masked).  
Formally, the forward process at time $t$ is defined as $q_{t|0}(x_t^i \mid \x)= \mathrm{Cat}(\cdot; \alpha_t \e_{x^i} + (1 - \alpha_t)\e_\mask)$,
where the forward process gradually adds noise as $t$ grows. In this regard, the noise scheduler 
is typically set to $\alpha_t = 1-t$. 
To mimic the true posterior induced from the forward process, the parametrized reverse process is designed as $p_{\theta}(\x_s^{i} \mid \x_t)
=
\mathbf{1}(x_t^i \neq \mask)\cdot\mathrm{Cat}(\e_{x_t^i})
+ \mathbf{1}(x_t^i = \mask)\cdot\mathrm{Cat}(\tfrac{1-\alpha_s}{1-\alpha_t}\e_\mask + \tfrac{\alpha_s - \alpha_t}{1 - \alpha_t}\model(\x_t, t))$
where $s<t$ and $\mathbf{1}(\cdot)$ denotes indicator function. The diffusion network $\model(\x_t, t)\in\Delta^{V+1}$ here predicts the clean token $x^i$ over the vocabulary. The reverse process is defined as a token-wise conditionally independent distribution, \ie,
$p_\theta(\x_s \mid \x_t)
= \prod_{i=1}^L p_\theta(\x_s^{i} \mid \x_t)$. 
The NELBO is then:
\begin{align}
\textstyle\mathcal{L}_{\mathrm{MDM}}=\int_{0}^{1}
\tfrac{-\alpha_t'}{1-\alpha_t}\mathbb{E}_{\x\sim p_{\mathrm{data}},\x_t\sim q_{t|0}}
\Bigl[
\sum_{i:x_t^i=\mask}
-\log p_\theta(x^i|\x_t,t)\Big]dt.\label{eq:loss_mdlm}
\end{align}

\noindent\textbf{Autoregressive models.}
A standard ARM factorizes the joint distribution in the L2R order:
\begin{align}
    \textstyle p_{\theta}^{\textrm{ARM}}(\x)=\prod_{i=1}^{L} p_{\theta}^{\textrm{ARM}}\!\left(x^i \mid \x^{<i}\right),\quad \mathcal{L}_{\mathrm{ARM}}=\mathbb{E}_{\x\sim p _{\mathrm{data}}}[\sum_{i=1}^L-\log p_{\theta}^{\textrm{ARM}}(x^i|\x^{<i})].\label{eq:loss_ARM}
\end{align}

\noindent\textbf{$\sigma$-ARMs and any-order ARMs.} 
Typical ARMs factorize the joint distribution in the fixed L2R order, yet we consider a fixed-permutation ARM, where a single random permutation $\sigma$ is sampled once and kept fixed throughout training. 
For brevity, we refer to this variant as $\sigma$-ARM.
Formally, fix any random permutation $\sigma\in S_L$ where $S_L$ denotes the set of all possible permutations of $[L]$. $\sigma$-ARMs then factorize joint probability as $
    p_{\theta}^{\sigma\textrm{-ARM}}(\x)
=
\prod_{r=1}^{L}
p_{\theta}^{\sigma\textrm{-ARM}}\!\left(x^{\sigma_r} \mid \x^{<\sigma_{r}}\right)$. Besides, Any-order ARMs (AoARMs)~\citep{xue2025anyordergpt} factorize joint probability with \emph{variable (not fixed)} random orders as $p_\theta^{\textrm{AoARM}}(\x)=\mathbb{E}_{\sigma\sim U(S_L)}\prod_{r=1}^{L}p_{\theta}^{\textrm{AoARM}}\!\left(x^{\sigma_r} \mid \x^{<\sigma_{r}}, \sigma\right)$. The objective of AoARM is then:
\begin{gather}
\textstyle \mathcal{L}_{\mathrm{AoARM}}=
\mathbb{E}_{\x\sim p_{\mathrm{data}},\sigma\sim U(S_L)}
\left[
\sum_{r=1}^{L}
-\log p_{\theta}^{\textrm{AoARM}}\!\left(x^{\sigma_r} \mid \x^{<\sigma_{r}}, \sigma\right)
\right].\label{eq:loss_aogpt}
\end{gather}

Both MDM and AoARM share the same philosophy of performing \emph{any-order generation}. This connection also extends to the NELBO~\citep{kim2025trainworstplanbest}: if the MDM model is time-agnostic and the AoARM model is $\sigma$-agnostic, then $\mathcal{L}_{\mathrm{AoARM}}$ (\eqref{eq:loss_aogpt}) is theoretically equivalent to $\mathcal{L}_{\mathrm{MDM}}$ 
(\eqref{eq:loss_mdlm}).  
The practical difference is that AoARM utilizes causal attention such that the model is aware of its currently generated order, while MDM utilizes full attention, which is agnostic of generated order. 

\subsection{Shared Format of Learning Objectives Across Various Frameworks}
Although MDM, ARM, $\sigma$-ARM, and AoARM arise from different frameworks, all of these objectives can be viewed through the same lens: first choose the number of context (\ie, observed) tokens, then construct a partially observed context, and finally predict one target token from that context. 

Formally, let $N_c\in\{0,\dots,L-1\}$ denote the number of context tokens, let $\mathbf{z}$ denote the resulting partially observed context, and let $p_{\mathcal{F}}(\mathbf{z}\mid \mathbf{x},N_c)$ and $p_{\mathcal{F}}(i\mid \mathbf{x},\mathbf{z}_{N_c})$ denote the framework-specific distributions for constructing the context and selecting the target position, respectively. Then, for each $\mathcal{F}\in\{\mathrm{MDM},\mathrm{ARM},\sigma\mathrm{-ARM},\mathrm{AoARM}\}$, the loss can be written in the shared form:
\begin{align}
\textstyle\mathcal{L}_{\mathcal{F}}
=\sum_{N_c=0}^{L-1}\mathcal{L}_{\mathcal{F}}^{N_c}
,\quad\mathcal{L}_{\mathcal{F}}^{N_c}\coloneqq
\mathbb{E}_{\mathbf{x}\sim p_{\mathrm{data}},\mathbf{z}_{N_c}\sim p_{\mathcal{F}}(\mathbf{z}\mid \mathbf{x},N_c),i\sim p_{\mathcal{F}}(i\mid \mathbf{x},\mathbf{z}_{N_c})}
[-\log p_{\theta}(x^i\mid \mathbf{z}_{N_c})].
\label{eq:loss_shared}
\end{align}
Refer to Appendix~\ref{sec:shared_loss} for a full explanation. Therefore, the essential distinction between these frameworks does not lie in the token-level cross-entropy itself, but in how they define the distribution over contexts and target positions. For example, in ARM, $\mathbf{z}_{N_c}$ is the prefix $\x^{<N_c+1}$ and $i$ is deterministically the next position, $N_c+1$.
In MDM, $\mathbf{z}_{N_c}$ is a masked sequence with $N_c$ context tokens ($L-N_c$ masks), while $i$ is sampled from $U(i:z_{N_c}^i=m)$. For simplicity, we omit the framework-specific sampling distributions in $\mathcal{L}_{\mathcal{F}}^{N_c}$ and write $\mathbb{E}_{\x,\mathbf{z}_{N_c},i}[-\log p_{\theta}(x^i\mid \mathbf{z}_{N_c})]$. 

\section{Why Masked Diffusion Models Learn Slowly}\label{sec:slow_learner}
We now analyze why MDMs are slow learners. While many prior works focus on the L2R bias of language, we instead consider the broader notion of \textit{locality bias}~\citep{futrell2019localitybias}: in natural language, each token is influenced much more strongly by nearby tokens than by distant ones. We show here that this locality bias is indeed a critical factor in the slow training of MDMs. Furthermore, we show that low- and high-context samples impair MDM training for different reasons, both arising from locality bias. All experimental details for analysis in this section are in Appendix~\ref{sec:omitted_experimental_details}.

\subsection{Why Learning Any-order Language Generation is Hard}
\begin{wrapfigure}{r}{0.40\textwidth}
    \vspace{-1em}
    \centering
    \includegraphics[width=0.39\textwidth]{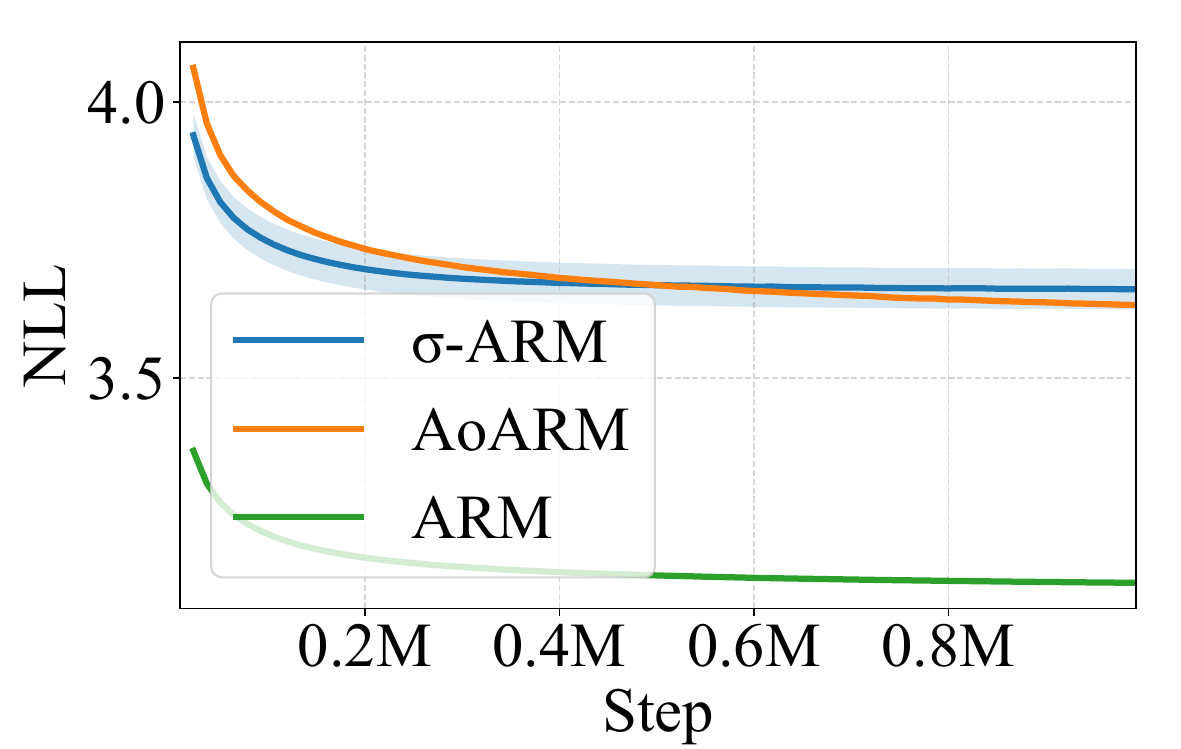}
    \vspace{-0.5em}
    \caption{Training curve on LM1B. Blue region for $\sigma$-ARM is 1-sigma band for 5 distinct $\sigma$, and the solid line is the mean.}
    \label{fig:prelim_dim_vs_locality}
    \vspace{-1.5em}
\end{wrapfigure}
Before analyzing why MDMs train slowly, we first ask a broader question: why is learning any-order language generation difficult in the first place? Among several possible factors, we focus on two main possibilities: whether the difficulty primarily arises from the large order-space complexity of any-order generation itself, or from the locality bias of natural language. Since MDMs differ fundamentally from ARMs in architecture, we isolate this question here by comparing ARM, $\sigma$-ARM, and AoARM instead: 
ARM follows the L2R order and is therefore relatively aligned with the locality structure of language. 
In contrast, $\sigma$-ARM uses a single fixed random order, which removes the large order-space complexity of any-order generation but still forces the model to predict tokens from contexts that are often misaligned with local linguistic dependencies. 
AoARM further exposes the model to unfixed random orders, and therefore combines locality mismatch with the full order-space complexity of any-order generation.

We compare ARM, $\sigma$-ARM and AoARM on LM1B~\citep{chelba2013lm1b} with sequence length 128 for 1M training steps. For $\sigma$-ARM, we fix a single random permutation throughout training, 
while AoARM is trained with \textit{non-fixed} random permutations and is therefore exposed to a much larger family of generation orders. 
The results are shown in Fig.~\ref{fig:prelim_dim_vs_locality}. If the primary difficulty of any-order generation came mainly from the order-space complexity, then $\sigma$-ARM should be substantially easier to optimize than AoARM. 
However, while AoARM appears slightly slower in the early stage of training, it eventually matches $\sigma$-ARM, suggesting that order-space complexity might not be the main problem. Meanwhile, ARM \textit{substantially outperforms} both models. These results lead to the following conclusion:

\begin{findingbox}
\textbf{Finding 1.} Locality bias appears to be an important factor in the difficulty of learning any-order generation, beyond the large order-space complexity of the generation problem.
\end{findingbox}

\subsection{Why Locality Bias in the Low-Context Region Makes MDM a Slow Learner}\label{sec:low-context}
We next analyze why MDMs are slow learners, and why locality bias makes MDMs particularly difficult to optimize.
A natural intuition is that locality bias slows down MDM training because MDMs inherently include the task of predicting targets in the low-context regime, where little local information is available.
More specifically, one might expect training to be difficult because $\mathcal{L}_{\mathrm{MDM}}^{N_c}$ is hard to reduce when $N_c$ is small.
However, we find the opposite.
The low-context region is, in fact easy to learn; rather, the problem is that \textit{MDMs waste training budget on easy low-context regions}.

A common misconception is that a task yielding high loss means the task itself is difficult to train. However, the problem should be separated: (1) what is the \textbf{\red{optimal lower bound of the $\mathcal{L}^{N_c}_{\mathcal{F}}$}}, and (2) what is \textbf{\blue{the gap between observed loss and optimal loss}}. We suggest that the \textbf{\blue{difficulty of a task}} is determined by the latter, not the former. With this distinction in mind, we further provide the following statement:

\begin{restatable}{proposition}{Entropy}\label{prop:entropy}
With infinite model capacity, the optimal loss for any framework $\mathcal{F}\in\{\mathrm{ARM,MDM},\sigma\mathrm{-ARM,AoARM}\}$ is equal to the entropy of the data distribution:
\begin{align}
   \min_{\theta}\mathcal{L}_{\mathcal{F}}=
    H(X),\quad H(X):=-\mathbb{E}_{\x\sim p_{\mathrm{data}}}\big[\log p_{\mathrm{data}}(\x)\big].
\end{align}
Furthermore, letting $I$ denote the random target index, $\mathcal{L}_{\mathcal F}^{N_c}$ can be decomposed as
\begin{align*}
    \mathcal{L}_{\mathcal F}^{N_c}(\theta)
=
\red{H_{\mathcal F}(X^I\mid Z_{N_c},I)}
+
\blue{\mathbb{E}_{Z_{N_c},I}
\mathrm{KL}\!\left(
p_{\mathrm{data},\mathcal F}(\cdot\mid Z_{N_c},I)
\,\|\,
p_\theta(\cdot\mid Z_{N_c},I)
\right)}.
\end{align*}
where the conditional entropy is induced by $p_{\mathrm{data}}$.
\end{restatable}
See Appendix~\ref{sec:proof-prop31} for the full proof.
From the proposition, \textit{ARM and MDM share the same optimal loss}.
What differs, however, is how this optimal loss is distributed across the number of context tokens.
Depending on the generation order in language modeling, the lower bound of $\mathcal{L}_{\mathcal{F}}^{N_c}$, namely the \textbf{\red{conditional entropy of data}}, varies with $N_c$, while only the lower bound of their sum remains fixed.

\noindent\textbf{Which region is \blue{hard to train}?}
A high conditional loss at small $N_c$ does not by itself imply that the region is hard to learn; rather, it means that the \textbf{\red{irreducible conditional entropy of data}} is higher.
Intuitively, the low-context region lacks visible context, and therefore there is \textit{no deterministic answer} that would make the \textbf{\red{conditional entropy}} low.
The \blue{\textbf{hard-to-train}} region is instead the region where the \textbf{\blue{difference between the true data and the model distribution}} remains high.
Even if the observed loss is smaller than in other regions, if there remains a large gap between the \textbf{\red{conditional entropy}} and the observed loss, then the region should be regarded as hard to learn.

\begin{figure}
    \centering
    \includegraphics[width=\linewidth]{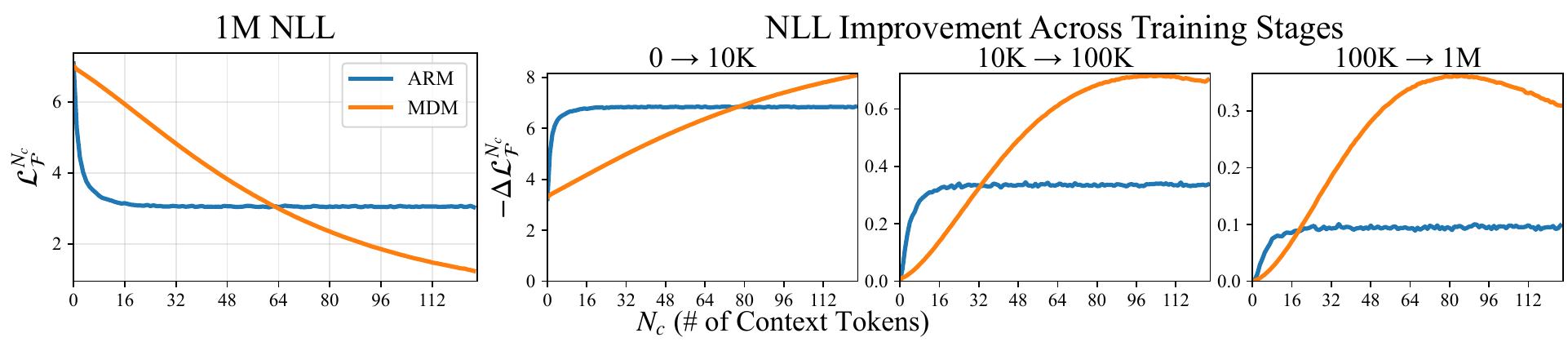}
    \caption{
    \textit{Left:} NLL per number of context tokens on LM1B after 1M training steps. ARM shows an approximately uniform distribution, whereas MDM shows a clearly uneven one.
    \textit{Right:} NLL drop across context token numbers for ARM and MDM over different training stages on LM1B.
    \textbf{For $N_c=0$, MDM has already converged to the optimal loss by 10K steps}.
    Since MDM models the joint distribution as a product of independent conditional distributions, the optimal prediction for a fully masked sentence reduces to the dataset vocabulary distribution.
    Here, the NLL of the vocabulary distribution exactly matches the NLL of the 10K-step MDM on fully masked sentences, namely 7.02.
    }
    \label{fig:uneven_nlls}
    \vspace{-0.55cm}
\end{figure}

\noindent\textbf{Empirical observations.}
We first identify where the \textbf{\red{lower bound of $\mathcal{L}_{\mathcal{F}}^{N_c}$ is high}}.
In Fig.~\ref{fig:uneven_nlls} (\textit{left}), we plot $\mathcal{L}_{\mathcal{F}}^{N_c}$ using ARM and MDM trained on LM1B for 1M steps.
Although these models are not fully optimal, the NLL improvement from 100K to 1M steps is small relative to the 1M-step NLL itself, suggesting that the 1M-step curves capture the overall shape of the lower bound.
ARM exhibits an approximately uniform curve, whereas MDM shows a highly uneven one: its loss is much larger in the low-context region and decreases as $N_c$ increases.
We attribute this difference to locality bias.
As illustrated in Fig.~\ref{fig:intro}, ARM always predicts under the same asymmetric context structure, where the left side is observed and the right side is unobserved.
In contrast, MDM changes the local context structure with $N_c$: in the low-context region, nearby tokens on both sides are mostly unobserved, while in the high-context region, they are mostly observed.
Thus, the lower bound of $\mathcal{L}_{\mathrm{ARM}}^{N_c}$ remains nearly flat, whereas that of $\mathcal{L}_{\mathrm{MDM}}^{N_c}$ becomes strongly skewed toward small $N_c$.

We now identify whether \textbf{\blue{the low-context region is difficult to learn or not}}. 
The extreme case $N_c=0$ gives the most reliable and interpretable results: as detailed in the caption of Fig.~\ref{fig:uneven_nlls}, MDM has already reached the optimal fully masked prediction by 10K steps.
Thus, the remaining high loss at $N_c=0$ is not reducible model error, but irreducible conditional entropy.
More broadly, under locality bias, low-context samples often provide little nearby information about the target token, and thus are expected to contain only limited learnable signal beyond this irreducible component.
Consistent with this view, Fig.~\ref{fig:uneven_nlls} (\textit{right}) shows that low-context samples yield almost no further NLL improvement from 10K to 1M steps, while most of the remaining learning occurs in regions with more context.
Therefore, low-context samples have high loss mainly because their conditional entropy is intrinsically high, not because the model still has much to learn from them.
Consequently, standard MDM training repeatedly samples low-context problems that provide little additional learning signal, thereby wasting training budget.
We summarize this observation as follows:
\begin{findingbox}
    \textbf{Finding 2.} Low-context samples are not hard to learn; they are learned early and then waste the training budget of MDM.
\end{findingbox}

\subsection{Why Locality Bias in the High-Context Region Also Makes MDM a Slow Learner}\label{sec:high-context}
We have shown that low-context samples are training-inefficient due to locality bias. We now consider the opposite: we show that high-context samples are also training-inefficient due to locality bias. When learning $p_\theta(x^i|\x_t,t)$, the model learns to predict the masked token while using the observed real tokens to build representations. The problem is that, due to locality bias, real tokens that are far from the mask contribute very little useful information to this learning process.

\begin{table*}[t]
\centering

\begin{minipage}[t]{0.48\textwidth}
\centering
\caption{Cosine similarity between gradients before and after masking 8 tokens within radius.}
\label{tab:grad_analysis}
\vspace{0.2em}
\begin{adjustbox}{scale=0.85,center}
\begin{tabular}{c c c c c c}
\toprule
Radius $r$ & 4 & 8 & 16 & 32 & 128 \\
\midrule
Cos. Sim. & 0.108 & 0.463 & 0.685 & 0.839 & 0.935 \\
\bottomrule
\end{tabular}
\end{adjustbox}
\end{minipage}
\hfill
\begin{minipage}[t]{0.50\textwidth}
\centering
\caption{Cosine similarity between the joint-mask gradient and the sum of single-mask gradients.}
\label{tab:additivity}
\vspace{0.2em}
\begin{adjustbox}{scale=0.85,center}
\begin{tabular}{c c c c c c}
\toprule
\# of Masks & 2 & 4 & 8 & 16 & 32 \\
\midrule
Cos. Sim. & 0.999 & 0.997 & 0.994 & 0.985 & 0.976 \\
\bottomrule
\end{tabular}
\end{adjustbox}
\end{minipage}

\vspace{-0.5em}
\end{table*}

We first show that real tokens far from a mask have little effect on the learning signal. For a controlled experiment, using an MDM trained for 1M steps, we randomly select one center mask and analyze the gradient of this center mask with and without additionally masking nearby real tokens. More specifically, we set several radii $r$ and randomly replace 8 tokens within each region with masks. As shown in Table~\ref{tab:grad_analysis}: the cosine similarity is 0.108 when $r=4$, but increases to 0.935 when $r=128$. This indicates that real tokens farther from the mask have little impact on the learning signal.

We next show that the gradient of high-context samples can be replaced by that of middle-context samples.
Fix the number of masks $k$, and randomly sample $k$ distinct single-mask patterns, \ie, $\Mask_1$.
We then construct $\Mask_k$ by taking the union of all masked positions.
For the same sequence $\x$, we measure the cosine similarity between the sum of the gradients obtained from the $k$ single-mask patterns and the gradient obtained from a single $\Mask_k$, and report the results in Table~\ref{tab:additivity}.
The cosine similarity remains very high across different values of $k$, suggesting that high-context samples are training-inefficient and can be effectively replaced by middle-context samples.

These results imply that many observed tokens in high-context samples do not introduce notable new learning signals, and middle-context samples can replace them. Consequently, even high-context training can be inefficient: although more tokens are observed, \textit{they waste the training budget}: 
\begin{findingbox}
\textbf{Finding 3.} Under locality bias, high-context samples waste the training budget of MDM.
\end{findingbox}
\vspace{-0.2cm}

\section{Mitigating the Locality Bias Problem in Masked Diffusion Models}
\subsection{Training MDM with Bell-Shaped Time Sampling} Sec.~\ref{sec:slow_learner} naturally suggests that MDM training can be improved by focusing more on middle-context problems. We propose the most straightforward way to achieve this: directly sampling more training points around $t=0.5$, \ie, the middle-context region. We first define a bell-shaped distribution:

\begin{definition}\label{def:bell-shape}
A probability distribution \(\pi\) over \([0,1]\) is called bell-shaped if 
$\mathbb{P}_{t\sim\pi}\left(t\in [a,1-a]\right)\ge \mathbb{P}_{t\sim\mathrm{U}([0,1])}\left(t\in [a,1-a]\right)$
for all \(a\in[0,1/2)\), and the inequality is strict for at least one \(a\in[0,1/2)\).
\end{definition}

Given a bell-shaped time distribution $\pi$, we train MDMs with the following objective:
\begin{align}
\textstyle \hat{\mathcal{L}}_{\pi}
=
\mathbb{E}_{\x\sim p_{\mathrm{data}},\,t\sim\pi,\,\x_t\sim q_{t|0}(\cdot\mid\x)}
\Bigl[
-
\sum_{i:x_t^i=\mask}
\log p_\theta(x^i\mid\x_t,t)
\Bigr].
\label{eq:loss_pi}
\end{align}
This objective is clearly different from the theoretical NELBO in \eqref{eq:loss_mdlm}: it samples $t$ from the tilted distribution $\pi$, and it does not include the coefficient $\alpha_t'/(1-\alpha_t)$ in the integrand.
Nevertheless, the change does not alter the target optimizer under the standard time-agnostic parameterization~\citep{sahoo2024simple,nie2025llada,ye2025dream,hong2026lomdm}.
In fact, the following statement is stronger: the optimizer is preserved not only for bell-shaped distributions, but for any time-sampling distribution with nonzero mass on $(0,1)$:

\begin{restatable}{proposition}{EqualOptimizer}\label{prop:equal_optimizer}
With infinite-capacity and time-agnostic model $\theta$, for any probability distribution $\rho$ over $[0,1]$ satisfying $\mathbb{P}_{t\sim\rho}\left(t\in (0,1)\right)>0$, the minimizers of $\mathcal{L}_{\mathrm{MDM}}$ and $\hat{\mathcal L}_{\rho}$ are equal.
\end{restatable}

\begin{table}[t]
\caption{Various time-sampling distributions on $[0,1]$ explored for MDM training. 
For Dirac-delta, Gaussian, and Laplace, they satisfy the definition of a bell-shaped distribution when $\mu=1/2$.}
\vspace{0.05cm}
\label{tab:time_sampling_distributions}
\begin{adjustbox}{scale=0.85,center}
\centering
\small
\begin{tabular}{llll}
\toprule
Uniform & Dirac-delta & Truncated Gaussian & Truncated Laplace \\
\midrule
$\pi(t)=1$ 
& $\pi(t)=\delta_\mu$
& $\pi(t)\propto \exp\!\left(-\frac{(t-\mu)^2}{2\sigma^2}\right)$
& $\pi(t)\propto \exp\!\left(-\frac{|t-\mu|}{b}\right)$ \\
\bottomrule
\end{tabular}
\end{adjustbox}
\vspace{-0.4cm}
\end{table}

See Appendix~\ref{sec:proof-cor41} for proof and note that $\mathbb{P}_{t\sim\pi}\left(t\in (0,1)\right)>0$ follows from Definition~\ref{def:bell-shape}.
In this regard, our objective is to obtain $\argmin_\theta\mathcal{L}_{\mathrm{MDM}}$ by learning $\mathcal{\hat L}_{\pi}$, which magnifies a learning signal on the middle-context region. As shown in Table~\ref{tab:time_sampling_distributions}, we consider several simple distributions with varying hyperparameters to understand which time sampling laws are beneficial.

\subsection{Theoretical Analysis}
\begin{figure}[t]
    \centering
    \includegraphics[width=0.7\linewidth]{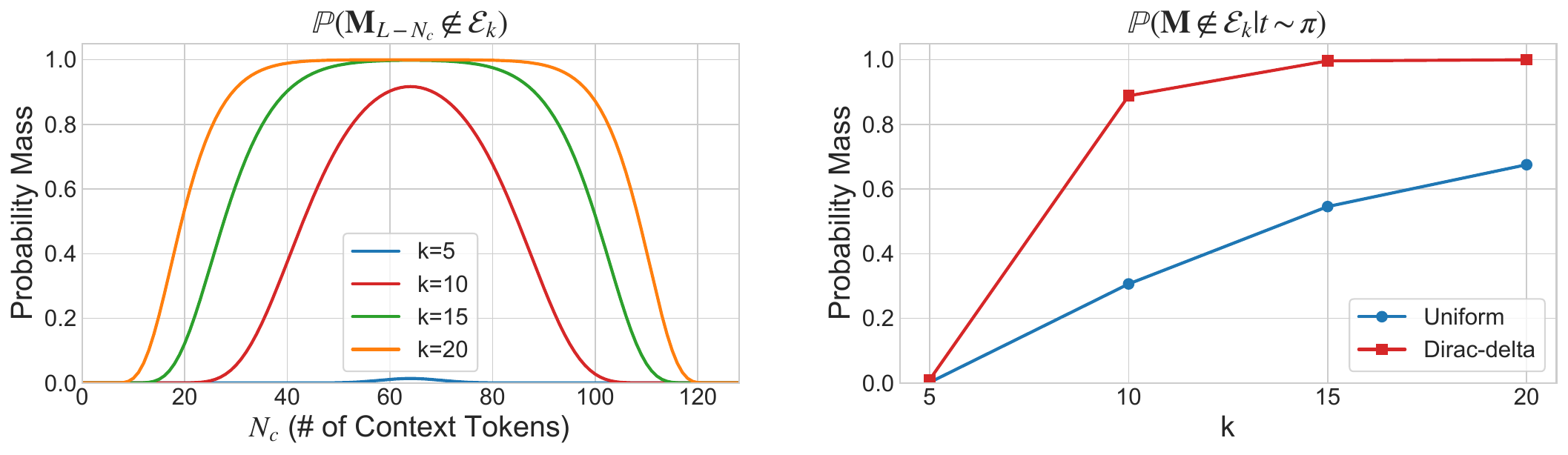}
    \caption{
    \textit{Left:} Probability that a uniformly sampled mask pattern with exactly $L-N_{c}$ masks avoids $k$-inefficient set $\mathcal{E}_k$.
    \textit{Right:} Probability that each uniform and Dirac-delta time sampling avoids $\mathcal{E}_k$.
    }
    \label{fig:theory}
\end{figure}

We now investigate which training samples are emphasized by bell-shaped time sampling. Recall that $\Mask_n$ is a mask indicator with $n$ masked positions, and $\mathcal{M}_{n}$ denotes the set of all possible $\Mask_n$.

\begin{definition}[$k$-inefficient set $\mathcal{E}_k$]
    If $\Mask$ contains $k$ consecutive masks or $k$ consecutive non-masks, we call $\Mask$ a $k$-inefficient mask pattern. We denote the set of all such $k$-inefficient patterns by $\mathcal{E}_k$.
\end{definition}
The value of meaningful $k$ may vary depending on 1) the intrinsic local bias of the language data, or 2) the model's capacity to exploit local information. Under locality bias, we regard samples from $\mathcal{E}_k$ as \textit{training-inefficient}: if we flip masks on the samples from $\mathcal{E}_k$, we might obtain a sample that gives a more diverse training signal. 
This is supported by Sec.~\ref{sec:low-context}, where we show that low-context samples converge quickly, and Sec.~\ref{sec:high-context}, where we directly show that middle-context samples can effectively replace high-context samples.
We visualize $\mathbb{P}(\Mask_{L-N_c}\notin\mathcal{E}_k\mid\Mask_{L-N_c}\sim\mathrm{U}(\mathcal{M}_{L-N_c}))$ in Fig.~\ref{fig:theory}~(\textit{left}), where $\mathcal{E}_k$ has substantial probability mass for low- and high-context regions.
Now we show the following:

\begin{restatable}{proposition}{BellShape}\label{prop:bell_shape} For any $L$, $2\le k\le L$, and any bell-shaped time distribution $\pi$, the following statement always holds:
    \begin{align*}
        \textstyle\mathbb{P}\big(\Mask\notin\mathcal{E}_k\mid t\sim\mathrm{U}([0,1]), M^i\sim \mathrm{Bernoulli}(t)\big)<\mathbb{P}\big(\Mask\notin\mathcal{E}_k\mid t\sim\pi, M^i\sim \mathrm{Bernoulli}(t)\big).
    \end{align*}
\end{restatable}
See Appendix~\ref{sec:proof-prop43} for the full proof. 
The above proposition implies that, when sampling with any bell-shaped distribution $\pi$, \textit{training observes fewer $k$-inefficient samples} than under conventional MDM sampling.
We further visualize two variants, $\pi=\mathrm{U}([0,1])$ and $\pi=\delta_{0.5}$, in Fig.~\ref{fig:theory}. 
As shown in Fig.~\ref{fig:theory}~(\textit{right}), $\mathbb{P}(\Mask \notin\mathcal{E}_k\mid t\sim\delta_{0.5})$ is much higher than with uniform sampling, suggesting that bell-shaped time sampling significantly filters out $k$-consecutive inefficient samples.

\section{Experiments}\label{sec:results}
\subsection{Training Dynamics of MDM with Various Time Sampling Distributions}\label{sec:training_dynamics}
\begin{figure}
    \centering
    \includegraphics[width=0.98\linewidth]{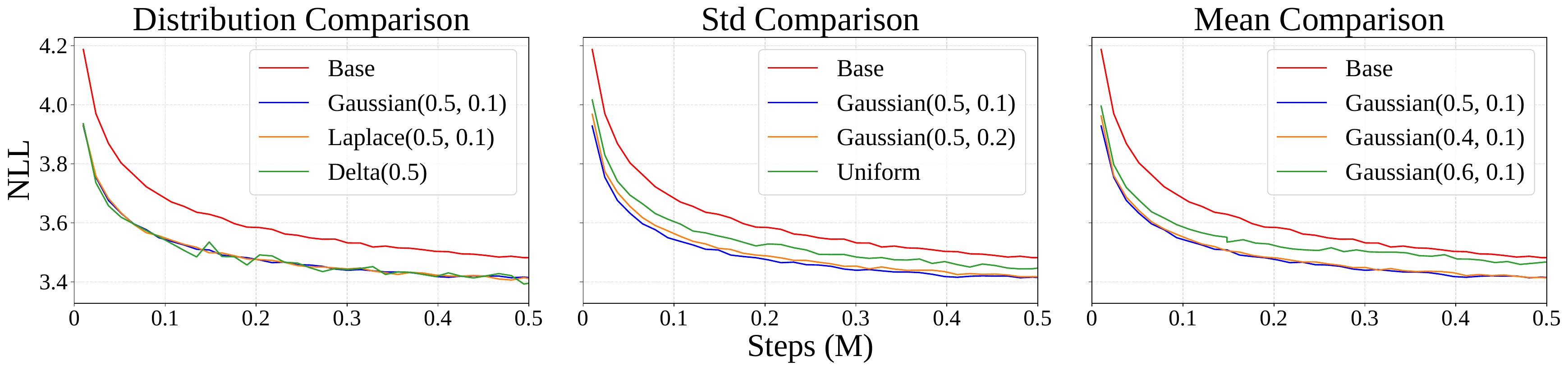}
    \caption{Validation NLL curves on LM1B for models trained with various time distributions listed in Table~\ref{tab:time_sampling_distributions}. \textit{Left}: comparison across distribution families ($\mu{=}0.5$). \textit{Middle}: effect of variance where uniform can be treated as an extreme case. \textit{Right}: effect of the mean $\mu$ under fixed $\sigma{=}0.1$. Here, ``Base'' denotes the conventional MDM trained with $\mathcal{L}_{\mathrm{MDM}}$~(\eqref{eq:loss_mdlm}) and is shown in red in all panels. Unlike ``Uniform'' with $\mathcal{\hat L}_\pi$ (\eqref{eq:loss_pi}), ``Base'' includes the $\alpha_t'/(1-\alpha_t)$ weighting in the loss integrand.}
    \label{fig:nll_curves}
\end{figure}

To confirm our understanding, we first analyze the trends exhibited by the various time sampling distributions in Table~\ref{tab:time_sampling_distributions}.  We train models on sentence-packed LM1B~\citep{chelba2013lm1b} and measure validation NLL bound based on \eqref{eq:loss_mdlm}. We follow prior works~\citep{sahoo2024simple,shi2025simplified} for the model architecture ($\sim$100M scale).

The results are shown in Fig.~\ref{fig:nll_curves}. In the left panel, \textbf{all bell-shaped distributions substantially outperform the standard MDM baseline}, while the specific class of distribution is not particularly important. The Dirac-delta, however, shows less stable dynamics.
In the middle panel, the smaller variance ($\sigma=0.1$) outperforms $\sigma{=}0.2$ and uniform, indicating the power of bell-shaped time sampling.
In the right panel, shifting the mean toward the low-context region ($\mu{=}0.6$) yields the worst result.
Based on these observations, we adopt a Gaussian with $\mu{=}0.5, \sigma{=}0.1$ as our default, offering stable training and natural centering at the middle-context region. Refer to Appendix~\ref{sec:omitted_experimental_details} for detailed experimental settings and Appendix~\ref{sec:ablation_t_sampling} for further results with various hyperparameters.

\subsection{Evaluating Bell-Shaped Time Sampling Across Multiple Pretraining Benchmarks}\label{sec:result_various_benchmark}
\begin{figure}
    \centering
    \includegraphics[width=0.98\linewidth]{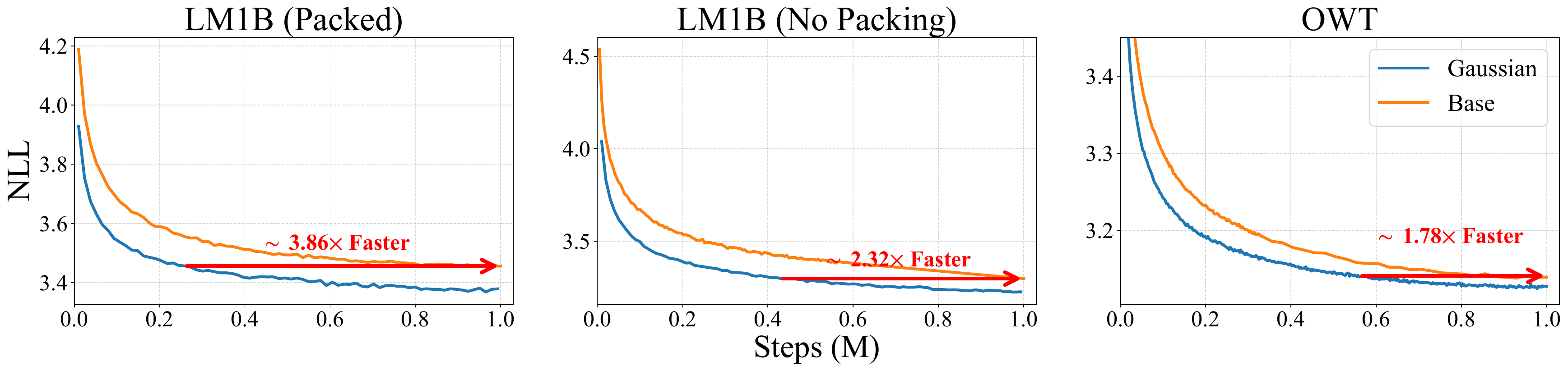}
    \caption{Validation NLL curves on various language modeling benchmarks. ``Base'' denotes base (standard) MDM, and ``Gaussian'' denotes MDM trained with $\mathcal{\hat L}_{\mathcal{N}(0.5,0.1)}$.}
    \label{fig:nll_speedup}
\end{figure}

\begin{table*}[t]
    \label{tab:downstream}
    \centering
    \caption{Language understanding and completion benchmark results for continual pretraining (CPT)-only and supervised fine-tuned (SFT) models. CPT was performed on GPT-2 Large (0.7B) for 100K steps, and SFT was performed for 10K steps. Gray rows indicate models learned by CPT with Gaussian time sampling. Boldface denotes the best result within each section. The colored values in parentheses indicate the absolute change relative to the corresponding Base row within each section. Results are averaged over evaluation with 5 random seeds, where $\pm$ denotes one standard deviation.}
    \label{tab:cpt_sft_results}
    \renewcommand{\arraystretch}{1.15}
    \begin{adjustbox}{width=\textwidth}
    \begin{tabular}{llcccccccc}
    \toprule
    \multicolumn{2}{c}{Training} & \multicolumn{8}{c}{Benchmarks} \\
\cmidrule(lr){1-2} \cmidrule(lr){3-10}
CPT & SFT & LAMBADA (\%) & Obqa (\%) & Wino (\%) & PIQA (\%) & SIQA (\%) & Infilling & TriQA(\%) & HellaSwag (\%) \\
\midrule

\multicolumn{9}{@{}l}{\textit{Continual Pretraining}} \\

Base & ---      
& 2.15 {\scriptsize$\pm$ 0.43}
& 17.00 {\scriptsize$\pm$ 0.472}
& \textbf{49.49} {\scriptsize$\pm$ 1.20}
& 52.12 {\scriptsize$\pm$ 0.36}
& 35.72 {\scriptsize$\pm$ 0.46}
& 6.94 {\scriptsize$\pm$ 0.09}
& 0.8 {\scriptsize$\pm$ 0.09}
& 28.87 {\scriptsize$\pm$ 1.31} \\

\rowcolor{gray!15}
Gaussian & ---      
& \textbf{10.25} {\scriptsize$\pm$ 0.29} (\textcolor{ForestGreen}{+8.10})
& \textbf{20.00} {\scriptsize$\pm$ 1.68} (\textcolor{ForestGreen}{+3.00})
& \textbf{49.49} {\scriptsize$\pm$ 0.33} (\textcolor{ForestGreen}{+0.00})
& \textbf{53.75} {\scriptsize$\pm$ 1.10} (\textcolor{ForestGreen}{+1.63})
& \textbf{37.05} {\scriptsize$\pm$ 1.11} (\textcolor{ForestGreen}{+1.33})
& \textbf{7.31} {\scriptsize$\pm$ 0.20} (\textcolor{ForestGreen}{+0.37})
& \textbf{1.3} {\scriptsize$\pm$ 1.02}
(\textcolor{ForestGreen}{+0.5})
& \textbf{34.70} {\scriptsize$\pm$ 1.44} (\textcolor{ForestGreen}{+5.83}) \\
    \midrule
    \multicolumn{9}{@{}l}{\textit{Continual Pretraining + Supervised Fine-tuning}} \\
Base     & Base     
& 5.82 {\scriptsize$\pm$ 0.15}
& 18.20 {\scriptsize$\pm$ 0.28}
& \textbf{50.59} {\scriptsize$\pm$ 0.67}
& 56.80 {\scriptsize$\pm$ 0.23}
& 38.95 {\scriptsize$\pm$ 0.80}
& 9.87 {\scriptsize$\pm$ 0.15}
& 6.4 {\scriptsize$\pm$ 0.19}
& 32.59{\scriptsize$\pm$ 1.24} \\

Base     & Gaussian 
& 5.65 {\scriptsize$\pm$ 0.05}
& 20.00 {\scriptsize$\pm$ 0.42}
& 48.15 {\scriptsize$\pm$ 0.39}
& 56.47 {\scriptsize$\pm$ 0.92}
& 38.54 {\scriptsize$\pm$ 0.47}
& 9.48 {\scriptsize$\pm$ 0.21}
& 7.6 {\scriptsize$\pm$ 0.06}
& 32.95{\scriptsize$\pm$ 1.14} \\

\rowcolor{gray!15}
Gaussian & Base     
& 15.54 {\scriptsize$\pm$ 0.36}
& \textbf{24.6} {\scriptsize$\pm$ 0.57}
& 50.36 {\scriptsize$\pm$ 1.95}
& \textbf{59.30} {\scriptsize$\pm$ 0.38}
& 41.04 {\scriptsize$\pm$ 0.18}
& 12.18 {\scriptsize$\pm$ 0.21}
& 27.4 {\scriptsize$\pm$ 0.07}
& 39.29{\scriptsize$\pm$ 1.33} \\

\rowcolor{gray!15}
Gaussian & Gaussian 
& \textbf{17.89} {\scriptsize$\pm$ 0.37} (\textcolor{ForestGreen}{+12.07})
& 23.4 {\scriptsize$\pm$ 1.41} (\textcolor{ForestGreen}{+5.2})
& 49.80 {\scriptsize$\pm$ 1.12} (\textcolor{red}{-0.79})
& \textbf{59.30} {\scriptsize$\pm$ 0.88} (\textcolor{ForestGreen}{+2.50})
& \textbf{42.07} {\scriptsize$\pm$ 0.04} (\textcolor{ForestGreen}{+3.12})
& \textbf{12.62} {\scriptsize$\pm$ 0.38} (\textcolor{ForestGreen}{+2.75})
& \textbf{27.8} {\scriptsize$\pm$ 0.06} (\textcolor{ForestGreen}{+21.4})
& \textbf{39.84}{\scriptsize$\pm$ 1.21} (\textcolor{ForestGreen}{+7.25}) \\
\bottomrule
    \end{tabular}
    \end{adjustbox}
\end{table*}

\noindent\textbf{Experimental settings.} 
We train an MDM with Gaussian time sampling and a base MDM on three datasets, including One Billion Words dataset (LM1B)~\citep{chelba2013lm1b} with/without sentence packing, and OpenWebText (OWT)~\citep{Gokaslan2019owt}. We generally follow the widely adopted experimental settings in discrete diffusion language modeling~\citep{sahoo2025diffusionduality,arriola2025bd3lm}: 1) validation NLL for models trained on LM1B and OWT, and 2) zero-shot PPL and generative PPL for models trained on OWT. 
See Appendix~\ref{sec:omitted_experimental_details} for details.

\noindent\textbf{Results.}
The main results are shown in Fig.~\ref{fig:nll_speedup}. 
As shown in Fig.~\ref{fig:nll_speedup}, Gaussian time sampling consistently accelerates MDM training across all language modeling benchmarks. Using the validation NLL achieved by the base MDM after 1M training steps as the target, Gaussian time sampling reaches the same NLL about $\mathbf{3.86\times}$ faster on sentence-packed LM1B, $\mathbf{2.32\times}$ faster on LM1B without packing, and $\mathbf{1.78\times}$ faster on OWT. These results indicate that the benefit of bell-shaped time sampling is robust across different data settings rather than being limited to a specific benchmark.
The same trend also appears in zero-shot PPL and generative PPL evaluations; refer to Appendix~\ref{sec:zero_and_gen_ppl}.

\subsection{Scaling Up to Billion-Scale Masked Diffusion Models}\label{sec:result_cpt}

Recent large-scale MDM~\citep{ye2025dream,bie2025llada2,fan2026stable,tae2025tess2} training recipes have increasingly shifted toward continual pretraining (CPT) from ARM. In brief, this approach~\citep{gong2024scaling} initializes the MDM with weights from the pretrained ARM and then continues with MDM pretraining. In this regard, we show that our method can scale to a real-world MDM training recipe as well. We adopt the most conventional CPT~\citep{gong2024scaling} and use pretrained GPT-2 Large at $\sim$0.7B parameter scale. Details are in Appendix~\ref{sec:omitted_experimental_details}.

\noindent\textbf{Language understanding and completion benchmarks.}
We evaluate on multiple-choice benchmarks such as  HellaSwag~\citep{zellers2019hellaswag}, OpenBookQA(Obqa)~\citep{mihaylov2018can}, PIQA~\citep{bisk2020piqa}, TriviaQA(TriQA)~\citep{joshi2017triviaqa}, Social IQA(SIQA)~\citep{sap2019social}, and WinoGrande~\citep{sakaguchi2021winogrande}, as well as last-word prediction on LAMBADA~\citep{paperno2016lambada} and story infilling on ROCStories~\citep{mostafazadeh2016corpus}, evaluated using ROUGE score~\citep{lin2004rouge}. We report results both after CPT only and after subsequent supervised fine-tuning (SFT) in Table~\ref{tab:cpt_sft_results}. Gaussian time sampling significantly outperforms the baseline across most tasks both for CPT only and SFT setup.

\begin{wraptable}{r}{0.37\textwidth}
  \vspace{-1.0em}
  \caption{Win rate of Gaussian time sampling compared to base MDM, judged by GPT-4o. Results are averaged over 5 random seeds, with $\pm$ denoting one standard deviation.}
  \centering
  \small
  \setlength{\tabcolsep}{4pt}
  \renewcommand{\arraystretch}{1.1}
  \begin{tabularx}{\linewidth}{X r}
    \toprule
    \textbf{Benchmark} & \textbf{WR (\%)} \\
    \midrule
    AlpacaEval 2.0(WR)   & $78.27\pm 1.37$ \\
    AlpacaEval 2.0(LC)   & $78.61\pm 1.42$ \\
    MT-Bench (non-tie)    & $84.34\pm7.26$ \\
    \bottomrule
  \end{tabularx}
  \label{tab:judge}
  \vspace{-1.0em}
\end{wraptable}

\noindent\textbf{Open-ended instruction-following benchmark.}
Beyond NLL-based metrics and multiple-choice QA accuracy, we further evaluate whether Gaussian time sampling improves the actual conditional generation quality. To this end, we evaluate models on two instruction-following benchmarks, MT-Bench~\citep{zheng2023judging} and AlpacaEval 2.0~\citep{li2023alpacaeval}. We report win rates (WR) judged by GPT-4o~\citep{hurst2024gpt} under the non-tie protocol; for AlpacaEval 2.0, we additionally report the length-controlled win rate (LC). We fine-tune the model on the ShareGPT dataset to equip it with conversational ability. Table~\ref{tab:judge} reports the results, where the model trained with Gaussian time sampling significantly outperforms the baseline. This indicates that the benefits of bell-shaped time sampling extend beyond NLL-based metrics and translate into improved generation quality required in the real world. Qualitative examples are in Appendix~\ref{sec:qualitative_examples}.

To sum up, our bell-shaped time sampling can (i) scale to billion-parameter-scale models, (ii) generalize to CPT with ARM initialization, and (iii) boost downstream task performance.

\subsection{Other Alternatives Explored for Focusing Training on the Middle-Context Region}\label{sec:other_alternatives}

We analyze other ways to bias MDM training toward the middle-context region.

\begin{wrapfigure}{r}{0.35\textwidth}
    \centering
    \vspace{-0.6em}
    \includegraphics[width=0.35\textwidth]{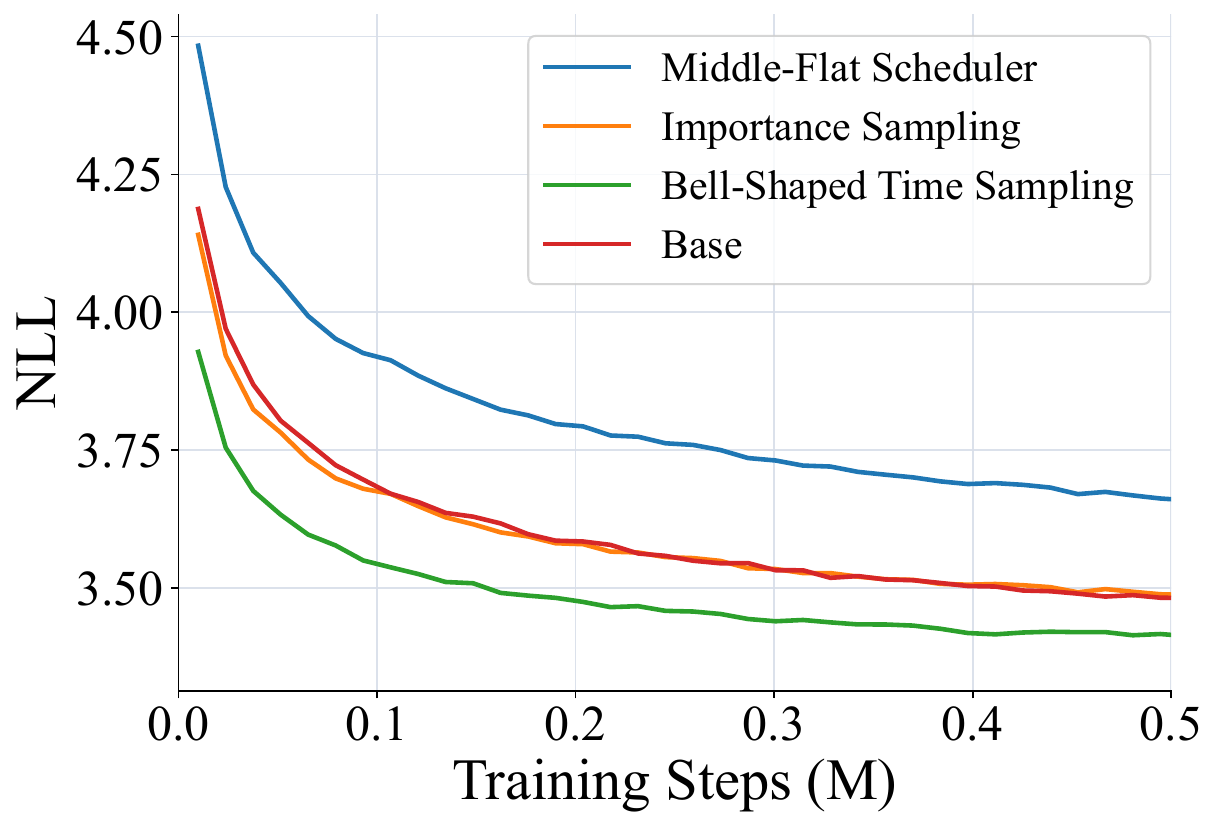}
    \vspace{-1.8em}
    \caption{Training curve on LM1B.}
    \label{fig:other_alternatives}
    \vspace{-1.0em}
\end{wrapfigure}

\noindent\textbf{Middle-flat noise scheduler.}
One can flatten the noise scheduler around $\alpha_t=0.5$ in $\mathcal{L}_{\mathrm{MDM}}$, so that the model naturally encounters more middle-context samples during training. The key difference from bell-shaped sampling is whether the target cross-entropy loss is reweighted. Under the theoretical NELBO, the scaling factor is smallest around $t=0.5$ and increases as $t$ moves away from the center. As a result, the objective is theoretically equivalent to the standard MDM loss with $\alpha_t=1-t$~\citep{shi2025simplified,sahoo2024simple}. Thus, although the model observes more middle-context during training, the learning signal assigned to each $N_c$ remains the same as in the original $\mathcal{L}_{\mathrm{MDM}}$.

\noindent\textbf{Importance sampling.}
Another option is to sample masks uniformly, as in conventional MDMs, while assigning a larger scaling factor to the middle region. Specifically, we train with $\int_{0}^{1} w(t)\mathbb{E}_{\x\sim p_{\mathrm{data}},\,\x_t\sim q_{t|0}}\bigl[-\sum_{i:x_t^i=\mask}\log p_\theta(x^i\mid \x_t,t)\bigr]dt$, where $w(t)$ is chosen to be largest near the center. In this case, the learning signal is similar to that of bell-shaped time sampling, but the model does not actually encounter middle-context samples more frequently during training.


We found that only the MDM trained with bell-shaped time sampling learns faster than the base MDM, while importance sampling achieves the same training speed, and the middle-flat scheduler even leads to worse performance. See Appendix~\ref{sec:exp_alternatives} for detailed experimental settings and our interpretation of the corresponding results.

\section{Conclusion}
We studied why masked diffusion language models learn slowly and identified locality bias as a key factor behind their inefficient training dynamics. 
Motivated by this analysis, we propose bell-shaped time sampling, a simple training strategy that emphasizes the middle-context region, and show that it consistently accelerates MDM training and improves performance across various language modeling tasks. 
Although this may appear similar to bell-shaped time sampling in continuous diffusion models~\citep{hang2025improved,karras2022edm,esser2024sd3}, Sec.~\ref{sec:slow_learner} and~\ref{sec:other_alternatives} show that our contribution is a non-trivial parallel contribution. We provide a detailed discussion of this difference in Appendix~\ref{sec:related_works}, and limitations in Appendix~\ref{sec:limitations}.

\section*{Acknowledgements}
This work was supported by the National Research Foundation of Korea under Grant RS-2024-00336454. Satoshi Hayakawa was partially supported by JST CREST (JPMJCR2115).

\bibliographystyle{neurips_2024}
\bibliography{main}

@STRING{NIPS	= "Advances in Neural Information Processing Systems (NeurIPS)"}

@STRING{ICCV	= "IEEE International Conference on Computer Vision (ICCV)"}

@STRING{ICML	= "International Conference on Machine Learning (ICML)"}

@STRING{ICLR	= "International Conference on Learning Representations (ICLR)"}

@STRING{AAAI = "AAAI Conference on Artificial Intelligence (AAAI)"}

@inproceedings{arriola2025bd3lm,
    author = {Marianne Arriola and Aaron Gokaslan and Justin T. Chiu and Zhihan Yang and Zhixuan Qi and Jiaqi Han and Subham Sekhar Sahoo and Volodymyr Kuleshov},
    title = {Block Diffusion: Interpolating Between Autoregressive and Diffusion Language Models},
    booktitle = ICLR,
    year = {2025}
}

@inproceedings{Austin2021,
  title={Structured denoising diffusion models in discrete state-spaces},
  author={Austin, Jacob and Johnson, Daniel D and Ho, Jonathan and Tarlow, Daniel and Van Den Berg, Rianne},
  booktitle=NIPS,
  volume={34},
  pages={17981--17993},
  year={2021}
}

@inproceedings{chelba2013lm1b,
  title     = {{One billion word benchmark for measuring progress in statistical language modeling}},
  author    = {Ciprian Chelba and Tomas Mikolov and Mike Schuster and Qi Ge and Thorsten Brants and Phillipp Koehn and Tony Robinson},
  year      = {2014},
  booktitle = {{Interspeech 2014}},
  pages     = {2635--2639},
  doi       = {10.21437/Interspeech.2014-564},
  issn      = {2958-1796},
}

@inproceedings{cohan2018discourse,
    title = "A Discourse-Aware Attention Model for Abstractive Summarization of Long Documents",
    author = "Cohan, Arman  and
      Dernoncourt, Franck  and
      Kim, Doo Soon  and
      Bui, Trung  and
      Kim, Seokhwan  and
      Chang, Walter  and
      Goharian, Nazli",
    booktitle = "Proceedings of the 2018 Conference of the North {A}merican Chapter of the Association for Computational Linguistics: Human Language Technologies, Volume 2 (Short Papers)",
    year = "2018",
    url = "https://aclanthology.org/N18-2097/",
    doi = "10.18653/v1/N18-2097",
    pages = "615--621",
}

@misc{Gokaslan2019owt,
    title={OpenWebText Corpus},
    author={Gokaslan, Aaron and Cohen, Vanya and Pavlick, Ellie and Tellex, Stefanie},
    howpublished={\url{http://Skylion007.github.io/OpenWebTextCorpus}},
    year={2019}
}

@article{grattafiori2024llama3herdmodels,
  title={The llama 3 herd of models},
  author={Grattafiori, Aaron and Dubey, Abhimanyu and Jauhri, Abhinav and Pandey, Abhinav and Kadian, Abhishek and Al-Dahle, Ahmad and Letman, Aiesha and Mathur, Akhil and Schelten, Alan and Vaughan, Alex and others},
  journal={arXiv preprint arXiv:2407.21783},
  year={2024}
}

@inproceedings{he2022diffusionbert,
  title={Diffusionbert: Improving generative masked language models with diffusion models},
  author={He, Zhengfu and Sun, Tianxiang and Tang, Qiong and Wang, Kuanning and Huang, Xuan-Jing and Qiu, Xipeng},
  booktitle={Proceedings of the 61st annual meeting of the association for computational linguistics (ACL)},
  volume={1},
  pages={4521--4534},
  year={2023}
}

@inproceedings{ho2020denoising,
  title={Denoising diffusion probabilistic models},
  author={Ho, Jonathan and Jain, Ajay and Abbeel, Pieter},
  booktitle=NIPS,
  volume={33},
  pages={6840--6851},
  year={2020}
}

@inproceedings{
hong2025improvingdiscretediffusionunmasking,
title={Improving Discrete Diffusion Unmasking Policies Beyond Explicit Reference Policies},
author={Chunsan Hong and Seonho An and Min-Soo Kim and Jong Chul Ye},
booktitle=ICLR,
year={2026},
}

@inproceedings{Hoogeboom2021b,
  title={Argmax flows and multinomial diffusion: Learning categorical distributions},
  author={Hoogeboom, Emiel and Nielsen, Didrik and Jaini, Priyank and Forr{\'e}, Patrick and Welling, Max},
  booktitle=NIPS,
  volume={34},
  pages={12454--12465},
  year={2021}
}

@inproceedings{kim2025trainworstplanbest,
  title={Train for the Worst, Plan for the Best: Understanding Token Ordering in Masked Diffusions}, 
  author={Jaeyeon Kim and Kulin Shah and Vasilis Kontonis and Sham Kakade and Sitan Chen},
  year={2025},
  booktitle=ICML,
}

@inproceedings{kingma2021variational,
  title={Variational diffusion models},
  author={Kingma, Diederik and Salimans, Tim and Poole, Ben and Ho, Jonathan},
  booktitle=NIPS,
  volume={34},
  pages={21696--21707},
  year={2021}
}

@inproceedings{lou2024discrete,
  title={Discrete Diffusion Modeling by Estimating the Ratios of the Data Distribution},
  author={Aaron Lou and Chenlin Meng and Stefano Ermon},
  booktitle=ICML,
  year={2024}
}

@article{marcus1993building,
  title={Building a large annotated corpus of English: The Penn Treebank},
  author={Marcus, Mitch and Santorini, Beatrice and Marcinkiewicz, Mary Ann},
  journal={Computational linguistics},
  volume={19},
  number={2},
  pages={313--330},
  year={1993}
}

@inproceedings{merity2016pointer,
  title={Pointer sentinel mixture models},
  author={Merity, Stephen and Xiong, Caiming and Bradbury, James and Socher, Richard},
  booktitle=ICLR,
  year={2017}
}

@inproceedings{nie2025llada,
  title={Large language diffusion models},
  author={Nie, Shen and Zhu, Fengqi and You, Zebin and Zhang, Xiaolu and Ou, Jingyang and Hu, Jun and Zhou, Jun and Lin, Yankai and Wen, Ji-Rong and Li, Chongxuan},
  booktitle=NIPS,
  year={2025}
}

@inproceedings{ou2024RADD,
      title={Your Absorbing Discrete Diffusion Secretly Models the Conditional Distributions of Clean Data}, 
      author={Jingyang Ou and Shen Nie and Kaiwen Xue and Fengqi Zhu and Jiacheng Sun and Zhenguo Li and Chongxuan Li},
      booktitle=ICLR,
      year={2025},
}

@inproceedings{paperno2016lambada,
  title={The LAMBADA dataset: Word prediction requiring a broad discourse context},
  author={Paperno, Denis and Kruszewski, Germ{\'a}n and Lazaridou, Angeliki and Pham, Ngoc-Quan and Bernardi, Raffaella and Pezzelle, Sandro and Baroni, Marco and Boleda, Gemma and Fern{\'a}ndez, Raquel},
  booktitle={Proceedings of the 54th annual meeting of the association for computational linguistics (ACL)},
  volume={1},
  pages={1525--1534},
  year={2016}
}

@inproceedings{peebles2023scalable,
  title={Scalable diffusion models with transformers},
  author={Peebles, William and Xie, Saining},
  booktitle=ICCV,
  pages={4195--4205},
  year={2023}
}

@inproceedings{
peng2025papl,
title={Planner Aware Path Learning in Diffusion Language Models Training},
author={Fred Zhangzhi Peng and Zachary Bezemek and Jarrid Rector-Brooks and Shuibai Zhang and Michael M. Bronstein and Anru Zhang and Joey Bose and Alexander Tong},
booktitle=ICLR,
year={2026},
}

@article{raffel2020exploring,
  title={Exploring the limits of transfer learning with a unified text-to-text transformer},
  author={Raffel, Colin and Shazeer, Noam and Roberts, Adam and Lee, Katherine and Narang, Sharan and Matena, Michael and Zhou, Yanqi and Li, Wei and Liu, Peter J},
  journal={Journal of machine learning research},
  volume={21},
  number={140},
  pages={1--67},
  year={2020}
}

@inproceedings{sahoo2024simple,
  title={Simple and effective masked diffusion language models},
  author={Sahoo, Subham and Arriola, Marianne and Schiff, Yair and Gokaslan, Aaron and Marroquin, Edgar and Chiu, Justin and Rush, Alexander and Kuleshov, Volodymyr},
  booktitle=NIPS,
  volume={37},
  pages={130136--130184},
  year={2024}
}

@inproceedings{sahoo2025diffusionduality,
      title={The Diffusion Duality}, 
      author={Subham Sekhar Sahoo and Justin Deschenaux and Aaron Gokaslan and Guanghan Wang and Justin Chiu and Volodymyr Kuleshov},
      year={2025},
      booktitle=ICML
}

@inproceedings{shi2025simplified,
  title={Simplified and generalized masked diffusion for discrete data},
  author={Shi, Jiaxin and Han, Kehang and Wang, Zhe and Doucet, Arnaud and Titsias, Michalis},
  booktitle=NIPS,
  volume={37},
  pages={103131--103167},
  year={2024}
}

@inproceedings{
xie2025vadd,
title={Variational Autoencoding Discrete Diffusion with Enhanced Dimensional Correlations Modeling},
author={Tianyu Xie and Shuchen Xue and Zijin Feng and Tianyang Hu and Jiacheng Sun and Zhenguo Li and Cheng Zhang},
booktitle=ICLR,
year={2026},
}

@article{ye2025dream,
  title={Dream 7B: Diffusion Large Language Models},
  author={Ye, Jiacheng and Xie, Zhihui and Zheng, Lin and Gao, Jiahui and Wu, Zirui and Jiang, Xin and Li, Zhenguo and Kong, Lingpeng},
  journal={arXiv preprint arXiv:2508.15487},
  year={2025}
}

@inproceedings{zhang2015character,
  title={Character-level convolutional networks for text classification},
  author={Zhang, Xiang and Zhao, Junbo and LeCun, Yann},
  booktitle=NIPS,
  volume={28},
  year={2015}
}

@inproceedings{zheng2024maskeddiffusionmodelssecretly,
  title={Masked diffusion models are secretly time-agnostic masked models and exploit inaccurate categorical sampling},
  author={Zheng, Kaiwen and Chen, Yongxin and Mao, Hanzi and Liu, Ming-Yu and Zhu, Jun and Zhang, Qinsheng},
  booktitle=ICLR,
  year={2025}
}

@inproceedings{
liu2025discrete,
title={Discrete Copula Diffusion},
author={Anji Liu and Oliver Broadrick and Mathias Niepert and Guy Van den Broeck},
booktitle=ICLR,
year={2025},
}

@inproceedings{
kim2026flexmdm,
title={Any-Order Flexible Length Masked Diffusion},
author={Jaeyeon Kim and Lee Cheuk Kit and Carles Domingo-Enrich and Yilun Du and Sham M. Kakade and Timothy Ngotiaoco and Sitan Chen and Michael Samuel Albergo},
booktitle=ICLR,
year={2026},
}

@inproceedings{
xue2025anyordergpt,
title={Any-Order {GPT} as Masked Diffusion Model: Decoupling Formulation and Architecture},
author={Shuchen Xue and Tianyu Xie and Tianyang Hu and Zijin Feng and Jiacheng Sun and Kenji Kawaguchi and Zhenguo Li and Zhi-Ming Ma},
booktitle={ES-FoMo III: 3rd Workshop on Efficient Systems for Foundation Models},
year={2025},
url={https://openreview.net/forum?id=KbRxn8fzrY}
}

@inproceedings{
zheng2026cadd,
title={Continuously Augmented Discrete Diffusion model for Categorical Generative Modeling},
author={Huangjie Zheng and Shansan Gong and Ruixiang ZHANG and Tianrong Chen and Jiatao Gu and Mingyuan Zhou and Navdeep Jaitly and Yizhe Zhang},
booktitle={The Fourteenth International Conference on Learning Representations},
year={2026},
url={https://openreview.net/forum?id=JNAZ3e7Bwt}
}

@inproceedings{joshi2017triviaqa,
  title={Triviaqa: A large scale distantly supervised challenge dataset for reading comprehension},
  author={Joshi, Mandar and Choi, Eunsol and Weld, Daniel S and Zettlemoyer, Luke},
  booktitle={Proceedings of the 55th Annual Meeting of the Association for Computational Linguistics (Volume 1: Long Papers)},
  pages={1601--1611},
  year={2017}
}

@article{zhou2025ccdd,
  title={Coevolutionary continuous discrete diffusion: Make your diffusion language model a latent reasoner},
  author={Zhou, Cai and Yang, Chenxiao and Hu, Yi and Wang, Chenyu and Zhang, Chubin and Zhang, Muhan and Mackey, Lester and Jaakkola, Tommi and Bates, Stephen and Zhang, Dinghuai},
  journal={arXiv preprint arXiv:2510.03206},
  year={2025}
}

@article{pynadath2025candi,
  title={Candi: Hybrid discrete-continuous diffusion models},
  author={Pynadath, Patrick and Shi, Jiaxin and Zhang, Ruqi},
  journal={arXiv preprint arXiv:2510.22510},
  year={2025}
}

@article{hong2026lomdm,
  title={Unifying Masked Diffusion Models with Various Generation Orders and Beyond},
  author={Hong, Chunsan and Lee, Sanghyun and Ye, Jong Chul},
  journal={arXiv preprint arXiv:2602.02112},
  year={2026}
}

@inproceedings{zhao2025d1scalingreasoningdiffusion,
      title={d1: Scaling Reasoning in Diffusion Large Language Models via Reinforcement Learning}, 
      author={Siyan Zhao and Devaansh Gupta and Qinqing Zheng and Aditya Grover},
      year={2025},
      booktitle=NIPS,
}

@article{lee2025lookahead,
  title={Lookahead unmasking elicits accurate decoding in diffusion language models},
  author={Lee, Sanghyun and Kim, Seungryong and Park, Jongho and Park, Dongmin},
  journal={arXiv preprint arXiv:2511.05563},
  year={2025}
}

@article{kim2026diffusionRag2,
  title={Adaptive Guidance for Retrieval-Augmented Masked Diffusion Models},
  author={Kim, Jaemin and Ye, Jong Chul},
  journal={arXiv preprint arXiv:2603.17677},
  year={2026}
}

@article{yu2026diffusionRag1,
  title={Unlocking the Potentials of Retrieval-Augmented Generation for Diffusion Language Models},
  author={Yu, Chuanyue and Wang, Jiahui and Li, Yuhan and Chang, Heng and Lan, Ge and Sun, Qingyun and Li, Jia and Li, Jianxin and Zhang, Ziwei},
  journal={arXiv preprint arXiv:2601.11342},
  year={2026}
}

@inproceedings{
prabhudesai2025LimitedData,
title={Diffusion Beats Autoregressive in Data-Constrained Settings},
author={Mihir Prabhudesai and Mengning Wu and Amir Zadeh and Katerina Fragkiadaki and Deepak Pathak},
booktitle=NIPS,
year={2025},
}

@article{sahoo2026scaling,
  title={Scaling Beyond Masked Diffusion Language Models},
  author={Sahoo, Subham Sekhar and Lemercier, Jean-Marie and Yang, Zhihan and Deschenaux, Justin and Liu, Jingyu and Thickstun, John and Jukic, Ante},
  journal={arXiv preprint arXiv:2602.15014},
  year={2026}
}

@inproceedings{
jia2026bringing,
title={Bringing Stability to Diffusion: Decomposing and Reducing Variance of Training Masked Diffusion Models},
author={Mengni Jia and Mengyu Zhou and Yihao Liu and xiaoxi jiang and guanjunjiang},
booktitle=ICLR,
year={2026},
}

@inproceedings{
karras2022edm,
title={Elucidating the Design Space of Diffusion-Based Generative Models},
author={Tero Karras and Miika Aittala and Timo Aila and Samuli Laine},
booktitle=NIPS,
year={2022},
}

@inproceedings{hang2025improved,
  title={Improved noise schedule for diffusion training},
  author={Hang, Tiankai and Gu, Shuyang and Bao, Jianmin and Wei, Fangyun and Chen, Dong and Geng, Xin and Guo, Baining},
  booktitle=ICCV,
  pages={4796--4806},
  year={2025}
}

@inproceedings{esser2024sd3,
  title={Scaling rectified flow transformers for high-resolution image synthesis},
  author={Esser, Patrick and Kulal, Sumith and Blattmann, Andreas and Entezari, Rahim and M{\"u}ller, Jonas and Saini, Harry and Levi, Yam and Lorenz, Dominik and Sauer, Axel and Boesel, Frederic and others},
  booktitle=ICML,
  year={2024}
}

@misc{lozhkov2024fineweb-edu,
    author       = { Lozhkov, Anton and Ben Allal, Loubna and von Werra, Leandro and Wolf, Thomas },  
    title        = { FineWeb-Edu: the Finest Collection of Educational Content }, 
    year         = 2024,  
    url          = { https://huggingface.co/datasets/HuggingFaceFW/fineweb-edu },  
    doi          = { 10.57967/hf/2497 },
    publisher    = { Hugging Face }
}

@inproceedings{zellers2019hellaswag,
  title={Hellaswag: Can a machine really finish your sentence?},
  author={Zellers, Rowan and Holtzman, Ari and Bisk, Yonatan and Farhadi, Ali and Choi, Yejin},
  booktitle={Proceedings of the 57th annual meeting of the association for computational linguistics},
  pages={4791--4800},
  year={2019}
}

@inproceedings{bisk2020piqa,
  title={Piqa: Reasoning about physical commonsense in natural language},
  author={Bisk, Yonatan and Zellers, Rowan and Gao, Jianfeng and Choi, Yejin and others},
  booktitle={Proceedings of the AAAI conference on artificial intelligence},
  volume={34},
  number={05},
  pages={7432--7439},
  year={2020}
}

@inproceedings{sap2019social,
  title={Social IQa: Commonsense reasoning about social interactions},
  author={Sap, Maarten and Rashkin, Hannah and Chen, Derek and Le Bras, Ronan and Choi, Yejin},
  booktitle={Proceedings of the 2019 conference on empirical methods in natural language processing and the 9th international joint conference on natural language processing (EMNLP-IJCNLP)},
  pages={4463--4473},
  year={2019}
}

@article{sakaguchi2021winogrande,
  title={Winogrande: An adversarial winograd schema challenge at scale},
  author={Sakaguchi, Keisuke and Bras, Ronan Le and Bhagavatula, Chandra and Choi, Yejin},
  journal={Communications of the ACM},
  volume={64},
  number={9},
  pages={99--106},
  year={2021},
  publisher={ACM New York, NY, USA}
}

@inproceedings{mihaylov2018can,
  title={Can a suit of armor conduct electricity? a new dataset for open book question answering},
  author={Mihaylov, Todor and Clark, Peter and Khot, Tushar and Sabharwal, Ashish},
  booktitle={Proceedings of the 2018 conference on empirical methods in natural language processing},
  pages={2381--2391},
  year={2018}
}

@inproceedings{hayakawa2025distillation,
  title={Distillation of discrete diffusion through dimensional correlations},
  author={Hayakawa, Satoshi and Takida, Yuhta and Imaizumi, Masaaki and Wakaki, Hiromi and Mitsufuji, Yuki},
  booktitle={Proceedings of the 42nd International Conference on Machine Learning},
  pages={22259--22297},
  year={2025},
  organization={PMLR}
}

@article{hayakawa2025demystifyingmaskgitsamplerbeyond,
      title={Demystifying {MaskGIT} Sampler and Beyond: Adaptive Order Selection in Masked Diffusion}, 
      author={Satoshi Hayakawa and Yuhta Takida and Masaaki Imaizumi and Hiromi Wakaki and Yuki Mitsufuji},
      year={2026},
      journal={Transactions on Machine Learning Research}
}

@inproceedings{
gong2024scaling,
title={Scaling Diffusion Language Models via Adaptation from Autoregressive Models},
author={Shansan Gong and Shivam Agarwal and Yizhe Zhang and Jiacheng Ye and Lin Zheng and Mukai Li and Chenxin An and Peilin Zhao and Wei Bi and Jiawei Han and Hao Peng and Lingpeng Kong},
booktitle=ICLR,
year={2025},
}

@article{bie2025llada2,
  title={Llada2. 0: Scaling up diffusion language models to 100b},
  author={Bie, Tiwei and Cao, Maosong and Chen, Kun and Du, Lun and Gong, Mingliang and Gong, Zhuochen and Gu, Yanmei and Hu, Jiaqi and Huang, Zenan and Lan, Zhenzhong and others},
  journal={arXiv preprint arXiv:2512.15745},
  year={2025}
}

@inproceedings{
nie2024scaling,
title={Scaling up Masked Diffusion Models on Text},
author={Shen Nie and Fengqi Zhu and Chao Du and Tianyu Pang and Qian Liu and Guangtao Zeng and Min Lin and Chongxuan Li},
booktitle=ICLR,
year={2025},
}

@article{song2025seed,
  title={Seed diffusion: A large-scale diffusion language model with high-speed inference},
  author={Song, Yuxuan and Zhang, Zheng and Luo, Cheng and Gao, Pengyang and Xia, Fan and Luo, Hao and Li, Zheng and Yang, Yuehang and Yu, Hongli and Qu, Xingwei and others},
  journal={arXiv preprint arXiv:2508.02193},
  year={2025}
}

@inproceedings{
gong2026diffucoder,
title={DiffuCoder: Understanding and Improving Masked Diffusion Models for Code Generation},
author={Shansan Gong and Ruixiang ZHANG and Huangjie Zheng and Jiatao Gu and Navdeep Jaitly and Lingpeng Kong and Yizhe Zhang},
booktitle={The Fourteenth International Conference on Learning Representations},
year={2026},
url={https://openreview.net/forum?id=58NA3unZj5}
}

@article{xie2025dreamcoder,
  title={Dream-coder 7b: An open diffusion language model for code},
  author={Xie, Zhihui and Ye, Jiacheng and Zheng, Lin and Gao, Jiahui and Dong, Jingwei and Wu, Zirui and Zhao, Xueliang and Gong, Shansan and Jiang, Xin and Li, Zhenguo and others},
  journal={arXiv preprint arXiv:2509.01142},
  year={2025}
}

@article{hui2024qwen25coder,
  title={Qwen2. 5-coder technical report},
  author={Hui, Binyuan and Yang, Jian and Cui, Zeyu and Yang, Jiaxi and Liu, Dayiheng and Zhang, Lei and Liu, Tianyu and Zhang, Jiajun and Yu, Bowen and Lu, Keming and others},
  journal={arXiv preprint arXiv:2409.12186},
  year={2024}
}

@inproceedings{futrell2019localitybias,
    title = "Information-theoretic locality properties of natural language",
    author = "Futrell, Richard",
    booktitle = "Proceedings of the First Workshop on Quantitative Syntax (Quasy, SyntaxFest 2019)",
    year = "2019",
    publisher = "Association for Computational Linguistics",
    url = "https://aclanthology.org/W19-7902/",
    doi = "10.18653/v1/W19-7902",
    pages = "2--15"
}

@article{zheng2023judging,
  title={Judging llm-as-a-judge with mt-bench and chatbot arena},
  author={Zheng, Lianmin and Chiang, Wei-Lin and Sheng, Ying and Zhuang, Siyuan and Wu, Zhanghao and Zhuang, Yonghao and Lin, Zi and Li, Zhuohan and Li, Dacheng and Xing, Eric and others},
  journal={Advances in neural information processing systems},
  volume={36},
  pages={46595--46623},
  year={2023}
}

@misc{li2023alpacaeval,
  title={Alpacaeval: An automatic evaluator of instruction-following models},
  author={Li, Xuechen and Zhang, Tianyi and Dubois, Yann and Taori, Rohan and Gulrajani, Ishaan and Guestrin, Carlos and Liang, Percy and Hashimoto, Tatsunori B},
  year={2023},
  howpublished = {\url{https://github.com/tatsu-lab/alpaca_eval}}
}

@article{hurst2024gpt,
  title={Gpt-4o system card},
  author={Hurst, Aaron and Lerer, Adam and Goucher, Adam P and Perelman, Adam and Ramesh, Aditya and Clark, Aidan and Ostrow, AJ and Welihinda, Akila and Hayes, Alan and Radford, Alec and others},
  journal={arXiv preprint arXiv:2410.21276},
  year={2024}
}

@inproceedings{
wang2024evaluatingdesignspacediffusionbased,
title={Evaluating the design space of diffusion-based generative models},
author={Yuqing Wang and Ye He and Molei Tao},
booktitle=NIPS,
year={2024},
}

@inproceedings{
havasi2025editflow,
title={Edit Flows: Variable Length Discrete Flow Matching with Sequence-Level Edit Operations},
author={Marton Havasi and Brian Karrer and Itai Gat and Ricky T. Q. Chen},
booktitle=NIPS,
year={2025},
}

@inproceedings{
hersche2026softmask,
title={Soft-Masked Diffusion Language Models},
author={Michael Hersche and Samuel Moor-Smith and Thomas Hofmann and Abbas Rahimi},
booktitle=ICLR,
year={2026},
}

@article{krzakala2009hiding,
  title={Hiding quiet solutions in random constraint satisfaction problems},
  author={Krzakala, Florent and Zdeborov{\'a}, Lenka},
  journal={Physical review letters},
  volume={102},
  number={23},
  pages={238701},
  year={2009},
  publisher={APS}
}

@inproceedings{choi2022perception,
  title={Perception prioritized training of diffusion models},
  author={Choi, Jooyoung and Lee, Jungbeom and Shin, Chaehun and Kim, Sungwon and Kim, Hyunwoo and Yoon, Sungroh},
  booktitle={Proceedings of the IEEE/CVF conference on computer vision and pattern recognition},
  pages={11472--11481},
  year={2022}
}

@inproceedings{nichol2021improved,
  title={Improved denoising diffusion probabilistic models},
  author={Nichol, Alexander Quinn and Dhariwal, Prafulla},
  booktitle={International conference on machine learning},
  pages={8162--8171},
  year={2021},
  organization={PMLR}
}

@article{radford2019language,
  title={Language Models are Unsupervised Multitask Learners},
  author={Radford, Alec and Wu, Jeff and Child, Rewon and Luan, David and Amodei, Dario and Sutskever, Ilya},
  year={2019}
}

@article{Su2021rope,
author = {Su, Jianlin and Ahmed, Murtadha and Lu, Yu and Pan, Shengfeng and Bo, Wen and Liu, Yunfeng},
title = {RoFormer: Enhanced transformer with Rotary Position Embedding},
year = {2024},
issue_date = {Feb 2024},
publisher = {Elsevier Science Publishers B. V.},
address = {NLD},
volume = {568},
number = {C},
issn = {0925-2312},
url = {https://doi.org/10.1016/j.neucom.2023.127063},
doi = {10.1016/j.neucom.2023.127063},
journal = {Neurocomput.},
month = feb,
numpages = {12}
}

@article{sanh2021multitask,
  title={Multitask prompted training enables zero-shot task generalization},
  author={Sanh, Victor and Webson, Albert and Raffel, Colin and Bach, Stephen H and Sutawika, Lintang and Alyafeai, Zaid and Chaffin, Antoine and Stiegler, Arnaud and Scao, Teven Le and Raja, Arun and others},
  journal={arXiv preprint arXiv:2110.08207},
  year={2021}
}

@inproceedings{mostafazadeh2016corpus,
  title={A corpus and cloze evaluation for deeper understanding of commonsense stories},
  author={Mostafazadeh, Nasrin and Chambers, Nathanael and He, Xiaodong and Parikh, Devi and Batra, Dhruv and Vanderwende, Lucy and Kohli, Pushmeet and Allen, James},
  booktitle={Proceedings of the 2016 Conference of the North American Chapter of the Association for Computational Linguistics: Human Language Technologies},
  pages={839--849},
  year={2016}
}

@inproceedings{lin2004rouge,
  title={Rouge: A package for automatic evaluation of summaries},
  author={Lin, Chin-Yew},
  booktitle={Text summarization branches out},
  pages={74--81},
  year={2004}
}

@article{fan2026stable,
  title={Stable-DiffCoder: Pushing the Frontier of Code Diffusion Large Language Model},
  author={Fan, Chenghao and Heng, Wen and Li, Bo and Liu, Sichen and Song, Yuxuan and Su, Jing and Qu, Xiaoye and Shen, Kai and Wei, Wei},
  journal={arXiv preprint arXiv:2601.15892},
  year={2026}
}

@inproceedings{tae2025tess2,
  title={Tess 2: A large-scale generalist diffusion language model},
  author={Tae, Jaesung and Ivison, Hamish and Kumar, Sachin and Cohan, Arman},
  booktitle={Proceedings of the 63rd Annual Meeting of the Association for Computational Linguistics (Volume 1: Long Papers)},
  pages={21171--21188},
  year={2025}
}

\clearpage
\appendix

\section{Background}\label{sec:related_works}
\subsection{Related Work}
\noindent\textbf{Discrete diffusion models.} Discrete diffusion models~\citep{Austin2021,Hoogeboom2021b} have emerged as a potential alternative framework for language modeling.
The main objective of discrete diffusion modeling is to model the discrete data distribution via a continuous-time diffusion process.
There are broadly two types of forward corruption processes in discrete diffusion models:
(1) uniform corruption~\citep{lou2024discrete,sahoo2025diffusionduality}, which replaces tokens with uniformly random tokens, and
(2) masking corruption~\citep{sahoo2024simple,shi2025simplified}, which replaces tokens with a special [\textsc{MASK}] token.
Among these, masked diffusion models (MDMs) have emerged as a leading class of discrete diffusion models for text generation~\citep{sahoo2024simple}.

\noindent\textbf{Various frameworks based on masked diffusion models.} One of the most active recent research directions has been the development of new frameworks that modify the diffusion process itself in order to overcome the theoretical limitations of MDMs. Such efforts include block generation scheme~\citep{arriola2025bd3lm}, modeling the joint probability distribution~\citep{xie2025vadd,liu2025discrete,hayakawa2025distillation}, introducing new operations such as editing~\citep{kim2026flexMDM,havasi2025editflow}, relaxing the discrete tokenization by augmenting the model with additional latent~\citep{zheng2026cadd,zhou2025ccdd,pynadath2025candi,hersche2026softmask}, and learned generation order rather than random order~\citep{peng2025papl,hong2026lomdm}.

\noindent\textbf{Standard masked diffusion models.}
Despite these extensions, the simple masking-based MDM framework~\citep{ou2024RADD,sahoo2024simple} remains the most widely used backbone for large-scale text generation, largely due to its simplicity. Recent large-scale models such as LLaDA~\citep{nie2025llada}, Seed Diffusion~\citep{song2025seed}, and Dream~\citep{ye2025dream} retain the standard generation scheme, where a fully masked sequence is gradually converted into tokens. Many recent advances, including inference-time decoding strategies~\citep{lee2025lookahead,hayakawa2025demystifyingmaskgitsamplerbeyond}, reinforcement learning methods~\citep{hong2025improvingdiscretediffusionunmasking,zhao2025d1scalingreasoningdiffusion}, and retrieval-augmented generation~\citep{yu2026diffusionRag1,kim2026diffusionRag2}, are also built on this basic framework. Empirically, standard MDMs have become increasingly competitive with autoregressive models (ARMs): LLaDA~\citep{nie2025llada} and Dream~\citep{ye2025dream} outperform same-sized LLaMA3 models~\citep{grattafiori2024llama3herdmodels}, while DiffuCoder~\citep{gong2026diffucoder} and Dream-Coder~\citep{xie2025dreamcoder} match same-sized Qwen2.5-Coder models~\citep{hui2024qwen25coder}.

\noindent\textbf{Training inefficiency of masked diffusion models.}
Meanwhile, MDMs require $\sim16\times$ more training compute than ARMs to reach the same validation loss~\citep{prabhudesai2025LimitedData,sahoo2026scaling}, which may become a major obstacle when scaling MDMs to real-world LLMs. Several prior works~\citep{kim2025trainworstplanbest,xue2025anyordergpt} attribute this slow learning to the left-to-right (L2R) bias of language; yet, \textit{how} this bias makes MDM training-inefficient remains underexplored. 
On the other hand, Jia et al.~\citep{jia2026bringing} suggest that the variance of the batch gradient estimator may be a key issue; yet lacks an experimental analysis of their hypothesis. 

The work of Kim~\etal~\citep{kim2025trainworstplanbest} is closely related to ours in that it also tackles the training difficulty of MDMs. They compare $\sigma$-ARMs with permutations $\sigma$ that are close to the L2R order, fully random permutations $\sigma$, and MDMs, showing that order-agnostic training, as in MDMs (and different from AoARMs), is difficult and that L2R bias further hinders learning. In contrast, our work compares the structurally similar $\sigma$-ARM and AoARM to identify the dominant factor among order-space complexity and locality bias. Furthermore, our study empirically analyzes how locality bias makes MDM training difficult, whereas Kim~\etal take a fundamentally different approach by proving, based on a physics-inspired conjecture~\citep{krzakala2009hiding}, that order-agnostic training inevitably includes unlearnable subproblems. Finally, another key distinction is that Kim~\etal propose an inference-time strategy, whereas we propose a training-time strategy.

\noindent\textbf{Bell-shaped time sampling in continuous diffusion models.} In continuous diffusion modeling, bell-shaped time sampling has also been studied as an effective training strategy~\citep{hang2025improved,karras2022edm,esser2024sd3}. 
At first glance, our method may look similar to these approaches, but the underlying reason it works is fundamentally different. 
In image diffusion, it is proven that bell-shaped sampling aids training since both low- and high-SNR regions are difficult to train~\citep{wang2024evaluatingdesignspacediffusionbased}. 
In MDMs, however, we identify a distinct mechanism: low-context samples are learned early and subsequently waste compute, whereas high-context samples are inefficient because locality bias makes many observed tokens redundant. 
Moreover, these continuous-diffusion-style alternatives have not been systematically studied in MDMs. 
In Sec.~\ref{sec:other_alternatives}, we examine such middle-biased training variants in the MDM setting and find that VDM-style noise-schedule modification~\citep{kingma2021variational} and importance weighting~\citep{ho2020denoising,nichol2021improved,choi2022perception} are ineffective compared to direct bell-shaped time sampling. 
These observations suggest that MDMs require their own analysis of training inefficiency and a sampling strategy tailored to the locality structure of language. 

\subsection{Various Language Modeling Frameworks in Detail}\label{sec:shared_loss}
\noindent\textbf{Masked diffusion models.}
Masked diffusion models (MDMs)~\citep{sahoo2024simple,shi2025simplified} define the forward corruption process using the absorbing mask strategy: once a token is masked, it remains masked throughout the remaining process.
For the diffusion process, define the time interval as $t\in\mathcal{T}=[0,1]$ where we corrupt the data from $t=0$ (least noisy) to $t=1$ (most noisy).  
Formally, the forward process $q(\mathbf{x}_t^{i} \mid \mathbf{x})=q_{t|0}($ is designed as follows:
\begin{equation}
q(\mathbf{x}_t^{i} \mid \mathbf{x})=q_{t|0}(\cdot \mid x^i) = \mathrm{Cat}\!\left(\alpha_t \e_{x^i} + (1 - \alpha_t)\e_m\right),\nonumber
\end{equation}
where the forward process gradually adds noise as $t$ grows, \ie, the noise scheduler should satisfy $\alpha_0=1$, $\alpha_1=0$, $\alpha_t'<0$. The noise scheduler~\citep{sahoo2024simple,shi2025simplified} is typically set to linear scheduler, $\alpha_t = 1-t$. Following prior works~\citep{sahoo2024simple, hong2026lomdm}, discretize the time interval $\mathcal{T}$ with $T+1$ steps, and define $s(\tau)=\tau/(T+1)$ and $t(\tau)=(\tau+1)/(T+1)$. Then the generative distribution is divided into $T$ diffusion reverse steps ($\mathbf{x}_{t(T)}\rightarrow\dots\rightarrow\mathbf{x}_{t(0)}$) and 1 reconstruction step ($\mathbf{x}_{t(0)}\rightarrow\mathbf{x}$). 
The true reverse posterior based on corresponding forward process is derived as follows:

\begin{align}
&q(\mathbf{x}_s^{i} \mid \mathbf{x}_t, \mathbf{x})=q(\x_s^{i} \mid \mathbf{x}_t^{i}, \mathbf{x})=
\begin{cases}
\mathrm{Cat}(\e_{x_t^i}), & \text{if } \mathbf{x}_t^{i} \neq \mathbf{m}, \\[6pt]
\mathrm{Cat}\!\left(\tfrac{1-\alpha_s}{1-\alpha_t}\e_\mask + \tfrac{\alpha_s - \alpha_t}{1 - \alpha_t}\e_{x^i}\right), & \text{if } x_t^i = \mathbf{m}.
\end{cases}\nonumber
\end{align}
where we drop $\tau$ in $s(\tau)$ and $t(\tau)$ for brevity.
To mimic the true reverse posterior, with the diffusion network $\model(\x_t, t)\in\Delta^{V+1}$ that predicts the clean token $x^i$ over the vocabulary, the parametrized reverse process is designed as follows:

\begin{align}
&p_{\theta}(\x_s^{i} \mid \x_t)
= q\!\left(\x_s^{i} \mid \x_t, \mathbf{x} = \mathbf{x}_{\theta}(\x_t, t)\right)= 
\begin{cases}
\mathrm{Cat}(\e_{x_t^i}), & \text{if } \mathbf{x}_t^{i} \neq \mathbf{m}, \\[6pt]
\mathrm{Cat}\!\left(\tfrac{1-\alpha_s}{1-\alpha_t}\e_\mask + \tfrac{\alpha_s - \alpha_t}{1 - \alpha_t}\model(\x_t, t)\right),  & \text{if } x_t^i = \mathbf{m}.\nonumber
\end{cases}
\end{align}
where $s<t$. Note that conventional MDM~\citep{sahoo2024simple,nie2025llada,ye2025dream,hong2026lomdm} typically set model output of mask probability as zero.
Due to curse of dimensionality~\citep{lou2024discrete}, MDM models the reverse process as a token-wise conditionally independent distribution rather than capturing full joint distribution, \ie,
$p_\theta(\x_s \mid \x_t)
= \prod_{i=1}^L p_\theta(\x_s^{i} \mid \x_t)$. 
In this regard, the model distribution is expressed as $p_\theta(\mathbf{x})=\sum_{\mathbf{x}_{t(0:T)}} p_\theta(\mathbf{x}_{t(T)})p_\theta(\mathbf{x}|\mathbf{x}_{t(0)})\prod_{\tau=1}^{T} p_\theta(\x_{s(\tau)}|\x_{t(\tau)})$ in the discrete-time case.
Finally, with $T\rightarrow\infty$, the continuos time NELBO is given as follows:
\begin{align*}
\mathcal{L}_{\mathrm{MDM}}=\int_{0}^{1}
\tfrac{-\alpha_t'}{1-\alpha_t}\mathbb{E}_{\x\sim p_{\mathrm{data}},\x_t\sim q_{t|0}}
\Bigl[
-\sum_{i:x_t^i=\mask}
\log p_\theta(x^i|\x_t,t)\Big]dt.
\end{align*}
where $p_\theta(x^i|\x_t,t)=\model(\x_t,t)$ and $\tfrac{-\alpha_t'}{1-\alpha_t}>0$. Prior works~\citep{sahoo2024simple,zheng2024maskeddiffusionmodelssecretly,ou2024RADD} have reported time-agnostic parameterization gives better performance than that of time-variant parameterization, and inspired by this, MDM typically utilizes time-agnostic network~\citep{sahoo2024simple,nie2025llada,ye2025dream,hong2026lomdm}, \ie, $\model(\x_t,t)=\model(\x_t)$.

\noindent\textbf{Various views of MDM objective.} Zheng~\etal~\citep{zheng2024maskeddiffusionmodelssecretly} have shown that the time-variant MDM objective can be rewritten with the time-agnostic form, and Kim~\etal~\citep{kim2025trainworstplanbest} have shown that the time-agnostic MDM objective can be rewritten as any-order autoregressive objective form. For a sequence $\x$ and a mask indicator $\Mask\in\{0,1\}^L$, define the masked sequence $\x[\Mask]$ that is obtained by replacing $x^i$ with $\mask$ whenever $M^i=1$, while leaving it unchanged whenever $M^i=0$. Let $\Mask_n$ as mask indicator with mask number of $n$, \ie, $\sum_iM_n^i=n$, and $\mathcal{M}_{n}$ as set of all possible combinatorial $\Mask_n$. Then, Zheng~\etal have proven that:

\begin{lemma}[Proposition 3.1 of Zheng~\etal~\citep{zheng2024maskeddiffusionmodelssecretly}]\label{lemma:zheng-1}
    For the time-variant diffusion network $\model(\cdot, t)\in\Delta^{V+1}$, the theoretical NELBO of MDM, $\mathcal{L}_{\mathrm{MDM}}$ (\eqref{eq:loss_mdlm}), can be rewritten as follows:
    \begin{align}
\mathcal L_{\mathrm{MDM}}
    =
    \mathbb E_{\x\sim p_{\mathrm{data}}}\Big[\sum_{N_c=0}^{L-1}\mathbb{E}_{
    \,\Mask_{L-N_c}\sim U(\mathcal M_{L-N_c}),
    \,i\sim U(i:M_{L-N_c}^i=1)}
    \big[-\log \bar p_\theta(x^i\mid \x[\Mask_{L-N_c}])\big]\Big],
\end{align}
where $\log \bar p_\theta(x^i\mid \x[\Mask_{L-N_c}])
    :=
    \mathbb E_{\alpha_{N_c}}
    \big[
    \log p_\theta(x^i\mid \mathbf{x}_t=\x[\Mask_{L-N_c}],t=\alpha_t^{-1}(\alpha_{N_c}))
    \big]$, $\alpha_{N_c}\sim \mathrm{Beta}(N_c+1,L-N_c)$, and $p_\theta(x^i\mid \x[\Mask_{L-N_c}],t)=\mu_\theta^i(\x[\Mask_{L-N_c}],t)$.
\end{lemma}
Note that the original version of Proposition 3.1 of Zheng~\etal includes a $1/(L-N_c)$ scaling term, which is accounted for in our notation by the expectation over the sampled index $i$. We note that $\bar p_\theta$ is not normalized; in particular,
$\sum_{a\in[V]} \bar p_\theta(x^i=a\mid\mathbf{x}[\Mask])\neq 1$.

Motivated by above lemma, Kim~\etal have shown that:
\begin{lemma}[Proposition 2.1 of Kim~\etal~\citep{kim2025trainworstplanbest}]\label{lemma:kim-1}
    For the time-agnostic diffusion network $\model(\cdot)\in\Delta^{V+1}$, the theoretical NELBO of MDM, $\mathcal{L}_{\mathrm{MDM}}$ (\eqref{eq:loss_mdlm}), can be rewritten as follows:
    \begin{align}
\mathcal{L}_{\mathrm{MDM}}=\mathbb{E}_{\x\sim p_{\mathrm{data}}}\Big[\sum_{N_c=0}^{L-1}\mathbb{E}_{\Mask_{L-N_c}\sim \mathrm{U}(\mathcal{M}_{L-N_c}),i\sim \mathrm{U}(i:M_{L-N_c}^i=1)}[-\log p_\theta(x^i|\x[\Mask_{L-N_c}])]\Big].\label{eq:eq_loss_mdlm_simple}
\end{align}
where \(p_\theta(x^i\mid \x[\Mask_{L-N_c}])=\mu_\theta^i(\x[\Mask_{L-N_c}])\).
\end{lemma}

Again, note that the original version of Proposition 2.1 of Kim~\etal includes a $\tfrac{1}{\binom{L}{L-N_c}\cdot(L-N_c)}$ scaling term, which is accounted for in our notation by the expectation over the sampled index $i$ and sampled mask $\Mask_{L-N_c}$.

\noindent\textbf{Shared format of learning objectives.}
Recall the shared objective in \eqref{eq:loss_shared}:

\begin{align}
\mathcal{L}_{\mathcal{F}}
=\sum_{N_c=0}^{L-1}\mathcal{L}_{\mathcal{F}}^{N_c}
,\quad\mathcal{L}_{\mathcal{F}}^{N_c}\coloneqq
\mathbb{E}_{\mathbf{x}\sim p_{\mathrm{data}},\mathbf{z}_{N_c}\sim p_{\mathcal{F}}(\mathbf{z}\mid \mathbf{x},N_c),i\sim p_{\mathcal{F}}(i\mid \mathbf{x},\mathbf{z}_{N_c})}
[-\log p_{\theta}(x^i\mid \mathbf{z}_{N_c})],
\label{eq:loss_shared_appendix}
\end{align}
where \(N_c\) is the
number of context tokens, \(\mathbf{z}_{N_c}\) is the partially observed context,
and \(p_{\mathcal F}(\mathbf{z}\mid \x,N_c)\) and
\(p_{\mathcal F}(i\mid \x,\mathbf{z}_{N_c})\) specify how each framework constructs
the context and selects the target position. We spell out these
framework-specific choices in Table~\ref{tab:shared_format_objectives}. Since
\(N_c\) counts observed tokens whereas \(\Mask_n\) counts masked positions, the
MDM row uses \(\Mask_{L-N_c}\).

\begin{table}[t]
\centering
\small
\caption{Framework-specific choices that instantiate the shared objective
in \eqref{eq:loss_shared}. Entries are written as sampling rules; deterministic
entries correspond to point masses at the displayed value.}
\label{tab:shared_format_objectives}
\begin{adjustbox}{width=\textwidth}
\begin{tabular}{lccc}
\toprule
Framework \(\mathcal F\)
&
Partially observed latent \(\mathbf{z}_{N_c}\)
&
\(p_{\mathcal F}(\mathbf{z}\mid \x,N_c)\)
&
\(p_{\mathcal F}(i\mid \x,\mathbf{z}_{N_c})\)
\\
\midrule
\rowcolor{gray!8}
ARM
&
\(\x^{<N_c+1}\)
&
\(\mathbf{z}_{N_c}=\x^{<N_c+1}\)
&
\(i=N_c+1\)
\\[0.8ex]
\(\sigma\)-ARM
&
\(\x^{<\sigma_{N_c+1}}\)
&
\(\mathbf{z}_{N_c}=\x^{<\sigma_{N_c+1}}\)
&
\(i=\sigma_{N_c+1}\)
\\[0.8ex]
\rowcolor{gray!8}
AoARM
&
\((\x^{<\sigma_{N_c+1}},\sigma)\)
&
\(\sigma\sim U(S_L),\quad
\mathbf{z}_{N_c}=(\x^{<\sigma_{N_c+1}},\sigma)\)
&
\(i=\sigma_{N_c+1}\)
\\[1.0ex]
Time-agnostic MDM
&
\(\x[\Mask_{L-N_c}]\)
&
\(
\begin{gathered}
\Mask_{L-N_c}\sim U(\mathcal M_{L-N_c}),\\
\mathbf{z}_{N_c}=\x[\Mask_{L-N_c}]
\end{gathered}
\)
&
\(
i\sim U(\{i:z_{N_c}^i=m\})
\)
\\[2.2ex]
\rowcolor{gray!8}
Time-variant MDM
&
\(
\left(
\x[\Mask_{L-N_c}],\alpha_{N_c}
\right)
\)
&
\(
\begin{gathered}
\Mask_{L-N_c}\sim U(\mathcal M_{L-N_c}),\quad
\alpha_{N_c}\sim \mathrm{Beta}(N_c+1,L-N_c),\\
\mathbf{z}_{N_c}=(\x[\Mask_{L-N_c}],\alpha_{N_c})
\end{gathered}
\)
&
\(
i\sim U(\{i:z_{N_c}^i=m\})
\)
\\
\bottomrule
\end{tabular}
\end{adjustbox}
\end{table}

Substituting the ARM row into \eqref{eq:loss_shared_appendix} recovers
\eqref{eq:loss_ARM}. Similarly, for a fixed permutation \(\sigma\),
\[
    \sum_{N_c=0}^{L-1}\mathcal L_{\sigma\text{-}\mathrm{ARM}}^{N_c}
    =
    \mathbb E_{\x\sim p_{\mathrm{data}}}
    \sum_{r=1}^{L}
    \Big[-\log p_\theta^{\sigma\textrm{-ARM}}
    \!\left(x^{\sigma_r}\mid \x^{<\sigma_r}\right)\Big],
\]
and averaging the same construction over
\(\sigma\sim U(S_L)\) gives the AoARM objective in
\eqref{eq:loss_aogpt}:
\[
    \sum_{N_c=0}^{L-1}\mathcal L_{\mathrm{AoARM}}^{N_c}
    =
    \mathbb E_{\x\sim p_{\mathrm{data}},\,\sigma\sim U(S_L)}
    \Big[-\sum_{r=1}^{L}
    \log p_\theta^{\textrm{AoARM}}
    \!\left(x^{\sigma_r}\mid \x^{<\sigma_r},\sigma\right)\Big].
\]

For time-agnostic MDM, \eqref{eq:loss_shared_appendix} becomes the context-level loss
\begin{align}
    \mathcal L_{\mathrm{MDM}}^{N_c}
    =
    \mathbb E_{\x\sim p_{\mathrm{data}},
    \,\Mask_{L-N_c}\sim U(\mathcal M_{L-N_c}),
    \,i\sim U(i:M_{L-N_c}^i=1)}
    \big[-\log p_\theta(x^i\mid \x[\Mask_{L-N_c}])\big],\label{eq:context_level_mdm_loss}
\end{align}
and this matches the theoretical NELBO of MDM by Lemma~\ref{lemma:kim-1}.

For a time-variant MDM, \eqref{eq:loss_shared_appendix} becomes
\begin{align}
\mathcal L_{\mathrm{MDM}}^{N_c}
    =
    \mathbb E_{\substack{
    \x\sim p_{\mathrm{data}},
    \,\Mask_{L-N_c}\sim U(\mathcal M_{L-N_c})\\
    i\sim U(\{i:M_{L-N_c}^i=1\})
    \alpha_{N_c}\sim \mathrm{Beta}(N_c+1,L-N_c)
    }}
    \big[
    -\log p_\theta(x^i\mid \x[\Mask_{L-N_c}],\alpha_t^{-1}(\alpha_{N_c}))
    \big],
    \label{eq:shared_loss_time_variant}
\end{align}
and this matches the theoretical NELBO of MDM by Lemma~\ref{lemma:zheng-1} where we explicitly include $p_\theta$ rather than $\bar p_\theta$.

\paragraph{Remark.} 
Since all experiments in our work use time-agnostic MDMs, as in typical MDM parameterizations
~\citep{sahoo2024simple,nie2025llada,ye2025dream,hong2026lomdm}, we simply evaluate
\(\log p_\theta(x^i\mid \x[\Mask_{L-N_c}])\) by
\(\log \mu_\theta^i(\x[\Mask_{L-N_c}])\). In other words, throughout all of experiments conducted in our paper, we measured $\mathcal{L}_{\mathrm{MDM}}^{N_c}$ by evaluating \eqref{eq:context_level_mdm_loss}.

\section{Proofs}
Before presenting the proofs, we introduce a shorthand for the data conditional induced by a partially observed context.
For a partially observed context $\mathbf z$, let $C_{\#}(\mathbf z)\subseteq [L]$ denote the set of positions whose clean token values are revealed by $\mathbf z$.
For example, for a masked sequence $\x[\Mask]$, since $\Mask^j=1$ means that position $j$ is masked, we have
\[
    C_{\#}(\x[\Mask])=\{j\in[L]:\Mask^j=0\}.
\]
For an autoregressive prefix $\x^{<i}$, we have
\[
    C_{\#}(\x^{<i})=\{1,\ldots,i-1\}.
\]
Similarly, for a permuted prefix $\x^{<\sigma_r}$, we have
\[
    C_{\#}(\x^{<\sigma_r})=\{\sigma_1,\ldots,\sigma_{r-1}\}.
\]
For each $j\in C_{\#}(\mathbf z)$, we write $z^j$ for the revealed token value at position $j$.

Using this notation, for any target position $i\in[L]$, token value $a\in[V]$, and partially observed context $\mathbf z$ with positive data mass, we define
\begin{align}
    p_{\mathrm{data}}(x^i=a\mid \mathbf z)
    &\coloneqq
    p_{\mathrm{data}}\!\left(
        X^i=a
        \,\middle|\,
        X^j=z^j,\ \forall j\in C_{\#}(\mathbf z)
    \right) \nonumber \\
    &=
    \frac{
        \sum_{\tilde{\x}\in\mathcal X:
        \tilde{x}^i=a,\,
        \tilde{x}^j=z^j,\ \forall j\in C_{\#}(\mathbf z)}
        p_{\mathrm{data}}(\tilde{\x})
    }{
        \sum_{\tilde{\x}\in\mathcal X:
        \tilde{x}^j=z^j,\ \forall j\in C_{\#}(\mathbf z)}
        p_{\mathrm{data}}(\tilde{\x})
    }.\label{eq:p_data_cond}
\end{align}
When $a=x^i$ is clear from context, we simply write $p_{\mathrm{data}}(x^i\mid\mathbf z)$.
If the denominator is zero, the conditional distribution can be defined arbitrarily, since such a context is never sampled under $p_{\mathrm{data}}$.

\subsection{Proof of Proposition~\ref{prop:entropy}}\label{sec:proof-prop31}
We prove Proposition~\ref{prop:entropy}, which was omitted in the main paper. We first provide one useful lemma:
\begin{lemma}[Appendix E.1 of Kim~\etal~\citep{kim2025trainworstplanbest}]\label{lemma:kim-2}
    For $\sigma\in S_L$ where $S_L$ is any permutation of $[L]$, let $[\x^{<\sigma_r}:\Mask]$ denotes the masked sequence that reveals
$\{x^{\sigma_1},\ldots,x^{\sigma_{r-1}}\}$ and masks all remaining positions, rather than the raw subsequence.
    For the time-agnostic diffusion network $\model(\cdot)\in\Delta^{V+1}$, the theoretical NELBO of MDM, $\mathcal{L}_{\mathrm{MDM}}$ (\eqref{eq:loss_mdlm}), can be rewritten as follows:
    \begin{align}
\mathcal{L}_{\mathrm{MDM}}=\mathbb{E}_{\x\sim p_{\mathrm{data}},\sigma\sim U(S_L)}
\left[\sum_{r=1}^L-\log p_\theta(x^{\sigma_r}\mid[\x^{<\sigma_r}:\Mask])]\right].
\end{align}
where \(p_\theta(x^i\mid \x[\Mask_{L-N_c}])=\mu_\theta^i(\x[\Mask_{L-N_c}])\).
\end{lemma}
Now we prove the following:
\begingroup
\renewcommand{\thetheorem}{\ref{prop:entropy}}
\renewcommand{\theHtheorem}{app.restate.\ref{prop:entropy}}
\begin{proposition}
With infinite model capacity, the optimal loss for any framework $\mathcal{F}\in\{\mathrm{ARM,MDM},\sigma\mathrm{-ARM,AoARM}\}$ is equal to the entropy of the data distribution:
\begin{align}
    \min_{\theta}\mathcal{L}_{\mathcal{F}}=
    H(X),\quad H(X):=-\mathbb{E}_{\x\sim p_{\mathrm{data}}}\big[\log p_{\mathrm{data}}(\x)\big].
\end{align}
Furthermore, letting $I$ denote the random target index, $\mathcal{L}_{\mathcal F}^{N_c}$ can be decomposed as
\begin{align*}
    \mathcal{L}_{\mathcal F}^{N_c}(\theta)
=
\red{H_{\mathcal F}(X^I\mid Z_{N_c},I)}
+
\blue{\mathbb{E}_{Z_{N_c},I}
\mathrm{KL}\!\left(
p_{\mathrm{data},\mathcal F}(\cdot\mid Z_{N_c},I)
\,\|\,
p_\theta(\cdot\mid Z_{N_c},I)
\right)}.
\end{align*}
where the conditional entropy is induced by $p_{\mathrm{data}}$.
\end{proposition}
\endgroup
\begin{proof}
We first prove the first claim,
\[
\min_{\theta} \mathcal{L}_{\mathcal F} = H(X),
\qquad
\mathcal F \in \{\mathrm{ARM}, \sigma\text{-}\mathrm{ARM}, \mathrm{AoARM}, \mathrm{MDM}\},
\]

For ARM, 
\begin{align*}
L_{\mathrm{ARM}}
&=
-\mathbb{E}_{\mathbf{x}\sim p_{\mathrm{data}}}
\Big[
\sum_{i=1}^{L}
\log p_{\theta}^{\mathrm{ARM}}(x^i\mid \mathbf{x}^{<i})
\Big]=
-\mathbb{E}_{\mathbf{x}\sim p_{\mathrm{data}}}
\log p_{\theta}^{\mathrm{ARM}}(\mathbf{x}) \\
&=
H(X)
+
\mathrm{KL}\!\left(
p_{\mathrm{data}}(\mathbf{x})
\,\|\, 
p_{\theta}^{\mathrm{ARM}}(\mathbf{x})
\right)
\ge H(X),
\end{align*}
where equality holds when $p_\theta(x^i|\mathbf{x}^{<i})=p_{\mathrm{data}}(x^i|\mathbf{x}^{<i})$ satisfies almost everywhere. Note that we can prove for $\sigma$-ARM case as above. For AoARM, 
\begin{align*}
\mathcal{L}_{\mathrm{AoARM}}
&=
-\mathbb{E}_{\sigma\sim U(S_L)}
\mathbb{E}_{\mathbf{x}\sim p_{\mathrm{data}}}
[\log p_{\theta}^{\mathrm{AoARM}}(\mathbf{x}\mid \sigma)] \\
&=
\mathbb{E}_{\sigma\sim U(S_L)}
\Big[
H(X)
+
\mathrm{KL}\!\left(
p_{\mathrm{data}}(\mathbf{x})
\,\|\, 
p_{\theta}^{\mathrm{AoARM}}(\mathbf{x}\mid \sigma)
\right)
\Big]\ge H(X),
\end{align*}
Under infinite model capacity, the equality is attained by
\[
p_{\theta}^{\mathrm{AoARM}}(\mathbf{x}^{\sigma_i}\mid \mathbf{x}^{<\sigma_i},\sigma)
=
p_{\mathrm{data}}(\mathbf{x}^{\sigma_i}\mid \mathbf{x}^{<\sigma_i})
\qquad
\text{for all } i\in[L] \text{ and }\sigma \in S_L.
\]

For time-variant MDM, rather than decomposing NELBO into above decomposition, we directly show that
$\min_\theta \mathcal{L}_{\mathrm{MDM}}=H(X)$ is achievable when
$p_\theta(x^i=a\mid \x_t,t)=p_{\mathrm{data}}(x^i=a\mid \x_t)$ for all
$i\in[L]$, for all
$\x_t\in\{\x[\Mask]:\forall \x\in\mathcal{X},\forall\Mask\in\{0,1\}^L\}$,
and for all $a\in[V]$.

Recall that NELBO of MDM is as follows:
\begin{align}
    \mathbb{E}_{p_{\mathrm{data}}}[-\log p_\theta(\mathbf{x})]\le \mathcal{L}_{\mathrm{MDM}}=\int_{0}^{1}
\tfrac{-\alpha_t'}{1-\alpha_t}\mathbb{E}_{\x\sim p_{\mathrm{data}},\x_t\sim q_{t|0}}
\Bigl[-
\sum_{i:x_t^i=\mask}
\log p_\theta(x^i|\x_t,t)\Big]dt.
\end{align}
Here, the left-hand side is lower bounded by $H(X)$. On the other hand, we show that if the infinite capacity model $\theta^*$ satisfies $p_{\theta^*}(x^i=a\mid \x_t,t)=p_{\mathrm{data}}(x^i=a\mid \x_t)$, the right-hand side gives $H(X)$. If $p_{\theta^*}(x^i=a|\x_t,t)=p_{\mathrm{data}}(x^i=a|\mathbf{x}_t)$ holds, since $p_{\theta^*}(x^i=a|\x_t,t)$ is parametrized by $\mu_{\theta^*}(\x_t,t)$, the optimal $\theta^*$ is no more time-variant rather time-agnostic. 
Substituting $\theta$ into $\theta^*$ such that $p_\theta(x^i=a|\x_t,t)=p_{\mathrm{data}}(x^i=a|\mathbf{x}_t)$, and using Lemma~\ref{lemma:kim-2} and \eqref{eq:p_data_cond}, the right-hand side gives
\begin{align}
    \mathbb{E}_{\x\sim p_{\mathrm{data}},\sigma\sim U(S_L)}
\left[\sum_{r=1}^L-\log p_{\mathrm{data}}(x^{\sigma_r}\mid\x^{<\sigma_r})\right],
\end{align}
which is equivalent to $H(X)$. 
Therefore, with infinite model capacity, $\min_\theta \mathcal{L}_{\mathrm{MDM}}=H(X)$, and this minimum is achievable. Note that time-agnostic MDM can be proven with the same approach as above. 

We now prove the second claim. Recall that
\begin{align}
\mathcal{L}_{\mathcal{F}}^{N_c}=
\mathbb{E}_{\mathbf{x}\sim p_{\mathrm{data}},\mathbf{z}_{N_c}\sim p_{\mathcal{F}}(\mathbf{z}\mid \mathbf{x},N_c),i\sim p_{\mathcal{F}}(i\mid \mathbf{x},\mathbf{z}_{N_c})}
[-\log p_{\theta}(x^i\mid \mathbf{z}_{N_c})].
\end{align}
For a fixed observation level $n$, define the framework-induced joint distribution
\[
p_{\mathrm{data},\mathcal{F}}(\mathbf{x},\mathbf{z}_{N_c},i\mid N_c)
:=
p_{\mathrm{data}}(\mathbf{x})\,
p_{\mathcal{F}}(\mathbf{z}_{N_c}\mid \mathbf{x},N_c)\,
p_{\mathcal{F}}(i\mid \mathbf{x},\mathbf{z}_{N_c}).
\]
Also, with a slight abuse of notation, write
\[
p_\theta(a\mid \mathbf{z}_{N_c},i)
:=
p_\theta(x^i=a\mid \mathbf{z}_{N_c}),
\]
namely, the model prediction for the $i$-th token given $\mathbf{z}_{N_c}$.
Then,
\begin{align*}
\mathcal{L}_{\mathcal{F}}^{N_c}
&=
\mathbb{E}_{\mathbf{x},\mathbf{z}_{N_c},i}
\bigl[-\log p_\theta(x^i\mid \mathbf{z}_{N_c})\bigr] \\
&=
\mathbb{E}_{(\mathbf{z}_{N_c},i)\sim p_{\mathrm{data},\mathcal{F}}(\cdot,\cdot\mid N_c)}
\left[
\sum_{a}
p_{\mathrm{data},\mathcal{F}}(a\mid \mathbf{z}_{N_c},i)
\bigl(-\log p_\theta(a\mid \mathbf{z}_{N_c},i)\bigr)
\right] \\
&=
\mathbb{E}_{(\mathbf{z}_{N_c},i)\sim p_{\mathrm{data},\mathcal{F}}(\cdot,\cdot\mid N_c)}
\Big[
H\!\bigl(p_{\mathrm{data},\mathcal{F}}(\cdot\mid \mathbf{z}_{N_c},i)\bigr)
+
\mathrm{KL}\!\left(
p_{\mathrm{data},\mathcal{F}}(\cdot\mid \mathbf{z}_{N_c},i)
\,\|\, 
p_\theta(\cdot\mid \mathbf{z}_{N_c},i)
\right)
\Big] \\
&=
H_{\mathcal{F}}(X^I\mid Z_{N_c},I)
+
\mathbb{E}_{(\mathbf{z}_{N_c},i)\sim p_{\mathrm{data},\mathcal{F}}(\cdot,\cdot\mid N_c)}
\mathrm{KL}\!\left(
p_{\mathrm{data},\mathcal{F}}(\cdot\mid \mathbf{z}_{N_c},i)
\,\|\, 
p_\theta(\cdot\mid \mathbf{z}_{N_c},i)
\right), 
\end{align*}
which concludes the proof. Therefore, for every framework $\mathcal{F}$ and every observation level $n$,
\[
\mathcal{L}_{\mathcal{F}}^{N_c} \ge H_{\mathcal{F}}(X^I\mid Z_{N_c},I).
\]
Moreover, equality holds if and only if
\[
p_\theta(x^i\mid \mathbf{z}_{N_c})
=
p_{\mathrm{data},\mathcal{F}}(x^i\mid \mathbf{z}_{N_c},i)
\qquad
\text{for }p_{\mathrm{data},\mathcal{F}}(\mathbf{z}_{N_c},i\mid N_c)\text{-almost every }(\mathbf{z}_{N_c},i).
\]

\end{proof}
\subsection{Proof of Proposition~\ref{prop:equal_optimizer}}\label{sec:proof-cor41}
We prove Proposition~\ref{prop:equal_optimizer} in this section. We first provide useful lemma to prove the statement:
\begin{lemma}[Restatement of Zheng~\etal~\citep{zheng2024maskeddiffusionmodelssecretly}]\label{lemma:equal-optimizer} Consider two different loss, $\mathcal{L}_{\mathrm{MDM}}$ and weighted loss $\hat{\mathcal{L}}$ defined as follows:
\begin{gather}
    \mathcal L_{\mathrm{MDM}}(\theta)
    =
    \sum_{N_c=0}^{L-1}
    \mathcal L_{\mathrm{MDM}}^{N_c}(\theta)
    \qquad\hat{\mathcal L}(\theta)
    =
    \sum_{N_c=0}^{L-1}
    \omega(N_c)
    \mathcal L_{\mathrm{MDM}}^{N_c}(\theta).
\end{gather}
With infinite capacity and time-agnostic $\theta$, if $\omega(N_c)>0$ holds for all $N_c\in\{0,\dots,L-1\}$, the minimizer of $\mathcal{L}_{\mathrm{MDM}}$ and $\hat{\mathcal{L}}$ are equal.
\end{lemma}
\begin{proof}
Note that above lemma can be inferred from the proof of Proposition 3.2 in Zheng~\etal~\citep{zheng2024maskeddiffusionmodelssecretly}. We here show the above statement more directly with our notations. Since the model is time-agnostic, we write
\(p_\theta(x^i\mid \x_t)\) instead of
\(p_\theta(x^i\mid \x_t,t)\). For \(n\in\{0,\ldots,L\}\), recall that
$\mathcal M_n=
\{
\mathbf M\in\{0,1\}^L:
\sum_{i=1}^L \mathbf M^i=n
\}$
is the set of mask indicators with exactly \(n\) masked positions, where
\(\mathbf M^i=1\) means that position \(i\) is masked. Also recall that
\(\x[\mathbf M_n]\) is the sequence obtained by replacing the masked
positions of \(\x\) with \(\mask\). Thus, when the number of context
tokens is \(N_c\), the number of masked tokens is \(L-N_c\), and the
corresponding masked sequence is \(\x[\mathbf M_{L-N_c}]\). 

For different values of \(N_c\), the corresponding masked contexts have
disjoint input supports, since \(\x[\mathbf M_{L-N_c}]\) contains exactly
\(L-N_c\) mask tokens. Hence, under infinite capacity and time-agnostic
parameterization, the model can realize the pointwise optimum for each
context level independently. Formally, define the common minimizer set
\[
\Theta^\star
:=
\left\{
\theta:
p_\theta(x^i=a\mid \x[\Mask_{L-N_c}])
=
p_{\mathrm{data}}(x^i=a\mid \x[\Mask_{L-N_c}])
\right.
\]
\[
\left.
\qquad\qquad
\forall N_c=0,\ldots,L-1,\ 
\forall \mathbf M_{L-N_c}\in\mathcal M_{L-N_c},\
\forall i:M_{L-N_c}^i=1,\ 
\forall \x\in [V]^L,
\forall a\in[V]
\right\}.
\]
Under infinite
capacity, this set is nonempty, since the model can realize the
pointwise conditional distribution on every masked-context input.

To see that this is exactly the minimizer set, fix \(N_c\),
\(\mathbf M_{L-N_c}\in\mathcal M_{L-N_c}\), a masked position
\(i\) with \(\mathbf M_{L-N_c}^i=1\), and a masked context \(\x_t\).
The corresponding token-level cross-entropy is
\[
    \sum_{a\in[V]}
    p_{\mathrm{data}}(x^i=a\mid \x[\Mask_{L-N_c}])
    \bigl[-\log p_\theta(x^i=a\mid \x[\Mask_{L-N_c}])\bigr].
\]
This can be decomposed as
\[
    H\!\left(
    p_{\mathrm{data}}(\cdot\mid \x[\Mask_{L-N_c}])
    \right)
    +
    \mathrm{KL}\!\left(
    p_{\mathrm{data}}(\cdot\mid \x[\Mask_{L-N_c}])
    \,\|\,p_\theta(\cdot\mid \x[\Mask_{L-N_c}])
    \right).
\]
The entropy term is independent of \(\theta\), and the KL term is
minimized if and only if
\[
    p_\theta(x^i=a\mid \\x[\Mask_{L-N_c}])
    =
    p_{\mathrm{data}}(x^i=a\mid \x[\Mask_{L-N_c}])
\]
for every token \(a\in[V]\). Hence each context-level loss
\(\mathcal L_{\mathrm{MDM}}^{N_c}\) is minimized exactly by matching
the true conditional distribution on its corresponding masked contexts.
Since \(\hat{\mathcal L}\) and \(\mathcal L_{\mathrm{MDM}}\) are
both positive weighted sums of the same context-level cross-entropy
terms,
\[
    \hat{\mathcal L}(\theta)
    =
    \sum_{N_c=0}^{L-1}
    \omega(N_c)
    \mathcal L_{\mathrm{MDM}}^{N_c}(\theta),
    \qquad
    \omega_\pi(N_c)>0,
\]
and
\[
    \mathcal L_{\mathrm{MDM}}(\theta)
    =
    \sum_{N_c=0}^{L-1}
    \mathcal L_{\mathrm{MDM}}^{N_c}(\theta),
\]
their minimizers are exactly the same:
\[
    \arg\min_\theta \hat{\mathcal L}(\theta)
    =
    \Theta^\star
    =
    \arg\min_\theta \mathcal L_{\mathrm{MDM}}(\theta).
\]
\end{proof}
We now prove the Proposition~\ref{prop:equal_optimizer} with Lemma~\ref{lemma:equal-optimizer}.
What we need to prove is that $w(N_c)>0,\forall N_c\in\{0,\dots,L-1\}$ holds for any probability distribution $\rho$ satisfying $\mathbb{P}_{t\sim\rho}\left(t\in (0,1)\right)>0$.
\begingroup
\renewcommand{\thetheorem}{\ref{prop:equal_optimizer}}
\renewcommand{\theHtheorem}{app.restate.\ref{prop:equal_optimizer}}
\begin{proposition} With infinite-capacity and time-agnostic model $\theta$, for any probability distribution $\rho$ over $[0,1]$ satisfying $\mathbb{P}_{t\sim\rho}\left(t\in (0,1)\right)>0$, the minimizers of $\mathcal{L}_{\mathrm{MDM}}$ and $\hat{\mathcal L}_{\rho}$ are equal.
\end{proposition}
\endgroup
\begin{proof}
Let \(\rho\) be a Borel probability measure on \([0,1]\). All integrals with respect to \(d\rho(t)\) are understood as
measure-theoretic integrals, so atomic choices such as
\(\rho=\delta_\mu\) are included. 
Under the absorbing forward process, for any
\(\mathbf M_{L-N_c}\in\mathcal M_{L-N_c}\),
\[
    \Pr(\x_t=\x[\mathbf M_{L-N_c}]\mid \x,t)
    =
    \alpha_t^{N_c}(1-\alpha_t)^{L-N_c}.
\]
Therefore, after marginalizing over \(t\sim\rho\), the probability of
sampling exactly \(N_c\) context tokens is
\[
    \Pr_\rho(N_c)
    =
    {L\choose N_c}
    \mathbb E_{t\sim\rho}
    \left[
    \alpha_t^{N_c}(1-\alpha_t)^{L-N_c}
    \right]
    =
    {L\choose N_c}
    \int_{[0,1]}
    \alpha_t^{N_c}(1-\alpha_t)^{L-N_c}\,d\rho(t).
\]
This formula remains valid when \(\rho\) has atoms; for example, if
\(\rho=\delta_\mu\), the integral simply evaluates the integrand at
\(t=\mu\). Conditioned on \(N_c\), all mask indicators in
\(\mathcal M_{L-N_c}\) are sampled uniformly, because the above
probability depends on \(\mathbf M_{L-N_c}\) only through its cardinality.
Hence
\begin{align*}
    \hat{\mathcal L}_{\rho}(\theta)
    &=
    \mathbb E_{\x\sim p_{\mathrm{data}},\,t\sim\rho,\,
    \x_t\sim q_{t|0}(\cdot\mid\x)}
    \left[
    -\sum_{i:x_t^i=\mask}
    \log p_\theta(x^i\mid \x_t)
    \right] \\
    &=
    \sum_{N_c=0}^{L-1}
    (L-N_c)\Pr_\rho(N_c)\,
    \mathcal L_{\mathrm{MDM}}^{N_c}(\theta)=
    \sum_{N_c=0}^{L-1}
    \omega_\rho(N_c)\,
    \mathcal L_{\mathrm{MDM}}^{N_c}(\theta),
\end{align*}
where 
$\omega_\rho(N_c)
:=
(L-N_c){L\choose N_c}
\int_{[0,1]}
\alpha_t^{N_c}(1-\alpha_t)^{L-N_c}\,d\rho(t)$.
By definition of $\alpha_t$ in conventional MDM~\citep{sahoo2024simple,shi2025simplified,hong2026lomdm} where $\alpha_0=1$, $\alpha_1=0$, and $\alpha_t'<0$, the integrand
\(\alpha_t^{N_c}(1-\alpha_t)^{L-N_c}\) is strictly positive for every
\(t\in(0,1)\) and every \(N_c\in\{0,\ldots,L-1\}\). Since
\(\mathbb{P}_{t\sim\rho}\left(t\in (0,1)\right)>0\), we have $\omega_\rho(N_c)>0$ for all $N_c\in\{0,\ldots,L-1\}$. Then, together with Lemma~\ref{lemma:equal-optimizer}, this concludes the proof.
\end{proof}

\paragraph{Remark.}
In measure-theoretic notation, the probability mass assigned by a probability distribution \(\rho\) to an interval \(I\subseteq[0,1]\) is conventionally written as \(\rho(I)\). To make this notation more accessible to readers less familiar with this convention, we instead write the same quantity as \(\mathbb P_{t\sim\rho}(t\in I)\) throughout the main text and proofs.

\subsection{Proof of Proposition~\ref{prop:bell_shape}}
\label{sec:proof-prop43}
We prove Proposition~\ref{prop:bell_shape} in this section. Recall that $\mathcal E_k$ denotes the set of
$k$-inefficient mask indicators, i.e., those containing either $k$ consecutive masks or $k$
consecutive non-masks. Accordingly, $\mathcal E_k^c$ denotes the set of $k$-efficient mask
indicators. Recall also that
\[
\mathcal M_n
=
\Bigl\{\mathbf M\in\mathcal \{0,1\}^L : \sum_{i=1}^L M^i = n\Bigr\},
\qquad
\mathbf M_n \sim U(\mathcal M_n),
\]
so that $\mathbf M_n$ is a uniformly sampled mask indicator with exactly $n$ masked positions.
Equivalently, it has exactly $L-n$ observed positions.

Define
\[
q_{L,k}(n)
:=
\mathbb P\bigl(\mathbf M_n\in\mathcal E_k^c\bigr|\mathbf M_n \sim U(\mathcal M_n)).
\]
Let
\[
N_{k}(a,b)
:=
\#\Bigl\{\mathbf M\in\{0,1\}^{a+b} : \sum_{i=1}^{a+b} M^i = a,\ \mathbf M\in\mathcal E_k^c\Bigr\},
\]
where $a$ counts masked positions and $b$ counts non-masked positions. Since $k$ is fixed
throughout this proof, we drop $k$ in the notation. Then
\[
q_{L,k}(n)=\frac{N(n,L-n)}{\binom{L}{n}}.
\]

For a time sampling law $\pi$ on $[0,1]$, define the induced mask-count law
\[
w_\pi(n)
:=
\mathbb P\left(
\sum_{i=1}^L M^i=n
\ \middle|\
 t\sim \pi,\ M^i\stackrel{\mathrm{i.i.d.}}{\sim}\mathrm{Bernoulli}(t)
\right).
\]
Here $\pi$ is understood as a Borel probability measure, and it need not admit a density.
Conditioning on $t$ gives
\[
w_\pi(n)
=
\binom{L}{n}\int_{[0,1]} t^n(1-t)^{L-n}\,d\pi(t).
\]
Therefore
\begin{equation}
\label{eq:prop43-count-decomp}
\mathbb P\Big(\mathbf M\in\mathcal E_k^c\mid t\sim\pi, M^i\stackrel{\mathrm{i.i.d.}}{\sim} \mathrm{Bernoulli}(t)\Big)
=
\sum_{n=0}^L w_\pi(n)q_{L,k}(n).
\end{equation}
In particular, when $\pi=U([0,1])$,
\[
w_U(n)
=
\binom{L}{n}\int_0^1 t^n(1-t)^{L-n}\,dt
=
\frac{1}{L+1},
\]
so uniform time sampling induces the uniform distribution over the mask count
$n\in\{0,\dots,L\}$.

The key step is to show that $q_{L,k}(n)$ is maximized at the middle mask count.

\begin{lemma}[Central monotonicity of the efficient probability]
\label{lem:q-central-monotone}
For every $L\in\mathbb N$ and $k\ge 2$,
\[
q_{L,k}(n)=q_{L,k}(L-n),
\]
and
\[
q_{L,k}(n)\le q_{L,k}(n+1)
\qquad
\text{for all }0\le n<\frac L2.
\]
Consequently, $q_{L,k}(n)$ is maximized at $n=L/2$ when $L$ is even, and at the two central
points $(L-1)/2$ and $(L+1)/2$ when $L$ is odd.
\end{lemma}

\begin{proof}
The symmetry
\[
q_{L,k}(n)=q_{L,k}(L-n)
\]
follows by flipping all mask bits, $0\leftrightarrow 1$. This operation maps $\mathcal M_n$
bijectively to $\mathcal M_{L-n}$ and preserves the event $\mathcal E_k^c$, since it only swaps
masks and non-masks.

It remains to prove monotonicity up to the center.
Fix $n<m$ and set $L=n+m$. We want to prove
\[
q_{L,k}(n)\le q_{L,k}(n+1),
\]
or equivalently
\[
\frac{N(n,m)}{\binom{L}{n}}
\le
\frac{N(n+1,m-1)}{\binom{L}{n+1}}.
\]
Since
\[
\frac{\binom{L}{n+1}}{\binom{L}{n}}
=
\frac{m}{n+1},
\]
this is equivalent to
\begin{equation}
\label{eq:q-mono-target}
(n+1)N(n+1,m-1)-mN(n,m)\ge 0.
\end{equation}

Unfortunately, $N(n,m)$ expands into a highly cumbersome form, making it difficult to prove the above inequality by a direct expansion. To handle this combinatorial counting problem, we use the concept of a \textit{generating function}.
Generating functions are a classical tool in enumerative combinatorics, used to encode combinatorial counts as coefficients of formal polynomials. For example, if a sequence $\{a_r\}_{r\ge 0}$ is encoded as
\[
G(x)=\sum_{r\ge 0} a_r x^r,
\]
then the coefficient-extraction notation $[x^r]G(x)$ denotes the coefficient of $x^r$ in $G(x)$, namely $a_r$. In this sense, questions about the combinatorial numbers $a_r$ can be translated into algebraic questions about the polynomial or power series $G(x)$.
In the present proof, we will consider $F(x,y)$ that assigns a formal weight $x$ to each masked position and $y$ to each non-masked position; then the coefficient of $x^a y^b$ records how many $k$-efficient mask indicators have $a$ masked and $b$ non-masked positions. Equivalently, $[x^ay^b]F(x,y)=N(a,b)$ where $F(x, y) = \sum_{a,b\ge0}N(a,b)x^ay^b$, so the \textit{operation} $[x^ay^b]$ pulls out the coefficient of $x^a y^b$. 

We divide three cases of $M\in\mathcal E_k^c$ to obtain $F(x,y)$: 1) the number of mask-runs and non-mask-runs are equal, \eg, $\mathbf{M}=1\dots0$ or $\mathbf{M}=0\dots1$, 2) there is one more mask-run than non-mask-run $\mathbf{M}=1\dots1$, and 3) there is one less mask-run than non-mask-run $\mathbf{M}=0\dots0$.
To formalize $F(x,y)$, first let
\[
P_k(z):=z+z^2+\cdots+z^{k-1}.
\]
Note the length of each consecutive 0's or 1's are in $\{1, \ldots, k-1\}$. So, for example, the case starting from $x$ and end at $y$ with $2r$ total runs can be represented as \[ P_k(x)P_k(y)\cdots P_k(x)P_k(y)=P_k(x)^rP_k(y)^r. \]
Furthermore, there also exists the case starting from $x$ and end at $y$ with $2r$ total runs:
\[ P_k(y)P_k(x)\cdots P_k(y)P_k(x)=P_k(x)^rP_k(y)^r. \]
Therefore, if the number of mask-runs and non-mask-runs
are both equal to $r\ge 1$, then the total contribution is
\[
2P_k(x)^rP_k(y)^r.
\]
If there is one more mask-run than non-mask-run, the contribution is
\[
P_k(x)^{r+1}P_k(y)^r
\qquad (r\ge 0),
\]
and similarly, if there is one more non-mask-run than mask-run, the contribution is
\[
P_k(x)^rP_k(y)^{r+1}
\qquad (r\ge 0).
\]
To compute $F(x, y)$, we take their sum over $r\ge0$ with the above cases. Then the bivariate generating function of nonempty $k$-efficient mask indicators is
\[
F(x,y)
:=
\sum_{a,b\ge 0} N(a,b)x^ay^b
=
\frac{P_k(x)+P_k(y)+2P_k(x)P_k(y)}{1-P_k(x)P_k(y)}.
\]
The empty string is irrelevant here, since we only extract coefficients of positive total degree.
Indeed, every $k$-efficient mask indicator is an alternating sequence of mask-runs and non-mask-runs,
and every run length belongs to $\{1,\dots,k-1\}$. 
Now the left-hand side of \eqref{eq:q-mono-target} can be converted into the coefficient of $x^ny^{m-1}$ in
\[
D(x,y):=\partial_xF(x,y)-\partial_yF(x,y).
\]
Therefore it suffices to show that
\[
[x^py^q]D(x,y)\ge 0
\qquad
\text{whenever }p\le q.
\]

We now prove this coefficientwise nonnegativity. Define
\[
A_k(z):=1+P_k(z)=1+z+\cdots+z^{k-1}.
\]
A direct differentiation of $F(x,y)$ gives
\[
D(x,y)
=
\frac{A_k(y)^2P_k'(x)-A_k(x)^2P_k'(y)}{\bigl(1-P_k(x)P_k(y)\bigr)^2}.
\]
Write
\[
B(x,y):=A_k(y)^2P_k'(x)-A_k(x)^2P_k'(y),
\qquad
K(x,y):=\frac{1}{\bigl(1-P_k(x)P_k(y)\bigr)^2}.
\]
Then
\[
D(x,y)=B(x,y)K(x,y).
\]
By the definition of a generating function,
\[
[x^p y^q]D(x,y)\ge 0
\quad\Longleftrightarrow\quad
[x^p y^q]\Bigg(\sum_{i,j}
\Big([x^i y^j]B(x,y)\Big)
\Big([x^{p-i}y^{q-j}]K(x,y)\Big)\Bigg)\ge 0.
\]
Therefore, it suffices to show the right-hand side holds for any $p\le q$. Let $c_{ij}:=[x^iy^j]B(x,y)$ and $K_{u,v}:=[x^uy^v]K(x,y)$. Our objective is then to show:
\[
\sum_{i,j}c_{ij}K_{p-i,q-j}\ge0,\qquad \text{whenever }p\le q.
\]
We first analyze the signs of the coefficients of $B$. Let
\[
a_i:=[z^i]P_k'(z),
\qquad
b_i:=[z^i]A_k(z)^2.
\]
Since
\[
P_k'(z)=1+2z+\cdots+(k-1)z^{k-2},
\]
and
\[
A_k(z)^2=(1+z+\cdots+z^{k-1})^2,
\]
we have
\[
\frac{a_i}{b_i}
=
\begin{cases}
1, & 0\le i\le k-2,\\
0, & k-1\le i\le 2k-2,
\end{cases}
\]
with the convention that coefficients outside the support are zero. Hence $i\mapsto a_i/b_i$ is
nonincreasing. Therefore, for $i<j$,
\[
a_ib_j-b_ia_j
=
b_ib_j\left(\frac{a_i}{b_i}-\frac{a_j}{b_j}\right)
\ge 0.
\]
Then by definition of $c_{ij}$,
\begin{equation}
\label{eq:B-coeff-sign}
c_{ij}\ge 0\quad (i<j),
\qquad
c_{ji}=-c_{ij},
\qquad
c_{ii}=0.
\end{equation}

Next we analyze $K$. For
\[
K(x,y)
=
\frac{1}{\bigl(1-P_k(x)P_k(y)\bigr)^2}
=
\sum_{r\ge 0}(r+1)P_k(x)^rP_k(y)^r,
\]
let
\[
f_r(s):=[z^s]P_k(z)^r\qquad\text{s.t. }K_{u,v}
=\sum_{r\ge 0}(r+1)f_r(u)f_r(v).
\]
For each fixed $r$, the sequence $s\mapsto f_r(s)$ is log-concave. Indeed,
$P_k(z)=z(1+z+\cdots+z^{k-2})$, and the nonzero coefficient sequence $(1,\dots,1)$ is
log-concave; moreover, convolution preserves log-concavity for nonnegative sequences without
internal zeros. Hence the coefficient sequence of $P_k(z)^r$ is log-concave. Equivalently, the ratio
\[
\rho_s:=\frac{f_r(s+1)}{f_r(s)}
\]
is nonincreasing wherever both terms are positive. Therefore, for each fixed total degree $S$ and for
$u\le S/2$,
\[
\frac{f_r(u+1)f_r(S-u-1)}{f_r(u)f_r(S-u)}
=
\frac{\rho_u}{\rho_{S-u-1}}
\ge 1,
\]
so the sequence
\[
u\mapsto f_r(u)f_r(S-u)
\]
is symmetric around $S/2$ and nondecreasing for $u\le S/2$. Summing over $r$ with nonnegative
weights $(r+1)$ preserves this property. Therefore,
\begin{equation}
\label{eq:K-central-monotone}
K_{u,v}\ge K_{u',v'}
\quad
\text{whenever }
u+v=u'+v'
\text{ and }
|u-v|\le |u'-v'|.
\end{equation}

Now fix $p\le q$. Recall that
\[
[x^py^q]D
=
\sum_{i,j}c_{ij}K_{p-i,q-j}.
\]
Pairing the terms $(i,j)$ and $(j,i)$ and using \eqref{eq:B-coeff-sign}, we get
\[
[x^py^q]D
=
\sum_{i<j}c_{ij}\bigl(K_{p-i,q-j}-K_{p-j,q-i}\bigr),
\]
where we use the convention that $K_{u,v}=0$ whenever $u<0$ or $v<0$.
We now show that each paired summand is nonnegative. Fix $i<j$.

First, suppose that $p<j$. Then $p-j<0$, and hence
\[
    K_{p-j,q-i}=0.
\]
Since all coefficients of $K$ are nonnegative, we have
\[
    K_{p-i,q-j}-K_{p-j,q-i}=K_{p-i,q-j}\ge 0,
\]
where the inequality also covers the case in which $p-i<0$ or $q-j<0$, since then
$K_{p-i,q-j}=0$ by convention.

It remains to consider the case $p\ge j$. Since $p\le q$ and $i<j$, all indices appearing below
are nonnegative:
\[
p-i,\quad q-j,\quad p-j,\quad q-i \ge 0.
\]
The two pairs $(p-i,q-j)$ and $(p-j,q-i)$ have the same total degree:
\[
(p-i)+(q-j)=(p-j)+(q-i).
\]
Moreover, since $p\le q$ and $i<j$,
\[
\bigl|(q-j)-(p-i)\bigr|
=
\bigl|q-p-(j-i)\bigr|
\le
q-p+(j-i)
=
\bigl|(q-i)-(p-j)\bigr|.
\]
Thus $(p-i,q-j)$ is at least as close to the diagonal as $(p-j,q-i)$. By
\eqref{eq:K-central-monotone},
\[
K_{p-i,q-j}\ge K_{p-j,q-i}.
\]
Therefore, in both cases,
\[
K_{p-i,q-j}-K_{p-j,q-i}\ge 0.
\]
Since also $c_{ij}\ge 0$ for $i<j$, every paired summand is nonnegative. Hence
\[
[x^py^q]D\ge 0
\qquad
\text{whenever }p\le q.
\]

Finally, take $p=n$ and $q=m-1$. Since $n<m$, we have $n\le m-1$, so
\[
[x^ny^{m-1}]D\ge 0.
\]
Equivalently,
\[
(n+1)N(n+1,m-1)-mN(n,m)\ge 0.
\]
As shown above, this is exactly
\[
q_{L,k}(n)\le q_{L,k}(n+1).
\]
This proves the monotonicity for every $n<L/2$, and the lemma follows.
\end{proof}
We now translate Lemma~\ref{lem:q-central-monotone} into a statement about time sampling.
Recall that
\begin{align}
\mathbb P\Big(\mathbf M\in\mathcal E_k^c\mid t\sim\pi, M^i\stackrel{\mathrm{i.i.d.}}{\sim} \mathrm{Bernoulli}(t)\Big)
=
\sum_{n=0}^L w_\pi(n)q_{L,k}(n)\label{eq:prob_mass_w_pi}.
\end{align}
where
\begin{align}
w_\pi(n)&=
\mathbb P\left(
\sum_{i=1}^L M^i=n
\ \middle|\
 t\sim \pi,\ M^i\stackrel{\mathrm{i.i.d.}}{\sim}\mathrm{Bernoulli}(t)
\right)\nonumber\\
&=
\binom{L}{n}\int_{[0,1]} t^n(1-t)^{L-n}\,d\pi(t)\label{eq:w_pi}
\end{align}
Substituting \eqref{eq:w_pi} into \eqref{eq:prob_mass_w_pi} gives:
\begin{align}
\mathbb P\Big(\mathbf M\in&\mathcal E_k^c\mid t\sim\pi, M^i\stackrel{\mathrm{i.i.d.}}{\sim} \mathrm{Bernoulli}(t)\Big)
=
\sum_{n=0}^L w_\pi(n)q_{L,k}(n)\nonumber\\
&=\sum_{n=0}^L\binom{L}{n}\int_{[0,1]} t^n(1-t)^{L-n}q_{L,k}(n)\,d\pi(t)\nonumber\\
&=\int_{[0,1]} \sum_{n=0}^L\binom{L}{n}t^n(1-t)^{L-n}q_{L,k}(n)\,d\pi(t)\label{eq:prob_mass_with_qLk}
\end{align}

Define
\[
R_{L,k}(t)
:=
\mathbb P\left(
\mathbf M\in\mathcal E_k^c
\ \middle|\
 M^i\stackrel{\mathrm{i.i.d.}}{\sim}\mathrm{Bernoulli}(t)
\right).
\]
Conditioning on the number of masks gives
\begin{equation}
\label{eq:R-bernstein}
R_{L,k}(t)
=
\sum_{n=0}^L \binom{L}{n}t^n(1-t)^{L-n}q_{L,k}(n).
\end{equation}
Substituting \eqref{eq:R-bernstein} into \eqref{eq:prob_mass_with_qLk} then gives:
\begin{align}
    \mathbb P\Big(\mathbf M\in&\mathcal E_k^c\mid t\sim\pi, M^i\stackrel{\mathrm{i.i.d.}}{\sim} \mathrm{Bernoulli}(t)\Big)
=\int_{[0,1]} R_{L,k}(t) d\pi(t).\label{eq:prob_mass_R_pi}
\end{align}
Now our goal is to prove $R_{L,k}(t)$ has property of central monotonicity, and then prove Proposition~\ref{prop:bell_shape}.

\begin{lemma}[Central monotonicity in time]
\label{lem:R-central-monotone}
The function $R_{L,k}(t)$ is symmetric around $1/2$ and nondecreasing on $[0,1/2]$. That is,
\[
R_{L,k}(t)=R_{L,k}(1-t),
\qquad
R_{L,k}'(t)\ge 0
\quad
\text{for }0\le t\le \frac12.
\]
Moreover, if $2\le k\le L$, then $R_{L,k}$ is strictly increasing on $[0,1/2]$.
\end{lemma}

\begin{proof}
The symmetry follows from the symmetry of $q_{L,k}(n)$:
\[
q_{L,k}(n)=q_{L,k}(L-n).
\]
To prove monotonicity, differentiate \eqref{eq:R-bernstein}. The derivative of $R_{L,k}(t)$ is
\[
R_{L,k}'(t)
=
L\sum_{n=0}^{L-1}
\binom{L-1}{n}
\bigl(q_{L,k}(n+1)-q_{L,k}(n)\bigr)t^n(1-t)^{L-1-n}.
\]
Let $d_n:=q_{L,k}(n+1)-q_{L,k}(n)$. By Lemma~\ref{lem:q-central-monotone},
$d_n\ge 0$ for $n<(L-1)/2$, and by symmetry of $q_{L,k}$, $d_{L-1-n}=-d_n$.
Now pair the $n$-th and $(L-1-n)$-th terms in the derivative. For $n<(L-1)/2$, the paired
contribution is
\[
L\binom{L-1}{n}d_n
\Bigl[t^n(1-t)^{L-1-n}-t^{L-1-n}(1-t)^n\Bigr].
\]
When $0\le t\le 1/2$, we have $1-t\ge t$, so every paired contribution is nonnegative. Hence
$R_{L,k}'(t)\ge 0$ on $[0,1/2]$.

It remains to prove strict monotonicity. When $2\le k\le L$, we have $R_{L,k}(0)=0$, because
$t=0$ gives the all-zero mask indicator, which contains $L\ge k$ consecutive non-masks and
therefore lies in $\mathcal E_k$. On the other hand, $R_{L,k}(1/2)>0$, since the alternating mask
indicator $0101\cdots$ is $k$-efficient for every $k\ge 2$ and occurs with positive probability under
i.i.d. Bernoulli$(1/2)$ sampling. Thus $R_{L,k}$ is not constant. Since $R_{L,k}'$ is a polynomial and
$R_{L,k}'\ge 0$ on $[0,1/2]$, if $R_{L,k}$ were not strictly increasing, then it would be constant on
some nonempty interval, so $R_{L,k}'$ would vanish on that interval and hence vanish identically.
This would contradict $R_{L,k}(0)<R_{L,k}(1/2)$. Therefore $R_{L,k}$ is strictly increasing on
$[0,1/2]$.
\end{proof}

We can now prove Proposition~\ref{prop:bell_shape}. Recall that Proposition~\ref{prop:bell_shape} is as follows:

\begingroup
\renewcommand{\thetheorem}{\ref{prop:bell_shape}}
\renewcommand{\theHtheorem}{app.restate.\ref{prop:bell_shape}}
\begin{proposition} For any $L$, $2\le k\le L$, and any bell-shaped time distribution $\pi$, the following statement always holds:
    \begin{align*}
        \mathbb{P}\bigg(\Mask\notin\mathcal{E}_k\mid t\sim\mathrm{U}([0,1]), M^i\sim \mathrm{Bernoulli}(t)\bigg)<\mathbb{P}\bigg(\Mask\notin\mathcal{E}_k\mid t\sim\pi, M^i\sim \mathrm{Bernoulli}(t)\bigg).
    \end{align*}
\end{proposition}
\endgroup
\begin{proof}
For brevity, write $R(t):=R_{L,k}(t)$. By \eqref{eq:prob_mass_R_pi},
for any probability measure $\mu$ on $[0,1]$,
\[
\mathbb P\left(
\mathbf M\in \mathcal E_k^c
\ \middle|\
t\sim \mu,\ M_i\stackrel{\mathrm{i.i.d.}}{\sim}\mathrm{Bernoulli}(t)
\right)
=
\int_0^1 R(t)\,d\mu(t).
\]
Thus it suffices to compare $\int R(t)\,d\pi(t)$ with
$\int R(t)\,dt$. Since $\pi$ may be atomic or assign zero mass to parts of $[0,1]$, we avoid any pointwise density comparison and instead use an integral over the central intervals $[a,1-a]$. By Lemma~\ref{lem:R-central-monotone}, $R(t)=R(1-t)$ and $R$ is strictly
increasing on $[0,1/2]$. In particular, $R'(a)\ge 0$ on $[0,1/2]$.
For every $t\in[0,1]$, we have
\[
R(t)
=
R(0)+\int_0^{1/2} R'(a)\mathbf 1\{t\in[a,1-a]\}\,da.
\]
Indeed, if $t\le 1/2$, the right-hand side is
$R(0)+\int_0^t R'(a)\,da=R(t)$; if $t\ge 1/2$, it is
$R(0)+\int_0^{1-t}R'(a)\,da=R(1-t)=R(t)$.

Integrating this identity with respect to $\pi$ gives
\[
\int_0^1 R(t)\,d\pi(t)
=
R(0)+
\int_0^{1/2} R'(a)\mathbb{P}_{t\sim\pi}\left(t\in [a,1-a]\right)\,da.
\]
Similarly, for the uniform distribution,
\[
\int_0^1 R(t)\,dt
=
R(0)+
\int_0^{1/2} R'(a)(1-2a)\,da,
\]
because the interval $[a,1-a]$ has length $1-2a$.

Therefore,
\[
\int_0^1 R(t)\,d\pi(t)-\int_0^1 R(t)\,dt
=
\int_0^{1/2}
R'(a)\left(\mathbb{P}_{t\sim\pi}\left(t\in [a,1-a]\right)-(1-2a)\right)\,da.
\]
Since $\pi$ is bell-shaped,
\[
\mathbb{P}_{t\sim\pi}\left(t\in [a,1-a]\right)\ge \mathbb{P}_{t\sim\mathrm{U}([0,1])}\left(t\in [a,1-a]\right)=1-2a
\qquad
\text{for all }a\in[0,1/2),
\]
and since $R'(a)\ge0$, the difference above is nonnegative.

It remains to show that the inequality is strict. By the definition of a bell-shaped
distribution, there exists $a_0\in(0,1/2)$ such that
\[
\mathbb{P}_{t\sim\pi}\left(t\in [a_0,1-a_0]\right)>1-2a_0.
\]
Let
\[
\eta:=\mathbb{P}_{t\sim\pi}\left(t\in [a_0,1-a_0]\right)-(1-2a_0)>0.
\]
Choose $0<\varepsilon<\min\{a_0,\eta/4\}$ and set
$I=[a_0-\varepsilon,a_0]$. For every $a\in I$, we have
$[a_0,1-a_0]\subset [a,1-a]$, so
\[
\mathbb{P}_{t\sim\pi}\left(t\in [a,1-a]\right)-(1-2a)
\ge
\mathbb{P}_{t\sim\pi}\left(t\in [a_0,1-a_0]\right)-(1-2a)
=
\eta-2(a_0-a)
>
\frac{\eta}{2}.
\]
Hence the bracketed term is strictly positive throughout the interval $I$.

Since $R$ is strictly increasing on $[0,1/2]$,
\[
\int_I R'(a)\,da
=
R(a_0)-R(a_0-\varepsilon)
>
0.
\]
Therefore,
\[
\int_0^1 R(t)\,d\pi(t)-\int_0^1 R(t)\,dt
\ge
\int_I
R'(a)\left(\mathbb{P}_{t\sim\pi}\left(t\in [a,1-a]\right)-(1-2a)\right)\,da
>
0.
\]
Consequently,
\[
\mathbb P\left(
\mathbf M\in\mathcal E_k^c
\ \middle|\
t\sim\pi,\ M_i\stackrel{\mathrm{i.i.d.}}{\sim}\mathrm{Bernoulli}(t)
\right)
>
\mathbb P\left(
\mathbf M\in\mathcal E_k^c
\ \middle|\
t\sim U([0,1]),\ M_i\stackrel{\mathrm{i.i.d.}}{\sim}\mathrm{Bernoulli}(t)
\right).
\]
Taking complements gives
\[
\mathbb P\left(
\mathbf M\in\mathcal E_k
\ \middle|\
t\sim\pi,\ M_i\stackrel{\mathrm{i.i.d.}}{\sim}\mathrm{Bernoulli}(t)
\right)
<
\mathbb P\left(
\mathbf M\in\mathcal E_k
\ \middle|\
t\sim U([0,1]),\ M_i\stackrel{\mathrm{i.i.d.}}{\sim}\mathrm{Bernoulli}(t)
\right).
\]
This concludes the proof.
\end{proof}

\paragraph{Remark.}
In measure-theoretic notation, the probability mass assigned by a probability distribution \(\pi\) to an interval \(I\subseteq[0,1]\) is conventionally written as \(\pi(I)\). To make this notation more accessible to readers less familiar with this convention, we instead write the same quantity as \(\mathbb P_{t\sim\pi}(t\in I)\) throughout the main text and proofs.

\section{Omitted Experimental Details and Results}
\subsection{Omitted Experimental Details}\label{sec:omitted_experimental_details}

\noindent\textbf{Experimental setup used in LM1B experiments.} We detokenize the One Billion Words (LM1B) dataset following prior works~\citep{lou2024discrete,sahoo2024simple}.
We tokenize LM1B using the \textsc{bert-base-uncased} tokenizer~\citep{he2022diffusionbert}, and concatenate sequences to a fixed length of 128 using sentence packing~\citep{raffel2020exploring}. We parameterize the diffusion backbone~\citep{lou2024discrete,sahoo2024simple} using a modified diffusion transformer (DiT) architecture~\citep{peebles2023scalable} with 12 layers, hidden size 768, and 12 attention heads.
We train using the AdamW optimizer with a batch size of 512, and adopt a linear warmup schedule from 0 to $3\mathrm{e}{-4}$ over 2,500 steps~\citep{sahoo2024simple,lou2024discrete,sahoo2025diffusionduality}. Training was conducted on 2$\times$H200 GPUs for both LM1B with and without sentence packing and consumed $\sim$100 hours to reach 1M steps.

\noindent\textbf{Experimental setup used in OWT experiments.} For OpenWebText (OWT), we use the GPT-2 tokenizer~\citep{radford2019language} following prior works~\citep{lou2024discrete,sahoo2024simple} and pack sequences to a context length of 1{,}024, inserting an \texttt{eos} token between consecutive documents.
As OWT does not provide an official validation split, we reserve the last 100K documents for validation. We use the same DiT architecture specified in above LM1B setup. Following prior work~\citep{shi2025simplified}, we train using the AdamW optimizer with a batch size of 512, and adopt a cosine annealing with weight decay scheduler with warmup from 0 to $3\mathrm{e}{-4}$ over 2,500 steps. Training was conducted on 4$\times$H100 GPUs for OWT and consumed $\sim$300 hours to reach 1M steps.

\noindent\textbf{Default setting for analysis on MDM training inefficiency (Sec.~\ref{sec:slow_learner}), training dynamics of MDM with various time distribution (Sec.~\ref{sec:training_dynamics}), and analysis on other alternatives (Sec.~\ref{sec:other_alternatives}).}
All experiments in these sections use LM1B with sentence packing as the default setting, and follow the aforementioned setup for LM1B experiments. For ARM, we utilize the same architecture for MDMs, and the only difference is that ARM utilizes causal attention. For $\sigma$-ARM, we just utilize same architecture of ARM, yet just permute the data and RoPE~\citep{Su2021rope} angle with randomly selected but fixed $\sigma$. For AoARM, we do not directly adopt the architecture of Xue et al.~\citep{xue2025anyordergpt}, since it does not use RoPE and injects source and target positions through sinusoidal embeddings and AdaLN, respectively. To keep the comparison compatible with our RoPE-based backbone, we implement a simple AoARM variant: source-position RoPE is applied in the first half of the layers, and target-position RoPE in the second half. Since this experiment only aims to test whether AoARM can outperform $\sigma$-ARM, this simple implementation is sufficient for our analysis.

\noindent\textbf{Experimental setup used in pretraining benchmarks (Sec.~\ref{sec:result_various_benchmark}).}
The training setup inherits all the aforementioned setups for LM1B and OWT, respectively. We report both generative and zero-shot perplexity.
Generative perplexity is computed using a GPT-2 Large autoregressive model, following prior works~\citep{sahoo2024simple,lou2024discrete,sahoo2025diffusionduality},
while zero-shot perplexity is evaluated on the corresponding validation splits of each dataset. See Appendix~\ref{sec:zero_and_gen_ppl} for details and results of zero-shot perplexity and generative perplexity.

\noindent\textbf{Experimental setup used in continual pretraining (Sec.~\ref{sec:result_cpt}).}
We use GPT-2 Large (36 layers, 20 attention heads, hidden size 1280; 774M parameters) as the backbone and train a discrete diffusion language model with full-parameter updates.
Continual pretraining (CPT) is performed on pre-tokenized FineWeb-Edu sequences with a global batch size of 256.
We use a learning rate of $3\mathrm{e}{-4}$ with 2K warmup steps and 10K annealing steps. For downstream supervised fine-tuning (SFT), we construct a 210K-example instruction-style dataset by combining the training splits of the following benchmarks: LAMBADA~\citep{paperno2016lambada}, HellaSwag~\citep{zellers2019hellaswag}, TriviaQA~\citep{joshi2017triviaqa}, OpenBookQA~\citep{mihaylov2018can}, PIQA~\citep{bisk2020piqa}, Social IQA~\citep{sap2019social}, and WinoGrande~\citep{sakaguchi2021winogrande}. This composition follows prior instruction tuning paradigm~\citep{sanh2021multitask}, where downstream benchmarks are repurposed as supervised training data to induce general reasoning capabilities.
Each example is formatted as a prompt ending with \texttt{Answer:} followed by the target answer string.
We fine-tune with a global batch size of 128, 64 diffusion steps, and a single annealing step, using a cosine learning rate schedule with 500 warmup steps in \texttt{bf16} precision and full-parameter updates.
All downstream experiments run for 10K SFT steps. For instruction-following evaluation, we further fine-tune on 103K cleaned \textsc{ShareGPT-GPT4} conversations, following the SFT protocol of Nie~\etal~\citep{nie2024scaling}.
Starting from 2K-step downstream SFT checkpoints, we train for an additional 10K steps with a global batch size of 128 and learning rate $1\mathrm{e}{-4}$.
We use sequence length 1024, 64 diffusion steps, and a single annealing step. Training was conducted on 4$\times$B200 GPUs and consumed $\sim$50 hours to reach 100K steps.

To briefly explain the CPT proposed by Gong~\etal~\citep{gong2024scaling} and used in our experiments, Gong~\etal adapt pretrained autoregressive language models into diffusion language models via continued pretraining, rather than training diffusion models from scratch.
Their recipe bridges the gap between ARM and diffusion modeling by converting causal attention into bidirectional denoising behavior and preserving the ARM-style shift operation during adaptation.
Importantly, this CPT procedure does not introduce a new diffusion objective; it directly optimizes the standard MDM loss $\mathcal{L}_{\mathrm{MDM}}$.
In this CPT setting, applying our bell-shaped time sampling simply amounts to replacing this loss with $\hat{\mathcal{L}}_{\pi}$, while keeping the rest of CPT recipe proposed by Gong~\etal unchanged. For supervised fine-tuning, following LLaDA~\citep{nie2025llada}, we keep the condition tokens, such as the instruction and prompt, fixed without masking them, and mask only the response tokens for denoising under the standard MDM loss $\mathcal{L}_{\mathrm{MDM}}$.
Similarly, applying our bell-shaped time sampling in this SFT setting simply amounts to replacing $\mathcal{L}_{\mathrm{MDM}}$ with $\hat{\mathcal{L}}_{\pi}$, while keeping the conditioning scheme unchanged.

\subsection{Additional Ablation on Bell-shaped Time Sampling}\label{sec:ablation_t_sampling}

We provide additional hyperparameter ablations for the time sampling distributions
introduced in Table~\ref{tab:time_sampling_distributions}. All experiments use the same LM1B sentence-packed setting
as in Sec.~\ref{sec:training_dynamics}, and we report validation NLL during training. In all figures,
``Base'' denotes the conventional MDM objective, while Gaussian$(0.5,0.1)$ denotes
the default bell-shaped time sampling distribution used in the main experiments.

\begin{figure}[t]
    \centering
    \includegraphics[width=0.6\linewidth]{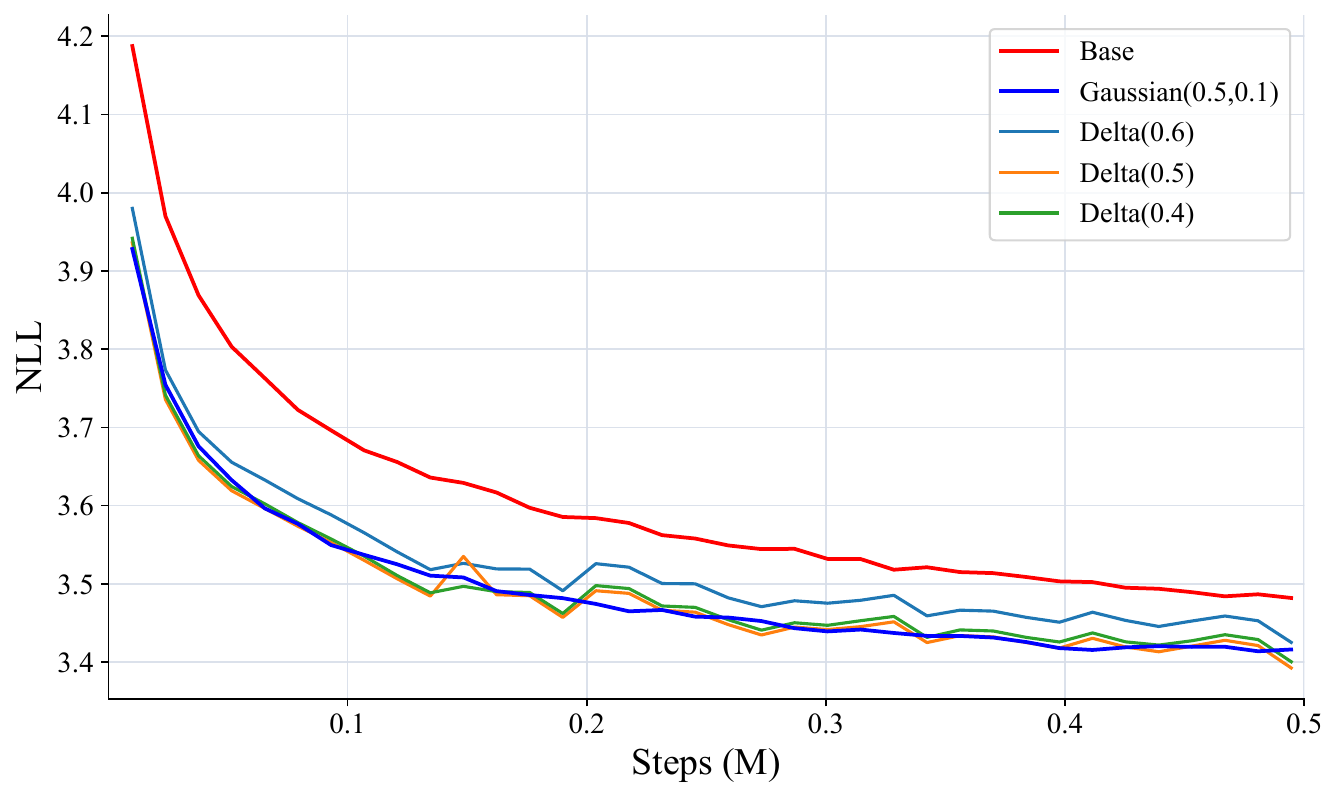}
    \caption{
    Additional ablation on Dirac-delta time sampling. We compare different
    choices of the delta location $\mu$ against the standard MDM baseline and
    the default Gaussian$(0.5,0.1)$ time sampling.
    }
    \label{fig:ablation_t_sampling_delta}
\end{figure}

\noindent\textbf{Dirac-delta time sampling.}
Fig.~\ref{fig:ablation_t_sampling_delta} compares Dirac-delta distributions
$\pi=\delta_\mu$ with different choices of $\mu$. All delta-based variants
substantially outperform the standard MDM baseline, confirming that concentrating
training on a narrow range of middle-context corruption levels is beneficial.
Among them, $\delta_{0.5}$ and $\delta_{0.4}$ show strong performance, whereas
$\delta_{0.6}$ is consistently worse. Since a larger $t$ corresponds to a lower-context
regime, this result is consistent with the observation in Sec.~\ref{sec:training_dynamics} that shifting
the sampling distribution toward $\mu=0.6$ degrades training dynamics. This result might indicate that the training inefficiency of low-context regime is larger than that of high-context regime. However,
although delta sampling can strongly emphasize a specific context level, it removes
diversity across corruption levels, which makes it less robust than the default
Gaussian sampling.

\begin{figure}[t]
    \centering
    \includegraphics[width=0.6\linewidth]{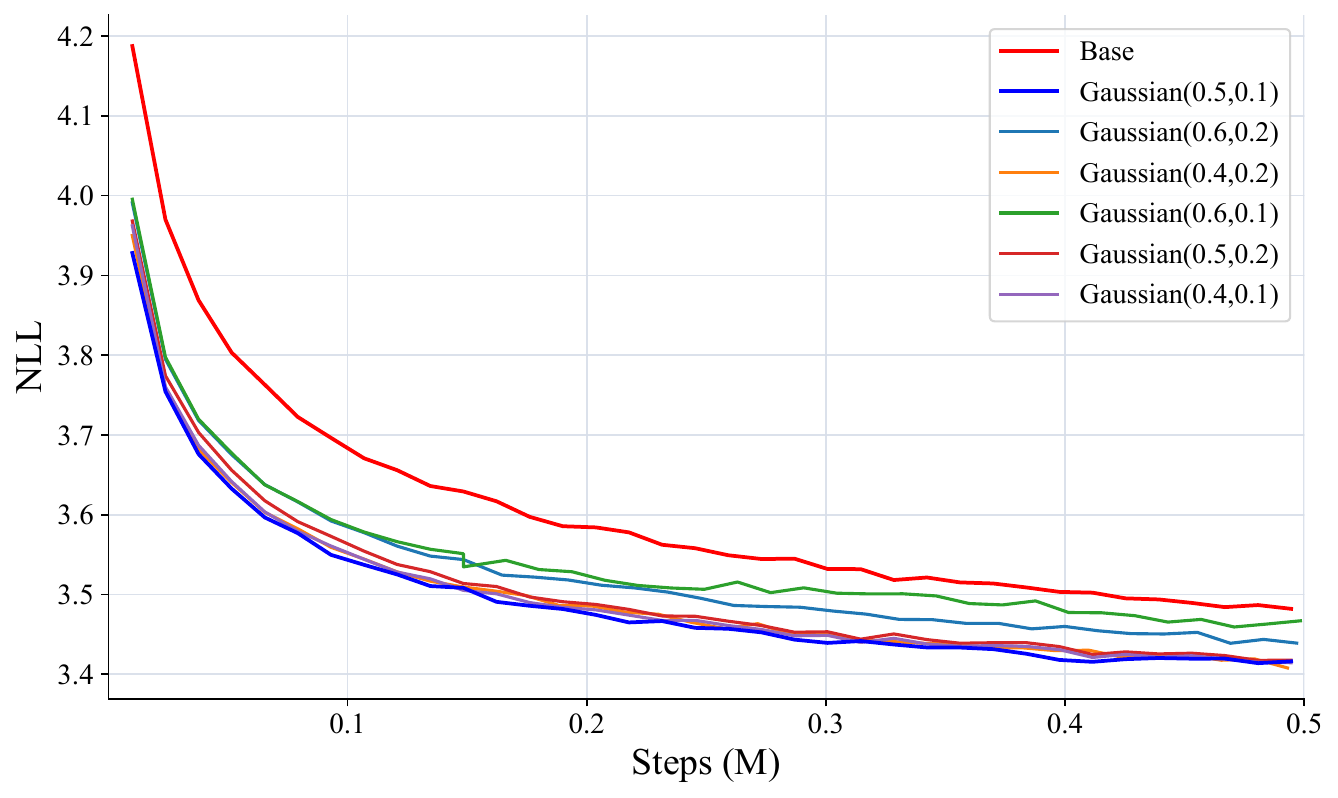}
    \caption{
    Additional ablation on truncated Gaussian time sampling. We vary the mean
    and standard deviation of the truncated Gaussian distribution and compare
    them with the standard MDM baseline and the default Gaussian$(0.5,0.1)$.
    }
    \label{fig:ablation_t_sampling_gaussian}
\end{figure}

\noindent\textbf{Gaussian time sampling.}
Fig.~\ref{fig:ablation_t_sampling_gaussian} further varies the mean and standard
deviation of the truncated Gaussian distribution. The results show that Gaussian
time sampling is broadly effective: all Gaussian variants improve over the base
MDM. Nevertheless, the mean is important. Distributions centered at $\mu=0.5$ or
slightly shifted to $\mu=0.4$ remain strong, whereas shifting the mean toward
$\mu=0.6$ yields worse training dynamics, especially when the variance is small.
Increasing the standard deviation from $\sigma=0.1$ to $\sigma=0.2$ partially
mitigates this degradation, since the distribution assigns more probability mass
back to the middle-context region. These results support our default choice of
Gaussian$(0.5,0.1)$: it is centered at the most balanced context level while
keeping the sampling distribution sufficiently concentrated.

\begin{figure}[t]
    \centering
    \includegraphics[width=0.6\linewidth]{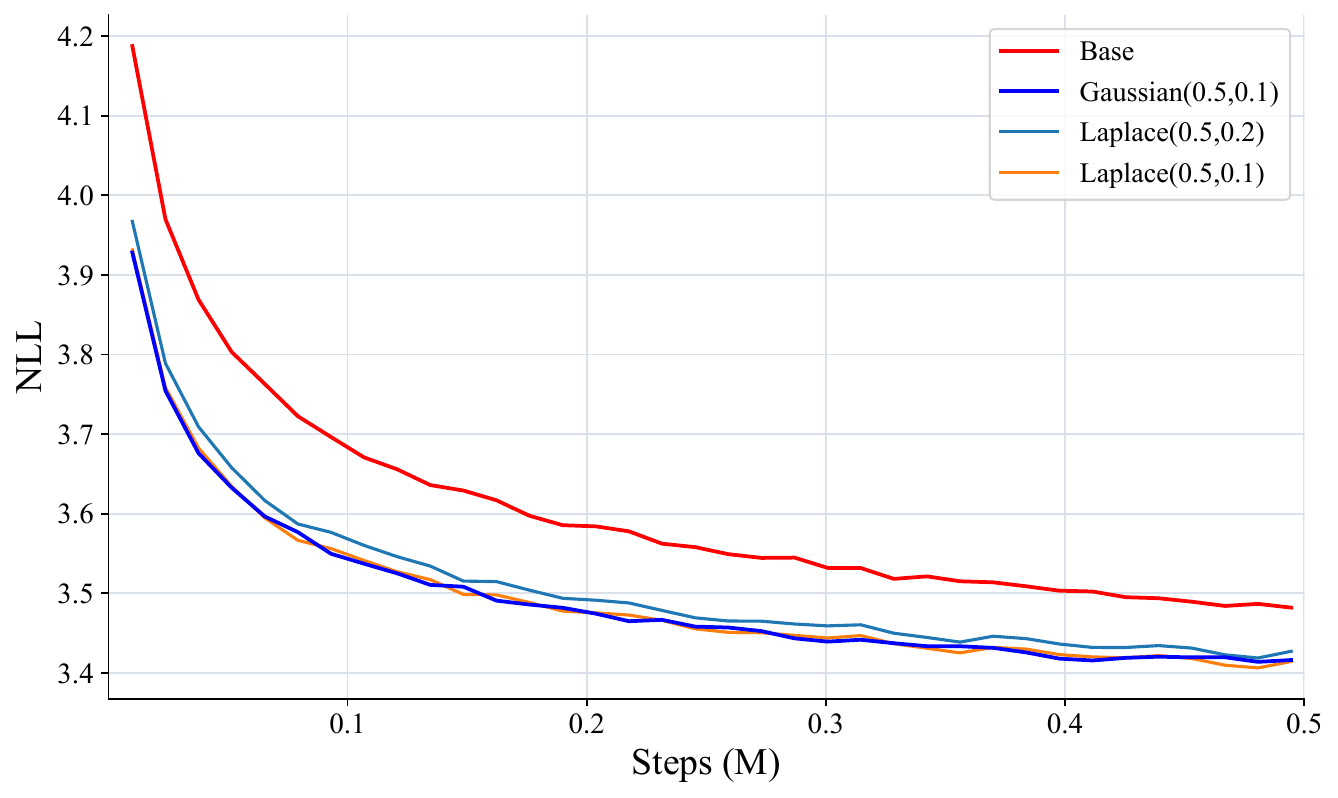}
    \caption{
    Additional ablation on truncated Laplace time sampling. We compare
    Laplace distributions with different scale parameters against the standard
    MDM baseline and the default Gaussian$(0.5,0.1)$.
    }
    \label{fig:ablation_t_sampling_laplace}
\end{figure}

\noindent\textbf{Laplace time sampling.}
Fig.~\ref{fig:ablation_t_sampling_laplace} evaluates truncated Laplace
distributions centered at $\mu=0.5$ with different scale parameters. Both
Laplace$(0.5,0.1)$ and Laplace$(0.5,0.2)$ outperform the standard MDM baseline,
showing that the benefit of bell-shaped time sampling is not specific to the
Gaussian family. The narrower Laplace distribution performs slightly better, while
the larger-scale variant remains effective but is somewhat less competitive because
its heavier tails allocate more probability mass to the low- and high-context
regions. Together with the Gaussian ablation, this suggests that the crucial factor
is not the exact parametric form of $\pi(t)$, but whether the distribution keeps
most of its probability mass around the middle-context region.

\subsection{Zero-shot Perplexity and Generative Perplexity of Models Trained on OpenWebText}\label{sec:zero_and_gen_ppl}

\begin{table*}[t]
\centering
\caption{Zero-shot perplexities ($\downarrow$) of models trained for 1M steps on OpenWebText (mean $\pm$ std over 5 seeds).}
\label{tab:zeroshot-ppl}
\small
\begin{adjustbox}{width=\linewidth}
\begin{tabular}{lccccccc}
\toprule
\textbf{Training} & \multicolumn{7}{c}{\textbf{Benchmark}} \\
\cmidrule(l){2-8}
 & PTB & Wikitext & LM1B & Lambada & AG News & Pubmed & Arxiv \\
\midrule
Base     
& 115.99 $\pm$ 5.60 
& 37.76 $\pm$ 0.94 
& 72.06 $\pm$ 0.45 
& 48.18 $\pm$ 1.18 
& 68.31 $\pm$ 1.19 
& 43.15 $\pm$ 0.13 
& \textbf{37.85 $\pm$ 0.14} \\

Gaussian 
& \textbf{98.16 $\pm$ 4.29} 
& \textbf{34.51 $\pm$ 0.80} 
& \textbf{67.71 $\pm$ 0.38} 
& \textbf{46.39 $\pm$ 1.05} 
& \textbf{61.74 $\pm$ 1.26} 
& \textbf{41.43 $\pm$ 0.11} 
& 37.88 $\pm$ 0.15 \\

\bottomrule
\end{tabular}
\end{adjustbox}
\end{table*}

\begin{table}[t]
\centering
\caption{Generative PPL ($\downarrow$) of models trained for 1M steps on OpenWebText (mean $\pm$ std over 5 seeds). NFE denotes the number of function evaluations used during generation.}
\label{tab:gen_ppl}
\small
\begin{tabular}{@{}lccc@{}}
\toprule
\textbf{NFE} & \textbf{512} & \textbf{1024} & \textbf{2048} \\
\midrule
Base     
& 56.51 $\pm$ 1.12 
& 42.81 $\pm$ 0.73 
& 36.04 $\pm$ 1.53 \\

Gaussian 
& \textbf{55.50 $\pm$ 1.05} 
& \textbf{41.55 $\pm$ 2.09} 
& \textbf{33.89 $\pm$ 0.91} \\

\bottomrule
\end{tabular}
\end{table}

We provide zero-shot perplexity and generative perplexity for models trained on OWT in Table~\ref{tab:zeroshot-ppl} and~\ref{tab:gen_ppl}, respectively. Following prior works~\citep{sahoo2024simple,hong2026lomdm}, we utilize 7 zero-shot perplexity benchmarks, including PTB~\citep{marcus1993building}, Wikitext~\citep{merity2016pointer}, LM1B, Lambada~\citep{paperno2016lambada}, AG News~\citep{zhang2015character}, Pubmed~\citep{cohan2018discourse}, and Arxiv~\citep{cohan2018discourse}. The generative PPL is computed
by a pre-trained GPT-2 Large on 100 generated samples with
length of 1024. As shown in the tables, the model trained with Gaussian time sampling consistently outperforms the baseline.

\subsection{Other Explored Training Techniques in Detail}\label{sec:exp_alternatives}

In this section, we further explain the details of the other training techniques briefly described in Sec.~\ref{sec:other_alternatives}.
Recall that, for a strictly decreasing scheduler $\alpha_t$, the continuous-time NELBO of MDM can be written as
\begin{align}
\mathcal{L}_{\mathrm{MDM}}
=
\int_{0}^{1}
\frac{-\alpha_t'}{1-\alpha_t}\,
\E_{\x\sim p_{\mathrm{data}},\,\x_t\sim q_{t|0}}
\Bigl[
\sum_{i:x_t^i=\mask}
-\log p_\theta(x^i\mid \x_t,t)
\Bigr]\,dt .
\label{eq:nelbo-general-scheduler}
\end{align}
As in the main experimental results, we use a truncated Gaussian density for the middle-biased alternatives in Sec.~\ref{sec:other_alternatives}.
Define
\begin{align}
f_{\mathrm{TG}}(t)
=
\frac{1}{\sigma Z}\,
\varphi\!\Bigl(\frac{t-\mu}{\sigma}\Bigr),
\qquad
t\in[\ell,h],
\label{eq:tn-density}
\end{align}
where $Z$ is the normalization constant.
We regard $f_{\mathrm{TG}}(a)$ as zero outside $[\ell,h]$. Note that we utilized $\ell=0,h=1,\mu=0.5$ and $\sigma=0.1$ for all experiments in Sec.~\ref{sec:other_alternatives}.

\begin{figure}[t]
    \centering
    \includegraphics[width=0.8\linewidth]{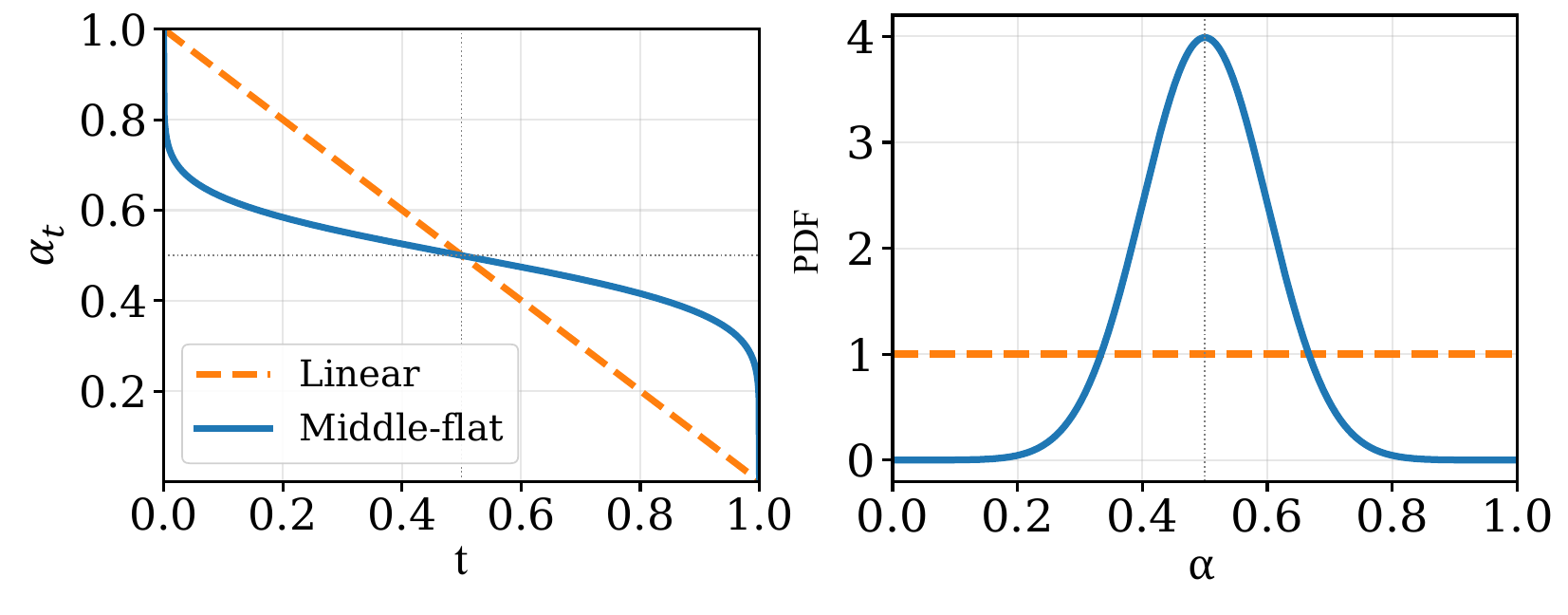}
    \caption{
    \textit{Left:} the middle-flat scheduler $\alpha_t^{\mathrm{MF}}=F_{\mathrm{TG}}^{-1}(1-t)$ with $\ell=0,h=1,\mu=0.5$ and $\sigma=0.1$ is flatter around $\alpha=0.5$ than the linear scheduler, so a uniform draw of $t$ produces many more middle-context corruption levels.
    \textit{Right:} under $t\sim\mathrm{Uniform}(0,1)$, the induced marginal distribution of $\alpha_t^{\mathrm{MF}}$ matches the target truncated-Gaussian density $f_{\mathrm{TG}}(a)$.}
    \label{fig:middle_flat_scheduler}
\end{figure}

\noindent\textbf{Middle-flat noise scheduler $\alpha_t^{\mathrm{MF}}$.} Prior works~\citep{sahoo2024simple,shi2025simplified} show that the NELBO of MDM is invariant to the specific choice of $\alpha_t$. In this regard, a natural alternative is to modify the scheduler itself so that the model encounters middle-context samples more often while still optimizing the theoretical NELBO.
To this end, define the truncated-normal CDF
\begin{align}
F_{\mathrm{TG}}(t)
:=
\int_{\ell}^{t} f_{\mathrm{TG}}(u)\,du,
\label{eq:tn-cdf}
\end{align}
and choose the scheduler as
\begin{align}
\alpha_t^{\mathrm{MF}}
=
F_{\mathrm{TG}}^{-1}(1-t),
\qquad
t\sim\mathrm{Uniform}(0,1).
\label{eq:middle-flat-scheduler}
\end{align}
Recall that the forward process of MDM is defined as follows:
\begin{equation}
q(\mathbf{x}_t^{i} \mid \mathbf{x})=q_{t|0}(\cdot \mid x^i) = \mathrm{Cat}\!\left(\alpha_t \e_{x^i} + (1 - \alpha_t)\e_m\right),\nonumber
\end{equation}
Thus, the model indeed sees many more middle-context samples if we set $\alpha_t$ as $\alpha_t^{\mathrm{MF}}$ rather than the original linear scheduler of MDM, \ie, $1-t$, as visualized in Fig.~\ref{fig:middle_flat_scheduler}.

However, this does not give the same objective as bell-shaped time sampling and is equal to the theoretical NELBO of MDM (\eqref{eq:loss_mdlm}). 
The reason is that the NELBO already contains the scheduler-dependent prefactor $-\alpha_t'/(1-\alpha_t)$. Formally,
\begin{align}
\frac{d\alpha_t^{\mathrm{MF}}}{dt}
=
-\frac{1}{f_{\mathrm{TG}}(\alpha_t^{\mathrm{MF}})},
\end{align}
and therefore
\begin{align}
\frac{-\alpha_t^{\mathrm{MF}\,\prime}}{1-\alpha_t^{\mathrm{MF}}}
=
\frac{1}{
(1-\alpha_t^{\mathrm{MF}})
f_{\mathrm{TG}}(\alpha_t^{\mathrm{MF}})
}.
\label{eq:middle-flat-weight-in-t}
\end{align}
Since $f_{\mathrm{TG}}(a)$ is largest around the middle region, the NELBO prefactor in \eqref{eq:middle-flat-weight-in-t} becomes smallest near the middle and larger away from it.
In other words, the apparent oversampling of middle-context scheduler values is compensated by a smaller NELBO prefactor.

Therefore, the practical difference between bell-shaped time sampling and the middle-flat scheduler is twofold: 1) as discussed above, the middle-flat scheduler does not modify the learning signal itself, and thus cannot concentrate training on the middle-context region. 2) The scaling factor ${-\alpha_t^{\mathrm{MF}\,\prime}}/{1-\alpha_t^{\mathrm{MF}}}$ of the middle-flat scheduler becomes much larger near $t=0$ or $t=1$ than that of $\alpha_t=1-t$, increasing the variance of the mini-batch estimator itself. For these two reasons, we conjecture that the middle-flat scheduler performs worse than the conventional MDM loss.

\noindent\textbf{Importance sampling.}
Another possible alternative is to keep the canonical absorbing-state scheduler $\alpha_t=1-t$, but to \emph{reweight the loss} so that the resulting objective matches bell-shaped time sampling in expectation.
Concretely, we keep drawing $t$ uniformly, and use the importance ratio
$f_{\mathrm{TG}}(t)$, then the reweighted objective then becomes
\begin{align*}
\mathcal{L}_{\mathrm{IS}}=\int_{0}^{1} \frac{f_{\mathrm{TG}}(t)}{t}\mathbb{E}_{\x\sim p_{\mathrm{data}},\,\x_t\sim q_{t|0}}\bigl[-\sum_{i:x_t^i=\mask}\log p_\theta(x^i\mid \x_t,t)\bigr]dt
\end{align*}
where $-\alpha_t'/(1-\alpha_t)=1/t$ for $\alpha_t=1-t$.
Here, if we drop $t$ in the denominator, $\mathcal{L}_{\mathrm{IS}}$ matches the \emph{same expectation} as bell-shaped time sampling. The result shown in Fig.~\ref{fig:other_alternatives} corresponds to the setting where we drop $t$ in the denominator, but we observed similar trends even when we do not drop $t$ in the denominator.

The practical difference between importance sampling and ours is that importance sampling does \emph{not} alter which corruption levels are actually drawn in a minibatch: samples are still generated from the uniform schedule, and only their scalar losses are reweighted afterward.
By contrast, bell-shaped time sampling directly draws middle-context samples more frequently.
Hence, although importance sampling and bell-shaped time sampling optimize the same objective in expectation, their optimization dynamics can differ substantially in practice; the former may suffer from a noisier gradient estimator because informative middle-context samples are not physically sampled more often.

\subsection{Stratified Sampling}
Recall that our bell-shaped time sampling loss is given as follows:
\begin{align}
\textstyle \hat{\mathcal{L}}_{\pi}
=
\mathbb{E}_{\x\sim p_{\mathrm{data}},\,t\sim\pi,\,\x_t\sim q_{t|0}(\cdot\mid\x)}
\Bigl[
-
\sum_{i:x_t^i=\mask}
\log p_\theta(x^i\mid\x_t,t)
\Bigr].
\end{align}
A practical issue arises when $\pi$ is sharply concentrated around the middle-context region. If each $t_i$ in a mini-batch is drawn independently from $\pi$, some batches can contain too many samples from the same context regime, which increases the variance of the gradient estimator and may destabilize optimization. 
To mitigate this, we utilize \emph{stratified time sampling}, which is widely used in prior works~\citep{sahoo2024simple,sahoo2025diffusionduality}.
Intuitively, stratified sampling divides the probability mass of the target time distribution into $B$ equal-probability strata and draws one sample from each stratum.
This makes each mini-batch cover the target time distribution more evenly than i.i.d. sampling, while preserving the expectation of the batch estimator.

To our knowledge, no prior work has applied stratified sampling to bell-shaped time sampling distributions, so we describe it here in detail. Let $F_{\pi}$ denote the CDF of $\pi$. For a batch of size $B$, instead of i.i.d. sampling $t_i {\sim} \pi$, we first stratify the unit interval:
\[
\hat{u}_i=\tfrac{i-1}{B}, \qquad \epsilon_i \sim \mathrm{U}\big([0,\tfrac{1}{B})\big), \qquad u_i=\hat{u}_i+\epsilon_i,
\]
for $i\in[B]$, and then map through the inverse CDF, $t_i^{\mathrm{strat}} = F_{\pi}^{-1}(u_i)$.
This preserves the expectation of the batch estimator while reducing its variance compared with naive sampling. Note that all experiments in our paper utilized stratified sampling.

\section{Limitations}\label{sec:limitations}
We showed that bell-shaped time sampling is effective not only for from-scratch training, but also for billion-scale CPT with ARM initialization.
In Sec.~\ref{sec:training_dynamics} and Appendix~\ref{sec:ablation_t_sampling}, we further analyzed the training dynamics of a wide range of bell-shaped distributions on LM1B.
Nevertheless, the optimal bell-shaped time distribution may depend on dataset characteristics and model scale, and a more systematic analysis of this dependence remains an important direction.
Still, our experiments suggest that a Gaussian distribution with mean $0.5$ and standard deviation $0.1$ works robustly across most settings.
Another promising future direction is to investigate whether bell-shaped time sampling can also benefit post-training methods, such as reinforcement learning.

\section{Qualitative Examples}\label{sec:qualitative_examples}

In this section, we provide qualitative generation examples from the instruction-following models evaluated in Sec.~\ref{sec:result_cpt}. Specifically, we compare the outputs of the base MDM and the bell-shaped time sampling trained MDM after applying the same ShareGPT supervised fine-tuning procedure. These qualitative examples indicate that, after the same number of training steps, the base MDM has not yet acquired the ability to perform conditional generation, whereas our MDM trained with bell-shaped time sampling already demonstrates this ability.

\begin{figure*}[t]
\centering
\small

\begin{subfigure}{\textwidth}
\centering

\begin{tcolorbox}[
    colback=white,
    colframe=black,
    boxrule=0.8pt,
    arc=2pt,
    left=6pt,
    right=6pt,
    top=6pt,
    bottom=6pt,
]

\textbf{(a) Salary-negotiation email generation}

\vspace{0.4em}

\begin{tcolorbox}[
    colback=gray!4,
    colframe=gray!45,
    boxrule=0.5pt,
    arc=2pt,
    left=4pt,
    right=4pt,
    top=4pt,
    bottom=4pt,
    title=\textbf{User Prompt},
    fonttitle=\small
]
You need to write an email to negotiate your salary.
\end{tcolorbox}

\vspace{0.4em}

\begin{minipage}[t]{0.485\linewidth}
\begin{tcolorbox}[
    colback=red!2,
    colframe=red!45!black,
    boxrule=0.5pt,
    arc=2pt,
    left=4pt,
    right=4pt,
    top=4pt,
    bottom=4pt,
    title=\textbf{Base MDM},
    fonttitle=\small
]
\ttfamily\scriptsize
However, to understand how you need to negotiate your salary, you need to understand the following steps.\\[0.4em]

To negotiate your salary, you need to need to understand the following of steps steps:\\[0.4em]

1. negotiate a fair salary: negotiating a fair salary is a salary that is fair to the parties and the b.\\
2. negotiate a fair salary that is a and the a a salary that is fair to the parties a and the b.\\
3. negotiate a fair salary that is fair to the parties a and the b.\\
4. negotiate fair salary5. negotiate the fair salary that is fair to the parties and the b\ldots
\end{tcolorbox}
\end{minipage}
\hfill
\begin{minipage}[t]{0.485\linewidth}
\begin{tcolorbox}[
    colback=blue!2,
    colframe=blue!45!black,
    boxrule=0.5pt,
    arc=2pt,
    left=4pt,
    right=4pt,
    top=4pt,
    bottom=4pt,
    title=\textbf{Gaussian time trained MDM},
    fonttitle=\small
]
\ttfamily\scriptsize
Sure, here's an example of an email that you could use to negotiate your salary:\\[0.4em]

Subject: Request for a Salary Raise\\[0.4em]

Dear [Name],\\[0.4em]

I am writing to request a salary raise. I understand that you are a full-time employee and you are entitled to a salary salary. However, I would like to request that you consider a salary raise.\\[0.4em]

I understand that you are a valuable employee and I would like to ensure that you are able to achieve your full potential. I would like to schedule an effort to negotiate your salary.\\[0.4em]

Thank you for your consideration\ldots
\end{tcolorbox}
\end{minipage}

\end{tcolorbox}

\end{subfigure}

\vspace{0.7em}

\begin{subfigure}{\textwidth}
\centering

\begin{tcolorbox}[
    colback=white,
    colframe=black,
    boxrule=0.8pt,
    arc=2pt,
    left=6pt,
    right=6pt,
    top=6pt,
    bottom=6pt,
]

\textbf{(b) Professional feedback email generation}

\vspace{0.4em}

\begin{tcolorbox}[
    colback=gray!4,
    colframe=gray!45,
    boxrule=0.5pt,
    arc=2pt,
    left=4pt,
    right=4pt,
    top=4pt,
    bottom=4pt,
    title=\textbf{User Prompt},
    fonttitle=\small
]
Draft a professional email seeking your supervisor's feedback on the `Quarterly Financial Report' you prepared. Ask specifically about the data analysis, presentation style, and the clarity of conclusions drawn. Keep the email short and to the point.
\end{tcolorbox}

\vspace{0.4em}

\begin{minipage}[t]{0.485\linewidth}
\begin{tcolorbox}[
    colback=red!2,
    colframe=red!45!black,
    boxrule=0.5pt,
    arc=2pt,
    left=4pt,
    right=4pt,
    top=4pt,
    bottom=4pt,
    title=\textbf{Base MDM},
    fonttitle=\small
]
\ttfamily\scriptsize
Additionally, please encourage the reader to read the rest of the report.\\[0.4em]

Now, here's the revised version of the email with the changes:\\[0.4em]

Subject: Your supervisor's feedback on the requested:\\
, the Quarterly Financial Report'\\[0.4em]

Subject: Feedback on the `Quarterly Financial Report'\\
To:\\
:: Your Supervisor's Name,\\
\ldots\\
{[Your Supervisor's Name]}\\
Best regards,\\
{[Your Name,\\
Your Supervisor Name]}\ldots
\end{tcolorbox}
\end{minipage}
\hfill
\begin{minipage}[t]{0.485\linewidth}
\begin{tcolorbox}[
    colback=blue!2,
    colframe=blue!45!black,
    boxrule=0.5pt,
    arc=2pt,
    left=4pt,
    right=4pt,
    top=4pt,
    bottom=4pt,
    title=\textbf{Gaussian time trained MDM},
    fonttitle=\small
]
\ttfamily\scriptsize
Subject:\\[0.4em]

Dear ,\\[0.4em]

I am writing to request your feedback on the `Quarterly Financial Report' prepared for the upcoming quarter. I would appreciate your feedback on the data analysis, presentation style, and clarity of the conclusions drawn.\\[0.4em]

I would also appreciate your feedback on the presentation of the data and the overall clarity of the financial report.\\[0.4em]

Thank you for your time and consideration\ldots
\end{tcolorbox}
\end{minipage}

\end{tcolorbox}

\end{subfigure}

\vspace{-0.3em}
\caption{
\textbf{Qualitative comparisons on professional email generation tasks.}
Across both email-writing tasks, the baseline MDM frequently degenerates into malformed formatting and repetitive text fragments. In contrast, the Gaussian time trained MDM generates substantially more coherent responses with recognizable email structure, appropriate formatting, and improved instruction following, despite minor repetition artifacts.
}
\label{fig:email_generation_combined}
\end{figure*}
\begin{figure*}[t]
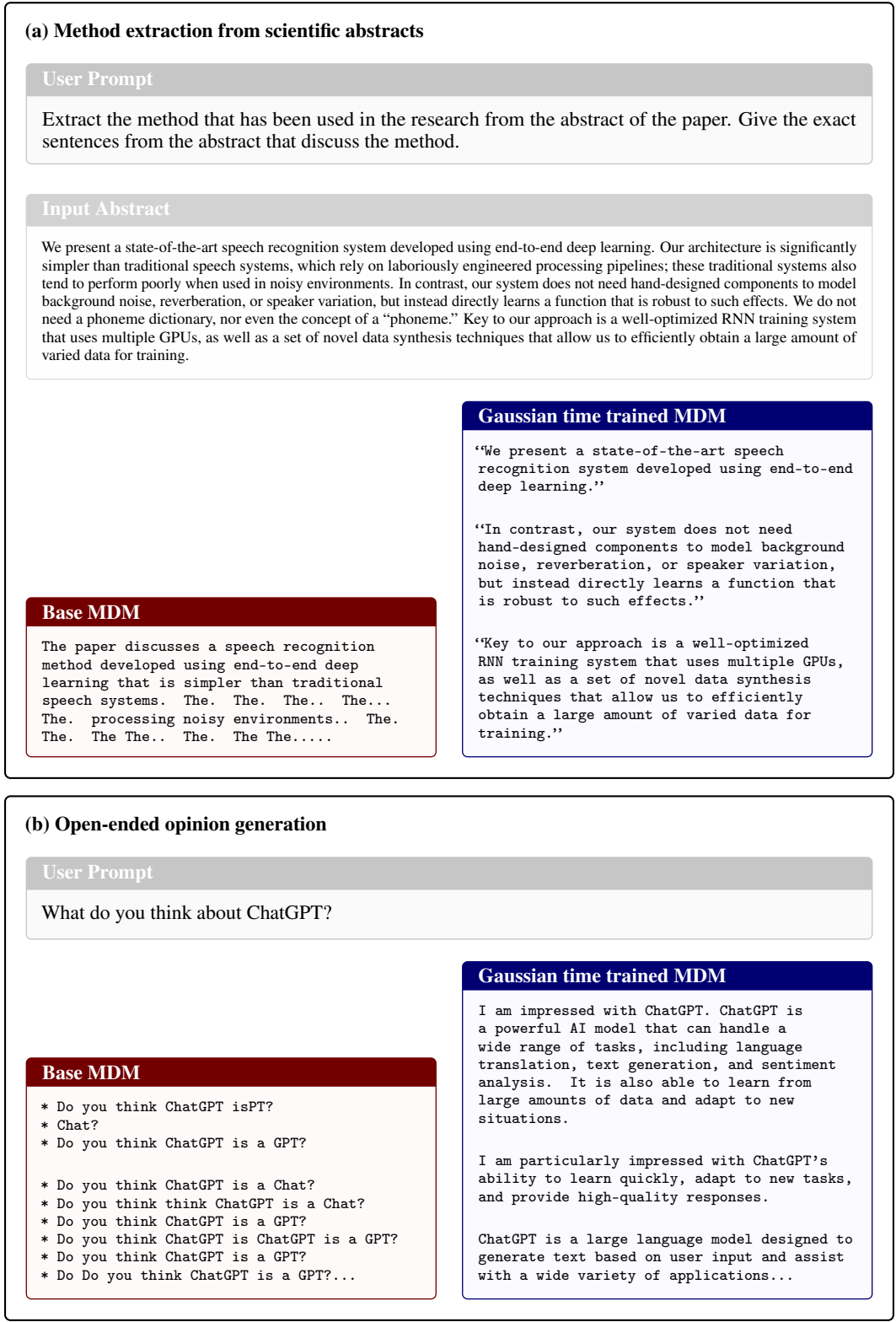

\centering
\small

\begin{subfigure}{\textwidth}
\centering

\begin{tcolorbox}[
    colback=white,
    colframe=black,
    boxrule=0.8pt,
    arc=2pt,
    left=6pt,
    right=6pt,
    top=6pt,
    bottom=6pt,
]

\textbf{(a) Method extraction from scientific abstracts}

\vspace{0.4em}

\begin{tcolorbox}[
    colback=gray!4,
    colframe=gray!45,
    boxrule=0.5pt,
    arc=2pt,
    left=4pt,
    right=4pt,
    top=4pt,
    bottom=4pt,
    title=\textbf{User Prompt},
    fonttitle=\small
]
Extract the method that has been used in the research from the abstract of the paper. Give the exact sentences from the abstract that discuss the method.
\end{tcolorbox}

\vspace{0.35em}

\begin{tcolorbox}[
    colback=gray!1,
    colframe=gray!35,
    boxrule=0.4pt,
    arc=2pt,
    left=4pt,
    right=4pt,
    top=4pt,
    bottom=4pt,
    title=\textbf{Input Abstract},
    fonttitle=\small
]
\scriptsize
We present a state-of-the-art speech recognition system developed using end-to-end deep learning. Our architecture is significantly simpler than traditional speech systems, which rely on laboriously engineered processing pipelines; these traditional systems also tend to perform poorly when used in noisy environments. In contrast, our system does not need hand-designed components to model background noise, reverberation, or speaker variation, but instead directly learns a function that is robust to such effects. We do not need a phoneme dictionary, nor even the concept of a ``phoneme.'' Key to our approach is a well-optimized RNN training system that uses multiple GPUs, as well as a set of novel data synthesis techniques that allow us to efficiently obtain a large amount of varied data for training.
\end{tcolorbox}

\vspace{0.4em}

\begin{minipage}[t]{0.485\linewidth}
\begin{tcolorbox}[
    colback=red!2,
    colframe=red!45!black,
    boxrule=0.5pt,
    arc=2pt,
    left=4pt,
    right=4pt,
    top=4pt,
    bottom=4pt,
    title=\textbf{Base MDM},
    fonttitle=\small
]
\ttfamily\scriptsize
The paper discusses a speech recognition method developed using end-to-end deep learning that is simpler than traditional speech systems. The. The. The.. The... The. processing noisy environments.. The. The. The The.. The. The The..\ldots
\end{tcolorbox}
\end{minipage}
\hfill
\begin{minipage}[t]{0.485\linewidth}
\begin{tcolorbox}[
    colback=blue!2,
    colframe=blue!45!black,
    boxrule=0.5pt,
    arc=2pt,
    left=4pt,
    right=4pt,
    top=4pt,
    bottom=4pt,
    title=\textbf{Gaussian time trained MDM},
    fonttitle=\small
]
\ttfamily\scriptsize
``We present a state-of-the-art speech recognition system developed using end-to-end deep learning.''\\[0.4em]

``In contrast, our system does not need hand-designed components to model background noise, reverberation, or speaker variation, but instead directly learns a function that is robust to such effects.''\\[0.4em]

``Key to our approach is a well-optimized RNN training system that uses multiple GPUs, as well as a set of novel data synthesis techniques that allow us to efficiently obtain a large amount of varied data for training.''
\end{tcolorbox}
\end{minipage}

\end{tcolorbox}

\end{subfigure}

\vspace{0.7em}

\begin{subfigure}{\textwidth}
\centering

\begin{tcolorbox}[
    colback=white,
    colframe=black,
    boxrule=0.8pt,
    arc=2pt,
    left=6pt,
    right=6pt,
    top=6pt,
    bottom=6pt,
]

\textbf{(b) Open-ended opinion generation}

\vspace{0.4em}

\begin{tcolorbox}[
    colback=gray!4,
    colframe=gray!45,
    boxrule=0.5pt,
    arc=2pt,
    left=4pt,
    right=4pt,
    top=4pt,
    bottom=4pt,
    title=\textbf{User Prompt},
    fonttitle=\small
]
What do you think about ChatGPT?
\end{tcolorbox}

\vspace{0.4em}

\begin{minipage}[t]{0.485\linewidth}
\begin{tcolorbox}[
    colback=red!2,
    colframe=red!45!black,
    boxrule=0.5pt,
    arc=2pt,
    left=4pt,
    right=4pt,
    top=4pt,
    bottom=4pt,
    title=\textbf{Base MDM},
    fonttitle=\small
]
\ttfamily\scriptsize
* Do you think ChatGPT isPT?\\
* Chat?\\
* Do you think ChatGPT is a GPT?\\[0.4em]

* Do you think ChatGPT is a Chat?\\
* Do you think think ChatGPT is a Chat?\\
* Do you think ChatGPT is a GPT?\\
* Do you think ChatGPT is ChatGPT is a GPT?\\
* Do you think ChatGPT is a GPT?\\
* Do Do you think ChatGPT is a GPT?\ldots
\end{tcolorbox}
\end{minipage}
\hfill
\begin{minipage}[t]{0.485\linewidth}
\begin{tcolorbox}[
    colback=blue!2,
    colframe=blue!45!black,
    boxrule=0.5pt,
    arc=2pt,
    left=4pt,
    right=4pt,
    top=4pt,
    bottom=4pt,
    title=\textbf{Gaussian time trained MDM},
    fonttitle=\small
]
\ttfamily\scriptsize
I am impressed with ChatGPT. ChatGPT is a powerful AI model that can handle a wide range of tasks, including language translation, text generation, and sentiment analysis. It is also able to learn from large amounts of data and adapt to new situations.\\[0.4em]

I am particularly impressed with ChatGPT's ability to learn quickly, adapt to new tasks, and provide high-quality responses.\\[0.4em]

ChatGPT is a large language model designed to generate text based on user input and assist with a wide variety of applications\ldots
\end{tcolorbox}
\end{minipage}

\end{tcolorbox}

\end{subfigure}

\vspace{-0.3em}
\caption{
\textbf{Qualitative comparisons between the baseline MDM and Gaussian time trained MDM.}
Across both structured extraction and open-ended generation tasks, the baseline model frequently degenerates into repetitive or malformed text, whereas the Gaussian-trained model produces substantially more coherent and instruction-following responses.
}
\label{fig:qualitative_combined}
\end{figure*}
\clearpage
\section{Assets and Licenses}
We provide licenses for datasets and benchmarks in Table~\ref{tab:asset_licenses}.
\begin{table}[t]
\centering
\small
\caption{Existing datasets and benchmarks used in this work. ``Not explicitly specified'' means that we could not find an explicit dataset license from the official provider or public dataset card.}
\label{tab:asset_licenses}
\begin{tabular}{p{0.28\linewidth} p{0.23\linewidth} p{0.41\linewidth}}
\hline
\textbf{Asset} & \textbf{Usage} & \textbf{License / usage terms} \\
\hline

LM1B~\citep{chelba2013lm1b}
& Pretraining / validation
& Not explicitly specified on the official benchmark page; derived from WMT 2011 News Crawl data. \\

OpenWebText~\citep{Gokaslan2019owt}
& Pretraining / validation
& CC0-1.0 for dataset packaging. \\

FineWeb-Edu~\citep{lozhkov2024fineweb-edu}
& Continual pretraining
& ODC-By 1.0; CommonCrawl Terms of Use apply. \\

\hline


WinoGrande~\citep{sakaguchi2021winogrande}
& Downstream SFT / evaluation
& CC-BY for dataset, version not specified in official README; Apache-2.0 for code \\

HellaSwag~\citep{zellers2019hellaswag}
& Downstream SFT / evaluation
& MIT. \\

Social IQA~\citep{sap2019social}
& Downstream SFT / evaluation
& CC-BY-4.0. \\

PIQA~\citep{bisk2020piqa}
& Downstream SFT / evaluation
& Academic Free License v3.0 (AFL-3.0) \\


OpenBookQA~\citep{mihaylov2018can}
& Downstream SFT / evaluation
& Apache-2.0 for official repository; dataset license not separately specified. \\


TriviaQA~\citep{joshi2017triviaqa}
& Downstream SFT / evaluation
& Apache-2.0 for code and data \\

ROCStories~\citep{mostafazadeh2016corpus}
& Downstream SFT / evaluation
& CC-BY-4.0. \\

\hline

PTB~\citep{marcus1993building}
& Zero-shot PPL evaluation
& LDC User Agreement for Non-Members. \\

WikiText~\citep{merity2016pointer}
& Zero-shot PPL evaluation
& CC-BY-SA-4.0. \\

LAMBADA~\citep{paperno2016lambada}
& Zero-shot PPL / downstream evaluation
& CC-BY-4.0. \\

AG News~\citep{zhang2015character}
& Zero-shot PPL evaluation
& Research / non-commercial use. \\

Scientific Papers: PubMed / arXiv~\citep{cohan2018discourse}
& Zero-shot PPL evaluation
& Source-specific OA licenses; code repository Apache-2.0. \\

\hline

\textsc{ShareGPT-GPT4}
& Instruction-following SFT
& CC-BY-4.0 for the commonly used \texttt{sharegpt\_gpt4} release. \\

AlpacaEval 2.0~\citep{li2023alpacaeval}
& Instruction-following evaluation
& CC-BY-NC-4.0 per HF dataset metadata; Apache-2.0 for code. Note: HF script contains an inconsistent CC-BY-4.0 field. \\

MT-Bench~\citep{zheng2023judging}
& Instruction-following evaluation
& Apache-2.0 via FastChat. \\

\hline
\end{tabular}
\end{table}



\end{document}